\documentclass{article} \usepackage{iclr2027_conference,times}

\usepackage[utf8]{inputenc}   \usepackage[T1]{fontenc}    	\usepackage{url}              		 \usepackage{booktabs}       	\usepackage{amsfonts}       	\usepackage{nicefrac}       	  \usepackage{microtype}

\usepackage{graphicx}
\usepackage{color}
\usepackage{rotating}
\usepackage{tabularx}
\usepackage{pdflscape}
\usepackage{amsmath,amsthm,amssymb}
\usepackage{algorithm,algorithmic}
\usepackage{bbm,dsfont}
\usepackage{subfig}
\usepackage{caption}
\usepackage{wrapfig}
\usepackage{cancel}
\usepackage{enumerate,cases}
\usepackage{thmtools,thm-restate}
\usepackage{mathtools}

\usepackage{bm}
\usepackage{multirow}
\usepackage{makecell}
\usepackage{setspace}
\usepackage{pifont}
\usepackage{array}
\usepackage{enumitem}
\usepackage{lineno}

\allowdisplaybreaks

\usepackage{hyperref}		 \hypersetup{
	colorlinks	= true,		 urlcolor     = blue,	 linkcolor	 = purple, 	 citecolor    = violet    }
\usepackage[capitalize]{cleveref}
\crefformat{figure}{Fig.~#2#1#3}       \crefformat{table}{Tab.~#2#1#3}        \crefformat{equation}{Eq.~(#2#1#3)}    \crefformat{section}{Sec.~#2#1#3}      \crefformat{appendix}{App.~#2#1#3}     \crefformat{algorithm}{Alg.~#2#1#3}    

\usepackage{graphicx}
\usepackage{float}
\usepackage{algorithm}
\usepackage{amsmath, amssymb, amsthm}
\usepackage{tikz}
\usetikzlibrary{arrows.meta,positioning,fit,calc}

\usepackage{silence}
\usepackage{fancyhdr}
\renewcommand{\headrulewidth}{0pt}

\newcommand{\ox}{x}
\newcommand{\tx}{{x^\star}}
\newcommand{\para}[1]{\textbf{#1}~}
\newcommand{\slin}[1]{\theta_\star^\top{#1}}
\newcommand{\tlin}[1]{\hat\theta_t^\top{#1}}

\newcommand{\tlina}[1]{\hat\theta_{t,-a}^\top{#1}}

\newcommand{\algo}{\textsc{COBRA}}
\newcommand{\neql}{\textsc{NE}}
\newcommand{\ucb}{\textsc{UCB}}
\newcommand{\lcb}{\textsc{LCB}}

\newcommand{\Regret}{\kR}

\definecolor{revadded}{rgb}{0.45,0.10,0.75}   \definecolor{revdeleted}{rgb}{0.85,0.00,0.00}

\newcommand{\EE}[1]{\bE\left[#1\right]}

\newcommand{\Prob}[1]{\bP\left\{#1\right\}}

\newcommand{\R}{\bR}

\newcommand{\ind}[1]{\mathbbmss{1}{\Lp #1 \Rp} }
\renewcommand{\phi}{\varphi}
\renewcommand{\epsilon}{\varepsilon}

\newcommand{\norm}[1]{\left\|#1\right\|}
\newcommand{\step}[1]{{\large \textcircled{\small #1}}}

\newcommand{\al}[1]{ \begin{align} #1  \end{align}}
\newcommand{\eq}[1]{ \begin{equation} #1  \end{equation}}
\newcommand{\als}[1]{ \begin{align*} #1  \end{align*}}
\newcommand{\eqs}[1]{ \begin{equation*} #1  \end{equation*}}

\newcommand{\Lp}{\left(}
\newcommand{\Rp}{\right)}
\newcommand{\Lb}{\left[}
\newcommand{\Rb}{\right]}

\newcommand{\el}{\end{flushleft}}
\newcommand{\bl}{\begin{flushleft}}

\newcommand{\argmax}{\arg\!\max}

\newcommand{\bE}{\mathbb{E}}

\newcommand{\bN}{\mathbb{N}}

\newcommand{\bP}{\mathbb{P}}

\newcommand{\bR}{\mathbb{R}}

\newcommand{\cA}{\mathcal{A}}
\newcommand{\cB}{\mathcal{B}}
\newcommand{\cC}{\mathcal{C}}

\newcommand{\cE}{\mathcal{E}}

\newcommand{\cG}{\mathcal{G}}

\newcommand{\cK}{\mathcal{K}}

\newcommand{\cN}{\mathcal{N}}
\newcommand{\cO}{\mathcal{O}}

\newcommand{\cX}{\mathcal{X}}

\newcommand{\kA}{\mathfrak{A}}

\newcommand{\kR}{\mathfrak{R}}

\theoremstyle{plain}

\newtheorem{lem}{Lemma}
\newtheorem{prop}{Proposition}

\newtheorem{rem}{Remark}
\newtheorem{defi}{Definition}
\newtheorem{assu}{Assumption}

\title{
    \algo: Contextual Bandit Algorithm for \\Ensuring Truthful Strategic Agents
}

\author{Arun Verma$^{1}$ \qquad Indrajit Saha$^{2}$ \qquad 
    Aadirupa Saha${^3}$  \qquad  \\ 
    \textbf{Daniela Rus$^{4,1}$ \qquad Makoto Yokoo$^{5}$ \qquad Bryan Kian Hsiang Low$^{6,1}$}\\
    $^{1}$Singapore-MIT Alliance for Research and Technology Centre, Republic of Singapore \\
    $^{2}$Department of Mathematical Sciences, Indian Institute of Technology (BHU), Varanasi, India \\
    $^{3}$University of Illinois Chicago, USA  \qquad
    $^{4}$CSAIL, Massachusetts Institute of Technology, USA \\
	$^{5}$Graduate School of ISEE, Kyushu University, Japan \\
    $^{6}$Department of Computer Science, National University of Singapore, Republic of Singapore }

\iclrfinalcopy 

\begin{document}
    \maketitle

    \begin{abstract}

This paper considers a contextual bandit problem involving multiple agents, where a learner sequentially observes the contexts and the agents' reported arms, and then selects the arm that maximizes the system's overall reward. Most existing work in contextual bandits assumes that agents always truthfully report their arms, which is unrealistic in many real-life applications. For instance, consider an online platform with multiple sellers; some sellers may misrepresent product features to maximize reward, such as having the platform preferentially recommend their products to its users. To address this challenge, we propose an algorithm, \algo{}, for contextual bandit settings involving strategic agents that disincentivizes their strategic behavior without using any monetary incentives. Under truthful reporting, \algo{} has sub-linear regret and, with high probability, never eliminates a truthful agent. Under over-reporting, \algo{}'s regret and the gain of a deviating agent are bounded conditionally and degrade gracefully with how well the other agents' observations cover the deviating agent's features. Experimental results also validate our theoretical results and the different performance aspects of \algo{}.
    \end{abstract}

    \section{Introduction}
    \label{sec:introduction}

Contextual bandit \citep{NOW19_slivkins2019introduction,Book_lattimore2020bandit} is a sequential decision-making framework in which a learner selects an arm for a given context to maximize its total reward. 
Unlike traditional multi-armed bandits \citep{ML02_auer2002finite,COLT11_garivier2011kl,COLT12_agrawal2012analysis}, contextual bandits~\citep{AISTATS11_chu2011contextual,NIPS11_krause2011contextual,UAI13_valko2013finite,ICML20_zhou2020neural} use
\begin{wrapfigure}[18]{r}{0.505\textwidth}
    \centering
    \vspace{-3mm}
    \includegraphics[width=\linewidth]{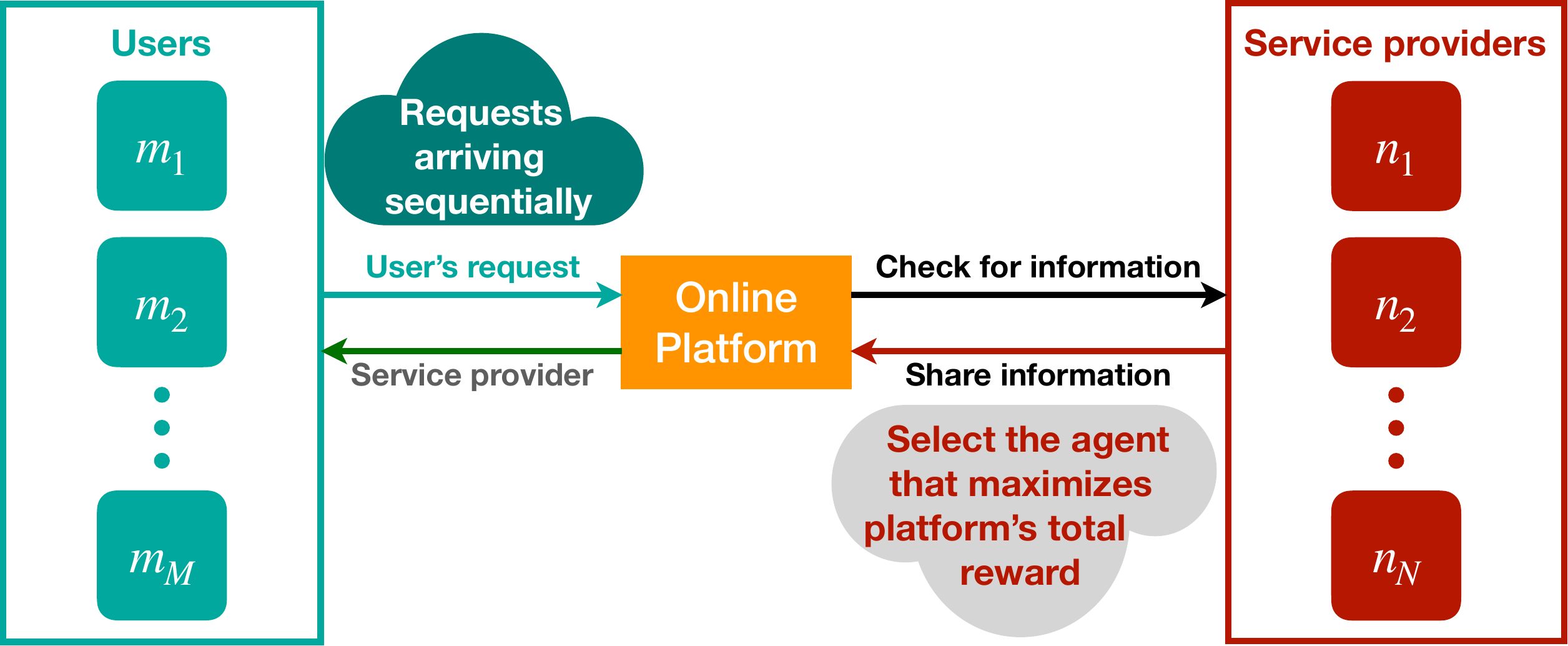}
    \vspace{-5mm}
    \caption{
        Example of a contextual bandit problem with strategic agents: Consider an online platform recommending service providers (agents) to users (context) who arrive sequentially. As service providers can misreport their private information to receive favorable recommendations, the platform must implement a mechanism incentivizing truthful reporting. With accurate private information, the platform can recommend the best service provider, thereby improving the overall 
        user experience.
        }
    \label{fig:csa}
    \vspace{-3mm}
\end{wrapfigure}
additional information, such as user profile, location, and purchase history, to make more informed and personalized decisions~\citep{WWW10_li2010contextual}. 
Contextual bandits have many real-life applications in personalized decision-making. Prominent examples include online recommendation systems~\citep{WWW10_li2010contextual}, online advertising \citep{geng2021comparison}, and clinical trials \citep{chow2006adaptive, aziz2021multi}, where the best recommendation depends on the context.

Many real-life applications of contextual bandits involve multiple strategic agents, from which the learner must select one to recommend based on the given context.
As illustrated in \cref{fig:csa}, consider an online platform with multiple service providers (agents), where the platform must recommend one provider to a user (context). 
In such settings, service providers can strategically misreport their information to influence the platform's decisions and increase their utilities by increasing their chances of being recommended \citep{resnick2007influence, zhang2019understanding}.
For example, an online food delivery platform wants to maximize the overall user experience by selecting and presenting the best restaurant options when a user searches for a specific type of food. 
Since users tend to order from restaurants listed at the top of search results \citep{malaga2008worst},
restaurants are incentivized to misrepresent their menu offerings to appear more prominently for specific food categories. 
This misreporting creates a significant challenge: If users consistently encounter misleading restaurant listings that fail to align with their true preferences, their experience with the platform will worsen; in the worst case, they may switch to competing platforms. 
Similar examples also include personalized pricing \citep{liu2024contextual}; algorithmic trading \citep{zeng2024partially}; and firms allocating resources \citep{zeng2024learning}.

These real-life applications highlight the importance of designing contextual bandit algorithms that discourage strategic misreporting by agents.
However, most existing contextual bandit algorithms 
do not account for strategic agent behavior, leading to suboptimal agent selection. 
We bridge this gap by designing a contextual bandit algorithm that accounts for potential misreporting and makes reporting arm features truthfully an approximate best response for agents when the other agents report truthfully.
Specifically, this paper answers the following question: \textbf{\em How to design an approximately incentive-compatible contextual bandit algorithm for settings where strategic agents may misreport their true features?}
To answer this, we propose a contextual bandit algorithm, \algo{}, that discourages strategic misreporting without relying on monetary incentives, while providing approximate incentive compatibility (IC) and sub-linear regret guarantees under appropriate assumptions, which we state explicitly.
Designing an incentive-compatible contextual bandit algorithm has the following key challenges, which we address using novel techniques.

\para{\step{1} Detecting strategic misreporting.} 
Contextual bandit algorithms \citep{NOW19_slivkins2019introduction, Book_lattimore2020bandit} typically assume that agents truthfully report the features of their arms, which may not hold in many practical applications where agents misreport to gain an advantage. 
To ensure that agents are discouraged from strategically misreporting arm features, we introduce a novel method, the \textit{Leave-One-Out-based Mechanism (LOOM) for identifying misreporting agents}, which draws inspiration from the exclusion principle of the Vickrey-Clarke-Groves (VCG) mechanism \citep{vickrey1961counterspeculation, clarke1971multipart, groves1973incentives} and uses the reported arm features of other agents to identify the misreporting agent; unlike VCG, it uses no payments and offers only an approximate IC guarantee.

\para{\step{2} Arm-selection under strategic misreporting.} 
Strategic misreporting of arm features introduces bias into the reward estimator, breaking the theoretical guarantees of existing contextual bandit algorithms~\citep{buening2024strategic}. 
To address this bias in the reward estimator, we first introduce the notion of a LOOM-compatible contextual bandit algorithm. 
A LOOM-compatible contextual bandit algorithm (e.g., algorithms based on upper confidence bound~\citep{NIPS11_abbasi2011improved,ICML17_chowdhury2017kernelized}) can integrate \textit{LOOM as a post-processing step after each round to identify the misreporting agent} by using suitable statistical estimates. 
This integration ensures that the underlying LOOM-compatible contextual bandit algorithm retains its guarantees under truthful reporting, while adaptively correcting for bias once LOOM detects a strategic agent.

\para{\step{3} Bounding regret due to strategic misreporting.} 
The arm-selection strategy relies on a reward estimator that may be biased by using strategically misreported arm features. This bias can lead to the selection of suboptimal arms, resulting in additional loss from not choosing the optimal arms (i.e., increased regret). 
To quantify this additional regret, we derive high-probability bounds on the estimation error of LOOM's reward estimators and use them to characterize the regret of the underlying contextual bandit algorithm when combined with LOOM.

We summarize the key contributions of this paper as follows:
\vspace{-3mm}
\begin{itemize}[leftmargin=*]
	\setlength\itemsep{0.0em}
    \item \para{Leave-one-out mechanism.}
    In \cref{sec:loom}, we introduce LOOM, a novel plug-in mechanism for contextual bandits with strategic agents inspired by the exclusion principle of the VCG framework. Unlike VCG, LOOM discourages strategic misreporting without monetary incentives and is compatible with any algorithm with valid confidence widths. We also show that under truthful reporting, with high probability, no truthful agent is ever eliminated (see~\cref{thm:optgtmarmsstayactive}).

    \item \para{Incentive-compatible contextual bandit algorithm.}
    We propose \algo{} in \cref{sec:cobra} that uses LOOM to disincentivize agents from misreporting. We prove that under truthful reporting, \algo{} has regret $\tilde{O}(d\sqrt{T})$ and no truthful agent loses more than $\tilde{O}(d\sqrt{T}/\Delta_{\min})$ selections relative to the oracle (see~\cref{thm:regretNE} for linear and \cref{thm:regretNonNE} for non-linear reward function).

    \item \para{Guarantees under strategic over-reporting.}
    Under the explicit validity condition in \cref{assu:NEandregret}, we prove that \algo{} has regret at most $\tilde{O}(Cd\sqrt{T}+\sqrt{NT})$ for every over-reporting strategy profile, and that truthfulness is an approximate Nash equilibrium against such deviations (see~\cref{thm:regretAllNE} for linear and \cref{thm:regretNonAllNE} for non-linear reward function), where $N$ is the number of agents, $d$ is the dimension of the context-arm feature vector, $T$ is the number of rounds, and $C\ge1$ is the maximum advantage any agent holds over the coverage provided by other agents. These bounds are sub-linear only when $C = o(\sqrt{T})$. We show that $C$ can grow as $\Theta(\sqrt{T})$ without peer coverage, so this dependence reflects a real limitation of leave-one-out detection (\cref{asec:loom}).
	
	\item \para{Empirical results.}
    In \cref{sec:experiments}, our experimental results on synthetic instances with strategic agents corroborate our theoretical results and validate the different performance aspects of \algo{}.
    
\end{itemize}

    \subsection{Related Work}
    \label{sec:related_work}

This section focuses on the most relevant work to our setting, i.e., strategic multi-armed and contextual bandits.
We discuss related topics, such as contextual bandits and strategic learning, in \cref{asec:related_work}.

\para{Strategic multi-armed bandits.}
The strategic variant of the multi-armed bandit problem was first studied by \cite{braverman2019multi}, in which the selected arm shares a fraction of its reward with the learner. Within this setting, they designed an incentive-compatible mechanism. Building on this work, \cite{yahmedstrategic} proposed an algorithm that rewards arms based on their reported values, achieving both incentive compatibility and sublinear regret. A related line of work by \cite{yin2022online} studied an online allocation problem that maximizes social welfare under fairness constraints in a strategic setting, assuming that valuations are unknown to the algorithm but follow an IID distribution. Their results show that when agents truthfully reveal their information, the mechanism maximizes social welfare while achieving sublinear regret relative to the offline-optimal policy. Our mechanism design follows a similar spirit but is applied to a different problem setting.
The robustness of bandit learning against strategic manipulation has also been explored, assuming a bounded manipulation budget \citep{feng2020intrinsic, dong2022combinatorial}. A distinct but related direction considers multi-armed bandits with replicas, where strategic agents can submit multiple copies of the same arm \citep{esmaeili2023replication, shin2022multi}. Finally, the integration of multi-armed bandits with mechanism design for online recommendations has been studied by \citet{kleine2023bandits}.

\para{Strategic contextual bandits.}
Our work is closely related to \cite{buening2024strategic}, which considers strategic agent behavior within a linear contextual bandit framework. Their method uses past allocation history to design agent-specific estimators that identify misreporting with high probability. In contrast, our method is inspired by the VCG mechanism \citep{vickrey1961counterspeculation, clarke1971multipart, groves1973incentives}, uses the reported arm features of other agents to identify misreports, and supports non-linear reward functions. 
Recent work by \cite{hu2025truthful} introduces a Bayesian contextual linear bandit framework in a similar spirit, with non-repeated interactions with agents, employing a linear programming-based approach to design an incentive-compatible mechanism. However, our work differs by explicitly considering repeated interactions with strategic agents in contextual bandits.

    \section{Contextual Bandits with Strategic Agents}
    \label{sec:problem}

\vspace{-1mm}
\para{Contextual bandits.}
This paper studies a contextual bandit problem with homogeneous strategic agents (\cref{assu:homogeneous}) who aim to maximize their utility (i.e., number of pulls) by strategically misreporting their arms' features to the learner, while the learner's goal is to select the agent for a given context that maximizes the total reward.
Our problem setting differs from standard contextual bandits as the arm features can be strategically manipulated by the agents to maximize their own utility.
Let $\cC$ be the set of all contexts and $\cA$ be the set of all arms of all agents.
Let $\cN$ be the set of all agents and $N \doteq |\cN|$ denote the number of agents.
For brevity, we use $\cX \subset \R^d$ to denote the set of all context-arm feature vectors, and $x^\star_{t, a} = \phi(c_t, a) \in \cX$ to represent the true feature vector associated with context $c_t$ and arm $a \in \cA$, where $\phi: \cC \times \cA  \rightarrow \cX$ is a feature map, and we assume that $\norm{x}_2 \le L,\ \forall x \in \cX$, i.e., the feature space is bounded.
At the start of round $t$, the environment generates a context $c_t \in \cC$ and each agent $a \in \cN_t$ reports the feature vector of its arm, where $\cN_t \subseteq \cN$ is the \emph{set} of active agents in round $t$, i.e., the agents not yet eliminated by the learner ($\cN_1 = \cN$), and $|\cN_t| > 2$. We write $\cA_t \subseteq \cA$ for the set of arms of the agents in $\cN_t$.
The learner then selects an arm $a_t \in \cA_t$ to recommend and observes a stochastic reward, denoted by $y_{t} \doteq f(x^\star_{t,a_t}) + \epsilon_t$, where $y_{t} \in \R$, $f: \cX  \rightarrow \R$ is an unknown fixed reward function, and $\epsilon_t$ is a $R$-sub-Gaussian noise. We further assume that the function is bounded, i.e., $|f(x)| \le B$ for all $x \in \cX$, for some $B< \infty$. For simplicity, each agent reports only one arm in each round so that we can use `agent' and `arm' interchangeably in the paper. We index both by $a$, so that $a \in \cN_t$ and $a \in \cA_t$ refer to the same agent and its arm.\footnote{
Our framework is more general and can also apply to settings where agents can report multiple arms per round, e.g., sellers offering multiple variants of the same product on an online platform. We also want to highlight that all agent-related computations can be performed in parallel, as they are independent of one another. The multiple-arms-per-agent case is treated in \cref{asec:loom} (`Agents with multiple arms').}

\para{Strategic manipulations by agents.}
A strategic agent can misreport the features of its arm by manipulating them such that the agent is selected more often, thereby maximizing its utility.
Let $x_{t,a}^\star$ be the true arm feature vector and $\ox_{t,a}$ be the reported arm feature vector for context $c_t$ and arm $a$.
Although an agent can strategically manipulate the arm feature vector; we assume the observed reward only depends on the true arm feature vector.\footnote{Sellers can misrepresent product features on the platform
such that it becomes a top recommendation. However, they cannot change the actual physical quality and nature of the product, i.e., observed reward $y_{t} = f(x_{t,a_t}^\star) + \epsilon_t$.}
Our goal is to design a contextual bandit algorithm with an approximately incentive-compatible mechanism under which truthful reporting (i.e., $\ox_{t,a} = x_{t,a}^\star$ for all $t \geq 1$ and $a \in \cA$) is an approximate best response for each agent when the other agents report truthfully.
Our results rely on the following structural assumption.

\begin{assu}[Homogeneous agents]
    \label{assu:homogeneous}
    All agents share a common, time-invariant reward function: for every round $t \ge 1$ and every agent $a \in \cN_t$, if $a$ is selected in round $t$, the learner observes $y_t = f(x_{t,a}^\star) + \epsilon_t$, with the same $f$ for all agents and all rounds (in the linear case of \cref{sec:cobra}, the same $\theta_\star$), where $\epsilon_t$ is $R$-sub-Gaussian conditionally on the history and on $a_t$.
\end{assu}
\vspace{-1mm}
\cref{assu:homogeneous} is what makes the observations of the other agents informative about the reward of agent $a$, and hence usable to test agent $a$'s reports. It does not, by itself, ensure that the other agents' observations \emph{cover} the feature directions of agent $a$; this coverage is quantified separately by the advantage constant $C$ (\cref{def:advantage}), and \cref{asec:loom} gives homogeneous instances without coverage in which leave-one-out detection fails.

\para{Incentive-Compatible algorithm.}
Let $\sigma_a$ denote the strategy of agent $a \in \cN$, which is history-dependent and maps the true features of its arm to reported features.
We use $\bm{\sigma}_{-a}$ to denote the strategies of all agents other than agent $a$, and $\bm{\sigma} = (\sigma_1, \sigma_2, \ldots, \sigma_{N})$ to represent the full strategy profile of all agents.
We first define what it means for an agent to be truthful.
\begin{defi}[Truthful]
    \label{def:truthfulstrategy}
    An agent $a \in \cN$ is said to be truthful if the agent reports the true features of its arm to the learner in each round, i.e., $\ox_{t,a} = x_{t,a}^\star$\ for all $t \ge 1$.
\end{defi}
\vspace{-1mm}

We use $\sigma^{\star}_a$ to denote the truthful strategy for the agent $a$ and $\bm{\sigma}^{\star}= (\sigma^{\star}_1, \sigma^{\star}_2, \ldots ,\sigma^{\star}_{N})$ to represent the vector of the truthful strategy for all agents.
Next, we formally define the utility of an agent $a$ in our setting.
Let $S_T(a) \doteq  \sum_{t=1}^T \ind{a_t = a}$ denote the number of times agent $a$ is selected by the learner up to round $T$.
Each agent's objective is to maximize its expected number of selections. Therefore, the utility of agent $a$ is given by
$
    \EE {S_T(a) \mid \bm{\sigma}},
$
where we condition on all agents' strategies $\bm{\sigma}$, and the expectation is over the contexts $c_1, \ldots, c_T$ (which determine all true features $x^\star_{t,a} = \phi(c_t,a)$), the reward noise, and any internal randomness of the learner and of the agents. 
In the following, we define the notion of $\epsilon$-Nash equilibrium (NE), in which no agent has more than $\epsilon$ incentive to deviate from the truthful reporting strategy.

\begin{defi}[$\epsilon$-Nash Equilibrium]
    \label{def:epsNE}
    Let $\epsilon >0$ and $T > 0$. We say that $\bm{\sigma} =(\sigma_1, \sigma_2, \ldots, \sigma_N)$ forms an $\epsilon$-Nash equilibrium if
    for any deviating strategy $\sigma^{\prime}_a (\neq \sigma_a)$ for any agent $a \in \cN$, the following holds:
    $
        \mathbb{E} \big[S_T(a) \mid \sigma_a, \bm{\sigma}_{-a}\big] \ge   \mathbb{E} \big[S_T(a) \mid \sigma^{\prime}_a, \bm{\sigma}_{-a}\big] -\epsilon.
    $
\end{defi}
\vspace{-1mm}

We next define (approximate) Nash incentive compatibility for a contextual bandit algorithm.\hspace{-1mm}
\begin{defi}[Incentive Compatibility]
    \label{def:IC}
    A contextual bandit algorithm is \emph{Nash incentive compatible} if the truthful profile $\bm{\sigma}^\star$ is a Nash equilibrium, i.e., reporting the true arm features is a best response for each agent \emph{when all other agents report truthfully}. For $\epsilon > 0$, it is \emph{$\epsilon$-incentive compatible} ($\epsilon$-IC) if $\bm{\sigma}^\star$ is an $\epsilon$-NE in the sense of \cref{def:epsNE}, i.e., for every agent $a \in \cN$ and every deviation $\sigma'_a$, $\EE{S_T(a) \mid \bm{\sigma}^\star} \ge \EE{S_T(a) \mid \sigma'_a, \bm{\sigma}^\star_{-a}} - \epsilon$.
\end{defi}
\vspace{-1mm}

\para{Performance measure.}
Let $a_t^\star$ denote the optimal arm (agent) for context $c_t$ having the maximum expected reward, i.e., $a_t^\star = \argmax_{a \in \cA_t} f(x_{t,a}^\star)$, and let $S^\star_T(a) \doteq \sum_{t=1}^T \ind{a^\star_t = a}$ denote the oracle number of selections of agent $a$.
After selecting arm $a_t$, the learner incurs a penalty $r_t$, where
$
    r_t = f(x_{t,a_t^\star}^\star) - f(x_{t,a_t}^\star).
$
Our aim is to learn a sequential policy that selects an arm for a given context such that the learner's total penalty for not selecting the optimal arm (or \emph{cumulative regret}) is as small as possible.
However, the performance of the contextual bandit algorithm depends on the incentive-compatible mechanism for the strategic agents whose strategy profile is represented by $\bm{\sigma} = (\sigma_1, \ldots, \sigma_{N})$.
We use strategic regret as a performance measure of a sequential policy $\pi$ under the strategy profile $\bm{\sigma}$ of the agents.
Specifically, for $T$ rounds and a strategy profile $\bm{\sigma}$ (e.g., $\bm{\sigma} \in \neql(\pi)$, the set of Nash equilibria under $\pi$, whenever it is non-empty), the strategic regret of a policy $\pi$ that selects arm $a_t$ in round $t$ is
\eq{
	\label{eq:regret}
	\Regret_T (\pi, \bm{\sigma}) \doteq \sum_{t=1}^{T} \Lp f(x_{t,a^{\star}_t}^\star) - f(x_{t,a_t}^\star)\Rp.
}

A policy $\pi$ is a good policy if it has sub-linear regret, i.e., $\lim_{T \rightarrow \infty}{\Regret_T(\pi, \bm{\sigma})}/T = 0$. This implies that, as $T$ increases, the policy $\pi$ will eventually start selecting optimal arms for the given contexts.

    \section{Leave-One-Out-based Mechanism (LOOM)}
    \label{sec:loom}

In the contextual bandit setting, designing an incentive-compatible mechanism that ensures truthful reporting of arm features by strategic agents is challenging due to limited access to true contexts, the potential for strategic misreporting, noisy reward feedback, and unknown reward function parameters.
These challenges naturally raise the question: \textit{How to design a mechanism that effectively incentivizes strategic agents to report truthfully while simultaneously achieving low regret?}

To answer this, we propose a method, \textit{Leave-One-Out-based Mechanism (LOOM) for identifying misreporting agents}, inspired by the Vickrey-Clarke-Groves (VCG) framework \citep{vickrey1961counterspeculation, clarke1971multipart, groves1973incentives} and uses the reported arm features of other agents to identify misreporting agents.
The connection to VCG is structural and narrow: as in VCG, agent $a$ is evaluated against a quantity computed \emph{without} agent $a$, so that $a$ cannot influence its own benchmark. Unlike VCG, LOOM uses no payments, its benchmark is a learned statistical estimate rather than an exactly computed welfare measure, and its incentive guarantee is approximate and coverage-dependent rather than dominant-strategy (see \cref{asec:loom}).
To identify whether an agent $a$ is misreporting (i.e., {\em over-reporting} arm features to increase its expected reward, such that $f(x_{t,a}) >  f(x_{t,a}^\star)$), LOOM uses three key components:
\step{1} a pessimistic estimate of the agent's total expected reward, derived from past data of all other agents, \step{2} an optimistic estimate of the agent's total reward, based on observed rewards when agent $a$ is selected, \step{3} a statistical test that uses these estimates to identify if agent $a$ is over-reporting with high probability.

\para{\step{1} Pessimistic estimate of the agent's total expected reward.}
Since the true reward function is unknown, LOOM estimates it using past observations (context-arm features and rewards) from all agents except agent $a$, ensuring the estimator is not influenced by agent $a$. 
Next, we formally introduce the notion of a {\em LOOM-compatible} contextual bandit algorithm, which refers to a class of algorithms that can integrate LOOM as a post-processing step after each round to identify the over-reporting agent.
Let $\cO_t$ be the past observations from all agents at the beginning of round $t$ and $\cO_{t,-a}$ represent the past observations from all agents except agent $a$. Let $f_{t}$ and $f_{t,-a}$ represent the estimate of reward function $f$ using observations $\cO_t$ and $\cO_{t,-a}$, respectively.

\begin{defi}[\textbf{LOOM-Compatible Contextual Bandit Algorithm}]
    \label{defi:loomCompatible}
    Any contextual bandit algorithm $\kA$ is {\em LOOM-compatible} if the following holds: (i) The estimated function $f_t^{\kA}$ from $\cO_t$, with probability at least $1-\delta$ simultaneously for all $t \ge 1$, satisfies:
    $
        \text{For any } x \in \cX: ~|f_t^{\kA}(x) - f(x)| \le h(x, \cO_t)
    $ 
    and (ii) The estimated function $f_{t,-a}^{\kA}$ from $\cO_{t,-a}$, with probability at least $1-\delta$ simultaneously for all $t \ge 1$, satisfies:
    $
        \text{For any } x \in \cX: ~|f_{t,-a}^{\kA}(x) - f(x)| \le h(x, \cO_{t,-a}).
    $
    Here, the value of $h(x, \cdot)$ depends on $x$, $\delta$, the past observations ($\cO_t$ or $\cO_{t,-a}$), and $\kA$.
\end{defi}
\vspace{-1mm}

\para{Linear example.} As an example, consider a linear reward $f(x) = \theta_\star^\top x$ with $\norm{\theta_\star}_2 \le S$. Lin-UCB estimates $\theta_\star$ by ridge regression on $\cO_t$ (resp.\ $\cO_{t,-a}$), i.e., $f_t(x) = \hat\theta_t^\top x$ and $f_{t,-a}(x) = \hat\theta_{t,-a}^\top x$, with widths $h(x, \cO_t) = \alpha_t\norm{x}_{V_t^{-1}}$ and $h(x, \cO_{t,-a}) = \alpha_{t,-a}\norm{x}_{V_{t,-a}^{-1}}$, where $V_t$ and $V_{t,-a}$ are the regularized Gram matrices of the observations in $\cO_t$ and $\cO_{t,-a}$ (see \cref{sec:cobra} for exact expressions). Both conditions of \cref{defi:loomCompatible} then hold simultaneously for all $t$ by the self-normalized bound of \citet{NIPS11_abbasi2011improved}, applied in (ii) to the rounds in which agent $a$ is not selected.

Many contextual bandit algorithms like Lin-UCB \citep{AISTATS11_chu2011contextual} (as shown in \cref{sec:cobra}), UCB-GLM \citep{ICML17_li2017provably}, IGP-UCB \citep{ICML17_chowdhury2017kernelized}, GP-TS \citep{ICML17_chowdhury2017kernelized}, Neural-UCB \citep{ICML20_zhou2020neural}, and Neural-TS \citep{ICLR21_zhang2020neural} are LOOM-compatible. 
Note that the assumptions required by contextual bandit algorithms must also hold in our setting, as they directly influence the performance of our proposed algorithm through $h(x, \cO_t)$ and $h(x, \cO_{t,-a})$. See \cref{table:hfunc} in \cref{asec:non_linear} for different values of  $h(x, \cO_t)$.

Henceforth, we assume that the underlying contextual bandit algorithm is LOOM-compatible and omit the superscript $\kA$ in estimators ($f_t^\kA$ and $f_{t,-a}^\kA$)  for notational simplicity.
For any $x \in \cX$, if $|f_{t,-a}(x) - f(x)| \le h(x, \cO_{t,-a})$ holds with probability at least $1-\delta^x_{t,a}$, then $\text{LCB}_{t,-a}(x)= f_{t,-a}(x) - h(x, \cO_{t,-a})$ is a pessimistic estimate of the expected reward for $x$, which also holds with probability at least $1-\delta^x_{t,a}$ (see \cref{lem:nonLinLooConfidenceBound} in \cref{asec:non_linear}).  
We further define
$
    \text{LCB}_{t,a}^{(x)} = \sum_{s=1,a_s = a}^{t} \text{LCB}_{t,-a}(x_{s,a_s}) 
$
to denote a pessimistic estimate of agent $a$'s total expected reward. 
We assume that $\text{LCB}_{t,a}^{(x)}$ holds with probability at least $1-\delta_{t,a}^{x}$ (more details are provided in \cref{asec:non_linear}).

\para{\step{2} Optimistic estimate of the agent's total reward.}
Since the observed reward depends only on the true feature vector, the learner receives a noisy reward, where the noise is sub-Gaussian.
Our following result provides an optimistic estimate of the agent $a$'s total expected reward.
\begin{restatable}{lem}{rewOptEst}
    \label{lem:rewOptEst}
    Let $S_t(a)$ be the number of times that agent $a$ is selected until round $t$, and $\epsilon_s$ be $R$-sub-Gaussian noise in the observed reward $y_s$, where $1 \le s \le t$. Let $\delta' \in (0,1)$ and $\delta_{t,a}^y \doteq \delta'/(4\max\{1, S_t(a)\}^2)$. Then,
    $
        \sum_{s\le t, a_s = a} f({x_{s, a_s}^\star}) \le \sum_{s\le t,a_s = a} y_s + \sqrt{2R^2 S_t(a)\log(1/\delta_{t,a}^y)}
    $
    holds simultaneously for all $t \ge 1$ with probability at least $1-\delta'$ (and so does the symmetric lower bound on $\sum_{s\le t,a_s = a} y_s$).
\end{restatable}
\vspace{-1mm}

\textbf{Proof outline.} Since $S_t(a)$ is random and adaptively determined, we index the noise of agent $a$ by its selection count: the noise terms at the successive selections of $a$ form a conditionally $R$-sub-Gaussian martingale difference sequence (optional skipping). We apply the sub-Gaussian (Hoeffding-type) concentration bound to each fixed count $n$ with confidence $\delta'/(4n^2)$ and take a union bound over $n \ge 1$, which yields an anytime-valid bound; detailed and omitted proofs are provided in \cref{asec:proofs}.

\para{\step{3} Statistical test for finding whether the agent is over-reporting.}
For simplicity, consider the case where the reward function is known.
In this case, we say that an agent is over-reporting if the total expected reward for reported arm features exceeds the total noiseless expected reward, i.e.,  $\sum_{s=1,a_s = a}^{t} f({x_{s, a_s}}) > \sum_{s\le t,a_s = a} \bar{y}_s$ for some $t \ge 1$, where $\bar{y}_s$ is the noiseless expected reward.
However, since the reward function is unknown and the observed reward is noisy in practice, we assess over-reporting using optimistic and pessimistic estimates of the expected rewards. 
We define 
$
    \text{UCB}_{t,a}^{(y)} = \sum_{s=1,a_s = a}^t y_s + \sqrt{2R^2S_t(a)\log(1/\delta_{t,a}^{y})}
$
as the optimistic estimate of the sum of agent $a$'s expected rewards that holds with probability at least $1-\delta_{t,a}^{y}$. 
Thus, if the following condition holds, then agent $a$ is over-reporting its arm features with probability at least $1-\delta_{t,a}^{x}-\delta_{t,a}^{y}$:
\eq{ 
    \label{eqn:loom}
    \boxed{\textbf{LOOM Condition:}~~ \text{LCB}_{t,a}^{(x)} >\text{UCB}_{t,a}^{(y)}.}
}
By eliminating the agents who satisfy \cref{eqn:loom} from future rounds, the LOOM condition creates an incentive to report truthfully; the formal guarantees are \cref{thm:optgtmarmsstayactive} (no false elimination) and \cref{thm:regretAllNE} and \cref{thm:regretNonAllNE} (approximate incentive compatibility under \cref{assu:NEandregret}).
Our next result shows that when an agent $a$ always reports truthfully, i.e., $\ox_{t,a} = x_{t,a}^\star$ for all $t\ge 1$, it does not get eliminated with probability at least $1- \delta_{t,a}^{x} - \delta_{t,a}^{y}$ in a given round. Since elimination is permanent, we also need this uniformly over all rounds and agents: \cref{asec:schedule} gives a confidence schedule under which all confidence bounds hold simultaneously for all $t \ge 1$ and $a \in \cN$ on an event $\cE$ with $\Prob{\cE} \ge 1-\delta$, at the cost of only logarithmic factors.

\begin{restatable}{thm}{optgtmarmsstayactive}
    \label{thm:optgtmarmsstayactive}
    Suppose that agent $a$ and all other agents report truthfully. Then, with probability at least $1 - \delta_{t,a}^{x} - \delta_{t,a}^{y}$, LOOM does not eliminate agent $a$ in round $t$. Moreover, under the confidence schedule of \cref{asec:schedule}, with probability at least $1-\delta$, LOOM never eliminates any agent, i.e., $\cN_t = \cN_1$ for all $t \ge 1$. 
\end{restatable}
\vspace{-1mm}

\textbf{Proof outline.}
The key idea of the proof is to apply the validity of the leave-one-out lower confidence bound (\cref{defi:loomCompatible}(ii), which in the linear case follows from the confidence ellipsoid) alongside high-probability upper bounds on the noisy reward (\cref{lem:rewOptEst}). When all agents are truthful, both bounds bracket the same total expected reward, so the LOOM condition cannot hold. The uniform statement follows on the event $\cE$ (\cref{asec:proofs}).

The performance of LOOM depends on how well the remaining agents represent the distribution of an over-reporting agent's arm feature. This structural property is central to our setting, as it allows estimating an agent's expected reward using the past observations from other agents under the shared reward function $f$ (an illustrative example is provided in \cref{asec:loom}, along with discussion of LOOM, including its failure cases, such as heterogeneous agents, multiple strategic agents, and collusion, as well as its connections to existing work~\citep{buening2024strategic}). 
Since LOOM's estimation depends on the information provided by other agents, we next define the maximum advantage that any agent can have over the remaining agents.
\begin{defi}[Maximum Advantage for any Agent]
    \label{def:advantage}
    For an agent $a \in \cN$ and round $t \ge 1$, the advantage of agent $a$ at round $t$ is
    $
        C_{t,a} \doteq \sup_{x \in \cX} \nicefrac{h(x, \cO_{t,-a})}{h(x, \cO_t)}\ge 1.
    $
    The maximum advantage is any constant $C$ satisfying
    $
        C \ge \max_{1 \le t \le T, a \in \cN} C_{t,a} \ge 1,
    $
    which uniformly bounds this contribution across all agents and rounds.
\end{defi}

\vspace{-1mm}
The term $C_{t,a}$ measures the largest multiplicative increase in confidence width caused by removing agent $a$'s observations, and thus quantifies how much those observations contribute relative to the coverage provided by the remaining agents. 
Both $C_{t,a} \ge 1$ and $C \ge 1$ follow from $\cO_{t,-a} \subseteq \cO_t$, since removing agent $a$'s observations can only increase the confidence width. In the linear case of \cref{sec:cobra}, where $h(x, \cO_{t,-a}) = \alpha_{t,-a}\norm{x}_{V_{t,-a}^{-1}}$ and $h(x, \cO_t) = \alpha_t\norm{x}_{V_t^{-1}}$, since $\alpha_{t,-a} \le \alpha_t$, the width ratio is at most the ratio of the two design norms, and the analysis uses the latter, $C_{t,a}=\sup_{x \ne 0}
\nicefrac{\lVert x \rVert_{V_{t,-a}^{-1}}}
     {\lVert x \rVert_{V_t^{-1}}}$
(see \cref{asec:thm4proof}).
\vspace{-1mm}

    \section{Incentive-Compatible Contextual Bandit Algorithm: \algo}
    \label{sec:cobra}

\vspace{-1mm}
In this section, we present our contextual bandit algorithm, \algo{}, which is specifically designed to incentivize strategic agents to report truthfully. 
To present our key ideas and results, we restrict our setting to linear reward functions and later extend them to non-linear reward functions.

\para{Linear reward function.}
We first consider the setting where the underlying reward function is linear, i.e., $f(x) = \slin{x}$ in which $\theta_\star \in \R^d$ is the unknown parameter. 
At the beginning of round $t$, the learner observes the randomly 
generated context $c_t \in \cC$ and the set of reported arms $\cA_t$.
After selecting the arm $a_t$, the learner observes stochastic reward $y_{t} = \slin{x^\star_{t,a_t}} + \epsilon_t$, where $x^\star_{t,a_t} = \phi(c_t,a_t)$ is the true feature vector of the selected arm and $\epsilon_t$ is $R$-sub-Gaussian.
We estimate the unknown parameter $\theta_\star$ using the available observations of context-arm features and corresponding rewards at the beginning of round $t$, denoted by $\cO_t \doteq \{(x_{s, a_s}, y_s)\}_{s=1}^{t-1}$, as follows:
$
    \hat\theta_{t} \doteq V_{t}^{-1} \sum_{s=1}^{t-1} x_{s,a_s} y_s,$ \text{where } $V_{t} \doteq \lambda I_d + \sum_{s=1}^{t-1} x_{s,a_s}x_{s,a_s}^\top,
$
$I_d$ is the $d \times d$ identity matrix, and $\lambda>0$ ensures the Gram matrix $V_t$ is positive definite.

\para{Optimistic reward estimate.}
In round $t$, the optimistic reward estimate/ upper confidence bound (UCB) of any context-arm feature vector $x_{t,a}$ is computed as follows:
$
    \text{UCB}_{t}(x_{t,a}) \doteq \tlin{x_{t,a}} + \alpha_t \norm{x_{t,a}}_{V_t^{-1}},
$
where $\tlin{x_{t,a}}$ denotes 
the estimated reward for $x_{t,a}$ and $\alpha_t \norm{x_{t,a}}_{V_t^{-1}}$ is the confidence bonus in which $\alpha_t \doteq R\Lp{ d\log \Lp \frac{1+ tL^2/\lambda}{\delta}\Rp}\Rp^{\nicefrac{1}{2}} + \lambda^{\frac{1}{2}}S$ (with $\norm{x}_2 \le L$ and $\norm{\theta_\star}_2 \le S$; in the analysis, $\delta$ is replaced by $\delta_0$ from the confidence schedule in \cref{asec:schedule}) is an increasing function of $t$, and  $\norm{x_{t,a}}_{V_t^{-1}}$ (the weighted $l_2$-norm of $x_{t,a}$ with respect to $V_t^{-1}$) goes to zero as $t$ increases.

\vspace{-2mm}
\begin{algorithm}[!ht]
	\renewcommand{\thealgorithm}{{COBRA}}
	\floatname{algorithm}{}
	\caption{Algorithm for \textbf{CO}ntextual \textbf{B}andits with St\textbf{RA}tegic Agents}
	\label{alg:COBRA}
	\begin{algorithmic}[1]
		\STATE \textbf{Input:} $\cN_1$: set of agents before the round $t=1$, $\delta \in (0,1)$,  and $\lambda > 0$
        \FOR{$t=1, 2,\ldots$}
            \STATE Observe context $c_t$ and then receive set of arms $\cA_t$ reported by agents in $\cN_t$, with $\{\ox_{t,a}\}_{a \in \cA_t}$
            \STATE Set $\cB_t = \emptyset$ and select an arm $a_t = \argmax_{a \in \cA_t} \text{UCB}_t(x_{t,a})$ \label{step:max_selection}\WHILE{$\text{UCB}_t(x_{t,a_t}) > \text{UCB}_{t,-a_t}(x_{t,a_t})$ and $\cA_t \setminus (\cB_t \cup \{a_t\}) \neq \emptyset$} \label{step:check}
                \STATE Set $\cB_t = \cB_t \cup \{a_t\}$ and  select an arm $a_t = \argmax_{a \in \cA_t\setminus \cB_t} \text{UCB}_t(x_{t,a})$ 
            \ENDWHILE
            \STATE \textbf{if} $\text{UCB}_t(x_{t,a_t}) > \text{UCB}_{t,-a_t}(x_{t,a_t})$ (all arms fail the check), \textbf{then} set $a_t = \argmax_{a \in \cA_t} \text{UCB}_{t,-a}(x_{t,a})$ \label{step:fallback}
            \STATE Play arm $a_t$ and observe noisy reward  $y_t$. \label{step:selection}
            \STATE Check LOOM condition in \cref{eqn:loom} 
            for each agent in $\cN_t$. Update $\cN_{t+1} = \cN_t\setminus\{a \in \cN_t: \text{LOOM condition holds for } a\}$. If $\cN_{t+1} = \emptyset$, stop and receive $0$ reward thereafter. \label{step:loom}

		\ENDFOR
	\end{algorithmic}
\end{algorithm}

\para{UCB-based algorithm.} 
The upper confidence bound \citep{AISTATS11_chu2011contextual, ICML20_zhou2020neural,WWW10_li2010contextual} is a widely used technique for addressing the exploration-exploitation trade-off in contextual bandit problems. 
Our UCB-based algorithm, \algo{} (UCB), for linear contextual bandit problems works as follows. 
At the start of round $t$ (see \cref{fig:algos}),  the learner observes the context and reported arm features $x_{t,a}$, and then selects an arm $a_t = \argmax_{a \in \cA_t}\text{UCB}_{t}(x_{t,a})$ (Line \ref{step:max_selection}). 
Importantly, \algo{} (UCB) does not have access to the true arm features or the true reward function parameter $\theta_{\star}$. 
As a result, over-reporting by agents can lead \algo{} (UCB) to make suboptimal arm selections.
To address this, we introduce a validity check (Line~\ref{step:check}) that rejects the candidate arm $a_t$ when its optimistic estimate using all observations exceeds the one built without its own observations, $\text{UCB}_{t,-a}(x) \doteq \hat\theta_{t,-a}^\top x + \alpha_{t,-a}\norm{x}_{V_{t,-a}^{-1}}$ (with $\hat\theta_{t,-a}$, $V_{t,-a}$, and $\alpha_{t,-a}$ defined below), which is the case produced by significant over-reporting. The set $\cB_t$ of rejected arms is local to round $t$: it is reset to $\emptyset$ at Line~\ref{step:max_selection} of every round. If every arm in $\cA_t$ fails the check, \algo{} falls back to the leave-one-out selection rule in Line~\ref{step:fallback} (see \cref{asec:assumptions}).
We incorporate LOOM (Line \ref{step:loom}, more details on how we adapt LOOM to linear contextual bandits are provided below) to identify the over-reporting agent. Eliminating agents who satisfy the LOOM condition defined in \cref{eqn:loom} from future rounds creates an incentive to report truthfully, formalized below.

\para{LOOM in \algo{} (UCB).} To check the LOOM condition defined in \cref{eqn:loom}, we need to compute $\text{LCB}_{t,a}^{(x)}$, which requires estimating the reward function parameters using observations from all agents except agent~$a$. 
To construct the aforementioned estimate, we exclude the observations from agent $a$, which is given as follows:
$
    \hat\theta_{t,-a} = V_{t,-a}^{-1} \sum_{s=1,a_s \ne a}^{t-1} x_{s,a_s} y_s$, \text{ where } $V_{t,-a} = \lambda I_d + \sum_{s=1,a_s\ne a}^{t-1} x_{s,a_s}x_{s,a_s}^\top,
$
so that $\cO_{t,-a} \subseteq \cO_t$, as used in \cref{def:advantage} and in \cref{asec:proofs}.
We now formally define the pessimistic estimate of the total expected reward for an agent $a$ as: $\text{LCB}_{t,a}^{(x)} = \sum_{s=1,a_s = a}^t \text{LCB}_{t,-a}(x_{s,a_s})$, where $\text{LCB}_{t,-a}(x_{s,a_s})= x_{s,a_s}^\top \hat\theta_{t, -a} - \alpha_{t,-a} \norm{x_{s,a_s}}_{V_{t,-a}^{-1}}$ for $1 \le s \le t$, with $\alpha_{t,-a} = R\Lp{d\log \Lp\nicefrac{1+ (t-S_t(a))L^2/\lambda}{\delta}\Rp} \Rp^{\frac{1}{2}} + \lambda^{\frac{1}{2}}S$ and $S_t(a)$ denoting the number of times agent $a$ has been selected up to round $t$. 
Using the upper bound $\text{UCB}_{t,a}^{(y)} = \sum_{s \le t,a_s = a} y_s + \sqrt{2R^2S_t(a)\log(1/\delta_{t,a}^{y})}$ from \cref{lem:rewOptEst}, we can apply LOOM condition to identify whether any agent is over-reporting its arm features. 
\begin{figure}[!ht]
    \vspace{-6mm}
    \centering
	\subfloat[Standard Contextual Bandit Algorithm]{\label{fig:contx}
		\includegraphics[height=2.65cm]{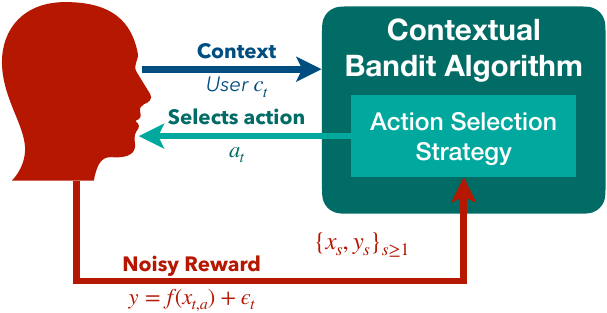}}
    \quad
	\subfloat[COBRA: Incentive Compatible Contextual Bandit Algorithm]{\label{fig:cobra}
		\includegraphics[height=2.65cm]{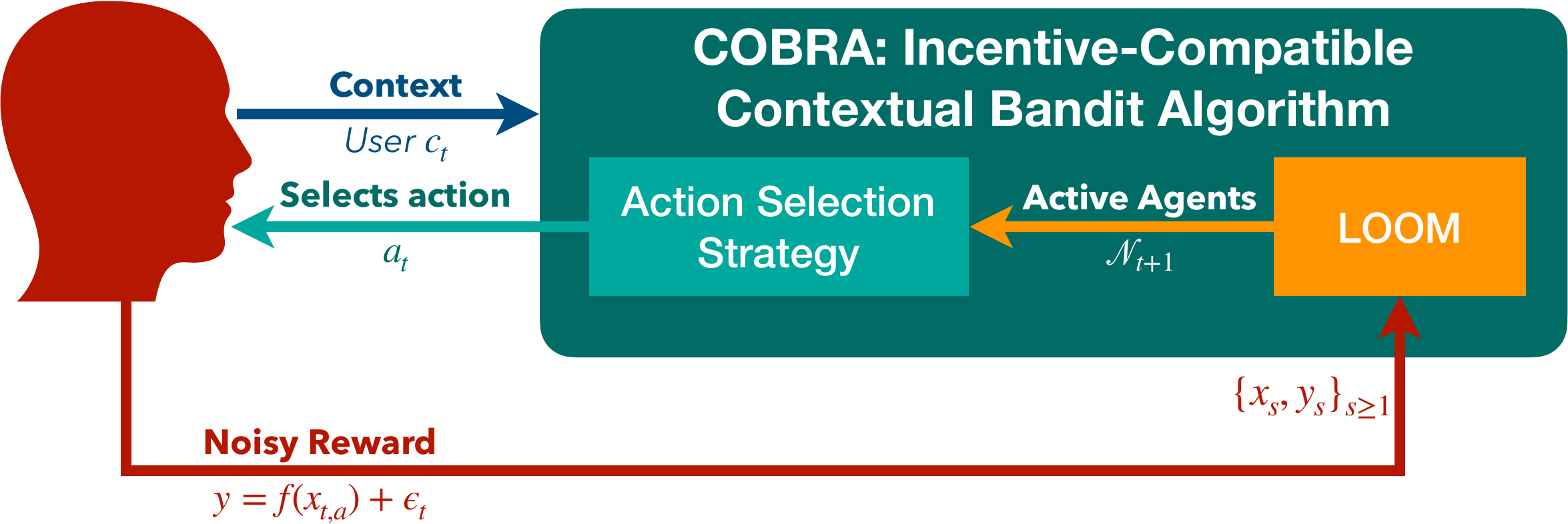}}
    \vspace{-1mm}
    \caption{
        \algo{} integrates LOOM as a post-processing step after each interaction round to identify over-reporting agents in a LOOM-compatible contextual bandit algorithm. The set $\cN_t$ of active agents is the output of LOOM in the previous round, and the observed reward depends on the \emph{true} features of the selected arm, $y_t = f(x^\star_{t,a_t}) + \epsilon_t$.
	}
	\label{fig:algos}
    \vspace{-3mm}
\end{figure}

\para{Non-linear reward function.}
We now consider contextual bandit problems with non-linear reward functions. \algo{} naturally generalizes to this setting by adopting any suitable {\em LOOM-compatible} contextual bandit algorithm to get optimistic reward estimates for context–arm feature vectors. These estimates are then used to select the best arm for a given context (as in Line~\ref{step:max_selection}), after which LOOM is applied as a post-processing step to identify over-reporting agents.

\para{NE guarantee and regret analysis.}
We derive NE and regret guarantees for \algo{} and establish its desirable properties, including approximate incentive compatibility (\cref{def:IC}) and a sublinear regret guarantee under truthful reporting. 
We assume that each agent only over-reports its arm features so that the corresponding reward is higher, i.e., for all $x^\star\in \cX:\ f(\ox) \ge f(\tx)$, where $x$ is the reported arm feature vector for the true arm feature $x^\star$. Notably, we impose no restrictions on how agents report their arm features, aside from the no collusion assumption, which is common in VCG-type mechanisms \citep{vickrey1961counterspeculation, clarke1971multipart, groves1973incentives}.
Let $\tilde{O}$ hide the logarithmic factors and universal constants. 
All results below hold on the event $\cE$ of the confidence schedule in \cref{asec:schedule}, i.e., with probability at least $1-\delta$. Since $0 \le S_T(a) \le T$, the complement of $\cE$ changes expected selection counts by at most $T\delta$, which is $O(1)$ for $\delta = 1/T$; this converts high-probability bounds on $S_T(a)$ into bounds on $\EE{S_T(a) \mid \bm{\sigma}}$ as required by \cref{def:epsNE}.

The selection counts (as opposed to regret) additionally use a uniform gap: there is $\Delta_{\min} > 0$ such that $f(x^\star_{t,a^\star_t}) - f(x^\star_{t,a}) \ge \Delta_{\min}$ for all $t \ge 1$ and $a \in \cN_t \setminus\{a^\star_t\}$. It holds, e.g., when contexts take finitely many values, each with a unique optimal agent, but generally fails for continuous contexts, in which case only the regret bounds apply. We keep $1/\Delta_{\min}$ explicit rather than hiding it in $\tilde{O}(\cdot)$.
Next, we present the results for \algo{} (UCB) under linear and non-linear reward functions.
\begin{restatable}[Linear]{thm}{regretNE}
	\label{thm:regretNE}
    Consider the truthful profile $\bm{\sigma}^\star$. With probability at least $1-\delta$, the regret of \algo{} is at most
    $
        \Regret_T(\text{\algo},\bm{\sigma}^\star) = \tilde{O}(d\sqrt{T}),
    $
    no agent is eliminated (\cref{thm:optgtmarmsstayactive}), and, under the gap condition, every agent $a$ is selected $S_T(a) \ge S^\star_T(a) - \tilde{O}(d\sqrt{T}/\Delta_{\min})$ times.
\end{restatable}
\vspace{-1mm}

\cref{thm:regretNE} concerns only the truthful profile: it shows that truthful agents are not penalized by \algo{} beyond estimation error. It does \emph{not} bound the gain of an agent that misreports while the others remain truthful, since then the full-data estimator is biased (\cref{lem:biasedEst}); such deviations are covered by \cref{thm:regretAllNE} and \cref{thm:regretNonAllNE} below, whose bound depends on $C$. Example~2 in \cref{asec:loom} shows that, without peer coverage, a unilateral over-reporting deviation can gain $\Theta(T)$ selections, consistent with $C = \Theta(\sqrt{T})$ there.

Let $\tilde{d}$ be the effective dimension associated with contextual bandit problems with non-linear reward functions. Let $\kA$ be a LOOM-compatible contextual bandit algorithm for which $|f_t(x) - f(x)| \le h(x, \cO_t)$ holds with probability at least $1-\delta$ for any $x \in \cX$ and $\sqrt{\sum_{t=1}^T \Lb h(x_{t,a_t}, \cO_t) \Rb^2 } = \tilde{O}(\tilde{d}\log T)$.
We assume that this bound holds for the algorithm $\kA$ used by \algo{}.

\begin{restatable}[Non-linear]{thm}{regretNonNE}
	\label{thm:regretNonNE}
    Let $\kA$ be a LOOM-compatible contextual bandit algorithm used by \algo{}. 
    Consider the truthful profile $\bm{\sigma}^\star$. With probability at least $1-\delta$, the regret of \algo{} is
    $
        \Regret_T\Lp\algo(\kA), \bm{\sigma^\star}\Rp = \tilde{O}(\tilde{d}\sqrt{T}),
    $
    no agent is eliminated, and, under the gap condition, every agent $a$ is selected $S_T(a) \ge S^\star_T(a) - \tilde{O}(\tilde{d}\sqrt{T}/\Delta_{\min})$ times.
\end{restatable}
\vspace{-1mm}

When an agent over-reports, all \algo{} estimators that use its reports become biased (\cref{lem:biasedEst} in the Appendix) because the over-reported arm features used to estimate the reward function no longer reflect the true distribution. Without a condition limiting the effect of this contamination, no sub-linear regret guarantee is possible (Example~2 in \cref{asec:loom} incurs $\Theta(T)$ regret). Our next results hold only under the conditions specified in the following assumption.

\begin{restatable}{assu}{Assumption}
\label{assu:NEandregret}
    Let $\ox$ and $\tx$ be the reported and true context-arm feature vector, respectively. Then, we assume $\forall t \ge 1, a \in \cA_t: f(\ox_{t,a}) \le \ucb_{t}(\ox_{t,a})$, where $ \ucb_t(x)= f_t(x) + h(x, \cO_t)$.
\end{restatable}
\vspace{-1mm}

The assumption states that each agent's expected true reward for the reported features is upper bounded by the optimistic reward estimate, $\ucb_{t}(\ox_{t,a})$, that uses all available context-arm features to estimate $f$.
\cref{assu:NEandregret} is a condition on the \emph{interaction}, not on the problem primitives, since $f_t$ is built from manipulated data. Its status is as follows. (i) Under truthful reporting, it holds with probability at least $1-\delta$ (\cref{lem:linConfidenceBound}). (ii) Under over-reporting, a sufficient primitive condition is a bounded total over-reporting gap together with a correspondingly inflated confidence width (\cref{prop:assuSufficient} in \cref{asec:assumptions}). (iii) It can fail: in Example~2 of \cref{asec:loom}, a persistent over-report keeps the bias of \cref{lem:biasedEst} from vanishing while the width shrinks, and $\ucb_t(\ox_{t,2}) \to 1/4 < 1 = f(\ox_{t,2})$. Hence \cref{thm:regretAllNE} and \cref{thm:regretNonAllNE} below are \emph{conditional} on \cref{assu:NEandregret}.
Next, we prove a strategic regret bound that holds for every over-reporting strategy profile under which \cref{assu:NEandregret} holds. The proof does not use the equilibrium property, so it covers adaptive and non-stationary strategies, and every NE whenever one exists.  \begin{restatable}[Linear]{thm}{regretAllNE}
	\label{thm:regretAllNE}
    Let $\bm{\sigma}$ be any strategy profile in which every agent over-reports, i.e., $f(\ox_{t,a}) \ge f(x^\star_{t,a})$ for all $t$ and $a$. If \cref{assu:NEandregret} holds then, with $C$ as in \cref{def:advantage}, with probability at least $1-\delta$, the regret of \algo{} is 
    $
        \Regret_T(\algo,\bm{\sigma}) = \tilde{O}(Cd\sqrt{T} + \sqrt{NT}).
    $
    In particular, this holds for every $\bm{\sigma} \in \neql(\algo)$ satisfying these conditions. Moreover, under the gap condition, \algo{} is $\epsilon$-IC (\cref{def:IC}) against unilateral over-reporting deviations $\sigma'_a$ under which \cref{assu:NEandregret} holds and no truthful agent is eliminated, with $\epsilon = \tilde{O}\big((Cd\sqrt{T} + \sqrt{NT})/\Delta_{\min}\big)$.
\end{restatable}
\vspace{-1mm}

Our next result extends the previous result to the general setting with non-linear reward functions.
\begin{restatable}[Non-linear]{thm}{regretNonAllNE}
	\label{thm:regretNonAllNE}
    Let $\kA$ be a LOOM-compatible contextual bandit algorithm used by \algo{}, and let $\bm{\sigma}$ be any strategy profile in which every agent over-reports.
    If \cref{assu:NEandregret} holds then, with $C$ as in \cref{def:advantage}, with probability at least $1-\delta$, the regret of \algo{} is 
    $
        \Regret_T(\algo(\kA),\bm{\sigma}) = \tilde{O}(C\tilde{d}\sqrt{T} + \sqrt{NT}).
    $
    In particular, this holds for every $\bm{\sigma} \in \neql(\algo(\kA))$ satisfying these conditions. Moreover, under the gap condition, \algo{}$(\kA)$ is $\epsilon$-IC against unilateral over-reporting deviations under which \cref{assu:NEandregret} holds and no truthful agent is eliminated, with $\epsilon = \tilde{O}\big((C\tilde{d}\sqrt{T} + \sqrt{NT})/\Delta_{\min}\big)$.
\end{restatable}
\vspace{-1mm}

\para{Outline of the proofs.} 
The proof of \cref{thm:regretAllNE} depends on the LOOM mechanism to identify agents who are over-reporting.
LOOM ensures that optimistic estimates are tightly bounded, thereby limiting the potential benefit from over-reporting and reinforcing truthfulness as an approximate best response for agents.
The $\sqrt{NT}$ term in \cref{thm:regretAllNE} arises due to the strategic nature of the agents who can exploit initial noisy estimates of \algo{}.
The detailed proofs are provided in \cref{asec:proofs}.
The proof of \cref{thm:regretNonNE} mirrors that of \cref{thm:regretNE} and does not involve $C$; the proof of \cref{thm:regretNonAllNE} relies on the LOOM mechanism to identify over-reporting agents (via $C$).
The remainder of the proof proceeds as before, making use of \cref{defi:loomCompatible} and \cref{assu:NEandregret}. 

\cref{thm:regretAllNE} and \cref{thm:regretNonAllNE} give sub-linear regret only when $C = o(\sqrt{T})$ and \cref{assu:NEandregret} holds. When peer coverage is poor, $C$ can be $\Theta(\sqrt{T})$ and the bounds become the trivial $\tilde{O}(dT)$; this degradation is unavoidable for leave-one-out detection (Example~2 in \cref{asec:loom}). 
The condition that no truthful agent is eliminated is needed because a deviator's reports enter the leave-one-out estimators used to test the other agents (\cref{asec:loom}).

    \section{Experiments}
    \label{sec:experiments}

In this section, we aim to corroborate our theoretical results and empirically demonstrate the performance of our proposed algorithm in different strategic contextual bandit problems.
We repeat all our experiments 20 times and show the regret (as defined in \cref{eq:regret}) with a 95\% confidence interval (the vertical line on each curve shows the confidence interval). To demonstrate the different performance aspects of our proposed algorithm, we have used different synthetic problem instances.

\para{Experiment setting.}
We use a $d_c$-dimensional space to generate the sample features of each context, where context $c_t$ is represented by $c_t = \Lp x_{c_t,1}, \ldots, x_{c_t,d_c} \Rp$ for $t \ge 1$.
Similarly, we use a $d_n$-dimensional space to generate the sample of each agent's arm features, where agent $a \in \cN$ is represented by $\Lp z_{t,1}^{(a)}, \ldots, z_{t,d_n}^{(a)} \Rp$. The value of the $i$-th feature $x_{c_t,i}$ (or $z_{t,i}^{(a)}$) is sampled uniformly at random from $\Lp 0, 2 \Rp$.
To get the context-arm feature vectors for context $c_t$ in round $t$, we concatenate the context features $c_t$ with each agent's arm feature vector. 
For context $c_t$ and agent $a$, the concatenated feature vector is denoted by $x_{t,a}$, which is a $d$-dimensional vector with $d = d_c + d_n$.
We sample a $d$-dimensional vector $\theta_\star$ uniformly from $(0,2)^d$ and normalize it to have unit $l_2$-norm. Across all experiments, we set $\lambda=0.01$, $R=0.1$, $\delta=0.05$, and $d_c=d_n$.
We have provided details about different strategic manipulations in \cref{asec:more_experiments}.

\para{Regret comparison with baselines.}
We compare the regret of the proposed algorithms with three baselines: {Lin-UCB}~\citep{WWW10_li2010contextual}, {Lin-TS}~\citep{ICML13_agrawal2013thompson}, and {OptGTM}~\citep{buening2024strategic}.  
\algo{} (TS) is identical to \algo{} (UCB) except for Line~\ref{step:max_selection}, where $\text{UCB}_t(x_{t,a})$ is replaced by $x_{t,a}^\top\tilde\theta_t$ with $\tilde\theta_t \sim \bN(\hat\theta_t, \beta_t^2 V_t^{-1})$ and $\beta_t = R\sqrt{9d\log(t/\delta)}$ \citep{ICML13_agrawal2013thompson}; the validity check and LOOM are unchanged (see \cref{asec:more_experiments}). Our theoretical guarantees are stated for \algo{} (UCB); \algo{} (TS) is an empirical variant.
For our experiments, we consider two reward functions: \textbf{Linear}, $f(x) = 5 x^\top \theta_\star$, and \textbf{Square}, $f(x) = 10 (x^\top \theta_\star)^2$. We use $1000$ contexts, $5$ agents, and $d_c = d_n = 5$, resulting in a context-arm feature dimension of $d = 10$.    
We evaluate four problem instances with the same setup, except using two different types of strategic over-reporting: \textbf{Agent type I:} Feature adaptation with a learning rate of $\nu = 0.05$. \textbf{Agent type II:} Systematic over-reporting, where $\hat{\Delta}_x \sim \bN(\Delta_x^\star, \sigma_\Delta^2)$. Here, $\bN$ denotes a normal distribution, $\Delta_x^\star$ is the optimal scaling factor such that $f((1 + \Delta_x^\star) x^\star)$ gives the highest reward among all arms, and $\sigma_\Delta$ represents the standard deviation. 
We assume that only one agent over-reports, and the maximum perturbation in each round is bounded by $\Delta_{\max} = 1.0$.

\begin{figure}[!ht]
	\vspace{-6mm}
    \centering
	\subfloat[Linear (Truthful setting)]{\label{fig:truthful}
		\includegraphics[width=0.245\linewidth]{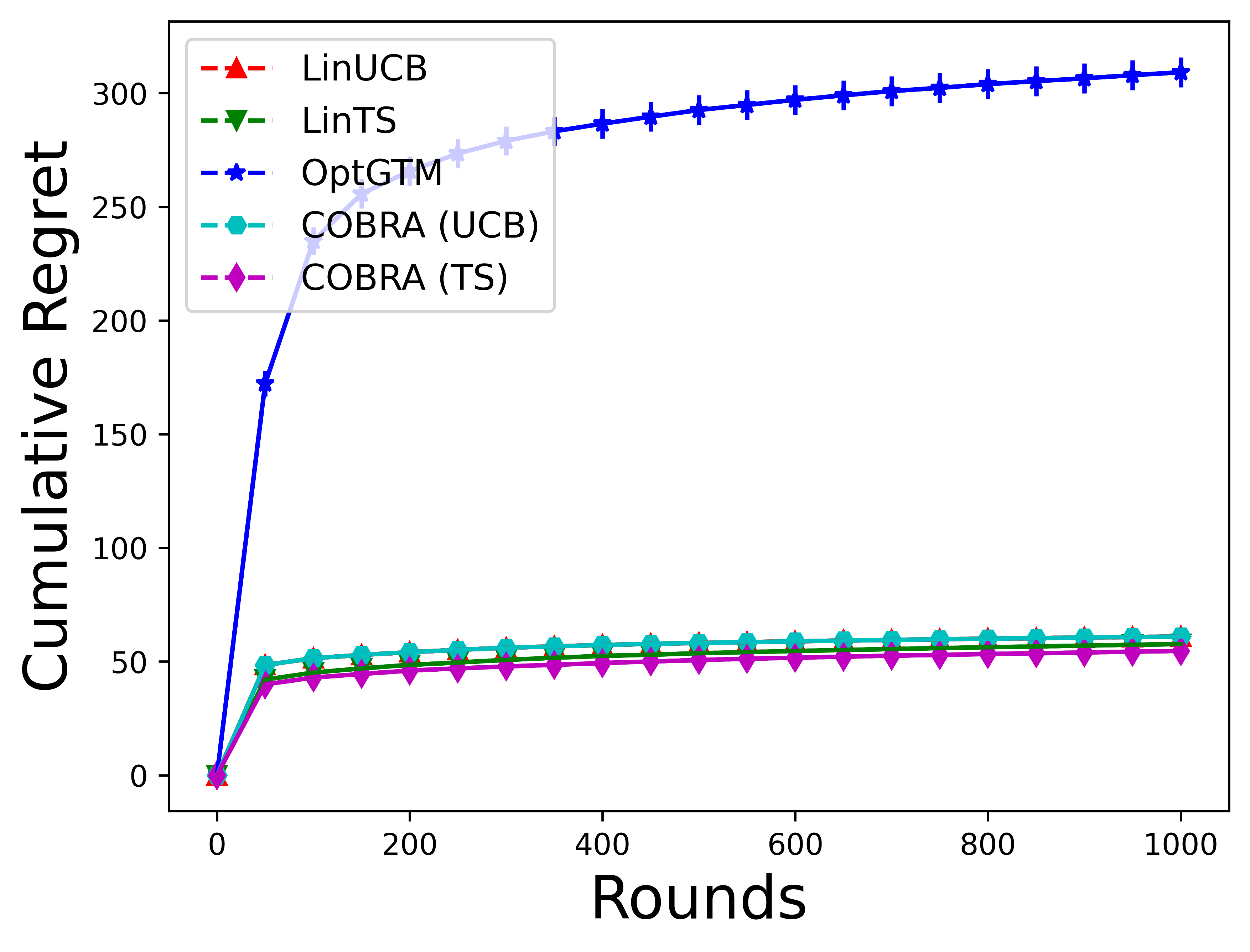}}
	\subfloat[Linear (Agent type I)]{\label{fig:prob1}
		\includegraphics[width=0.245\linewidth]{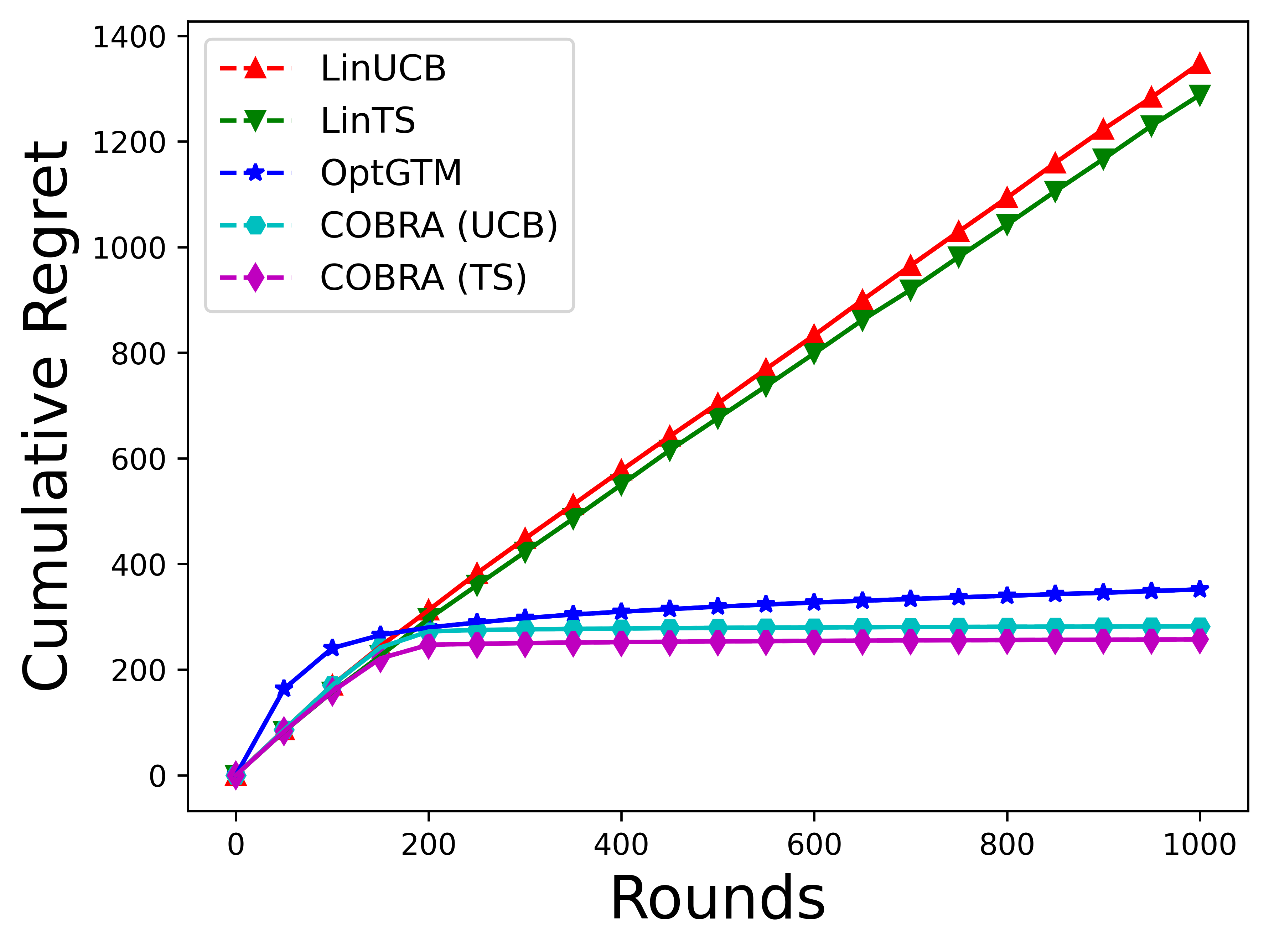}}
	\subfloat[Linear (Agent type II)]{\label{fig:prob2}
		\includegraphics[width=0.245\linewidth]{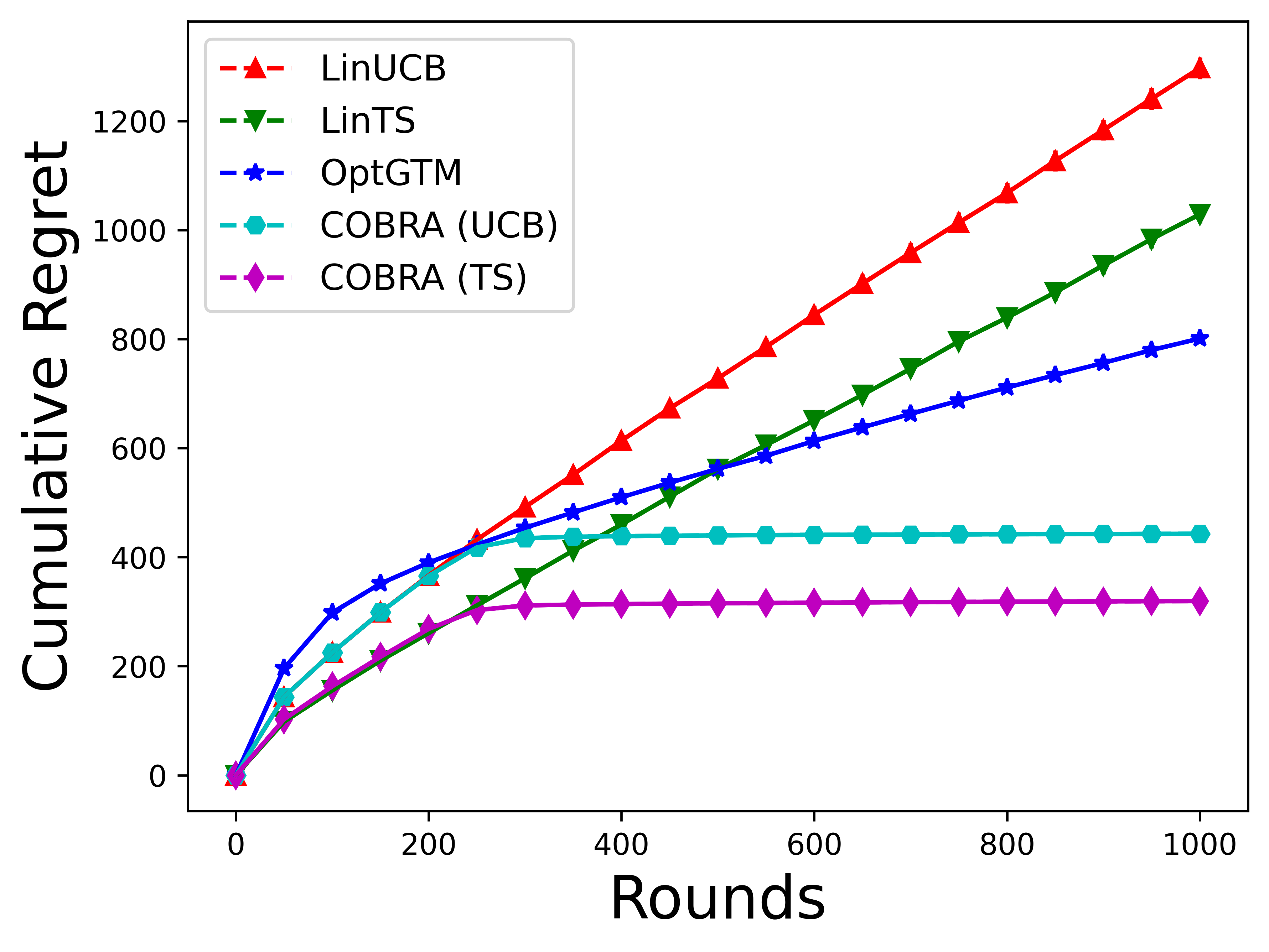}}
	\subfloat[Square (Agent type II)]{\label{fig:prob3}
		\includegraphics[width=0.245\linewidth]{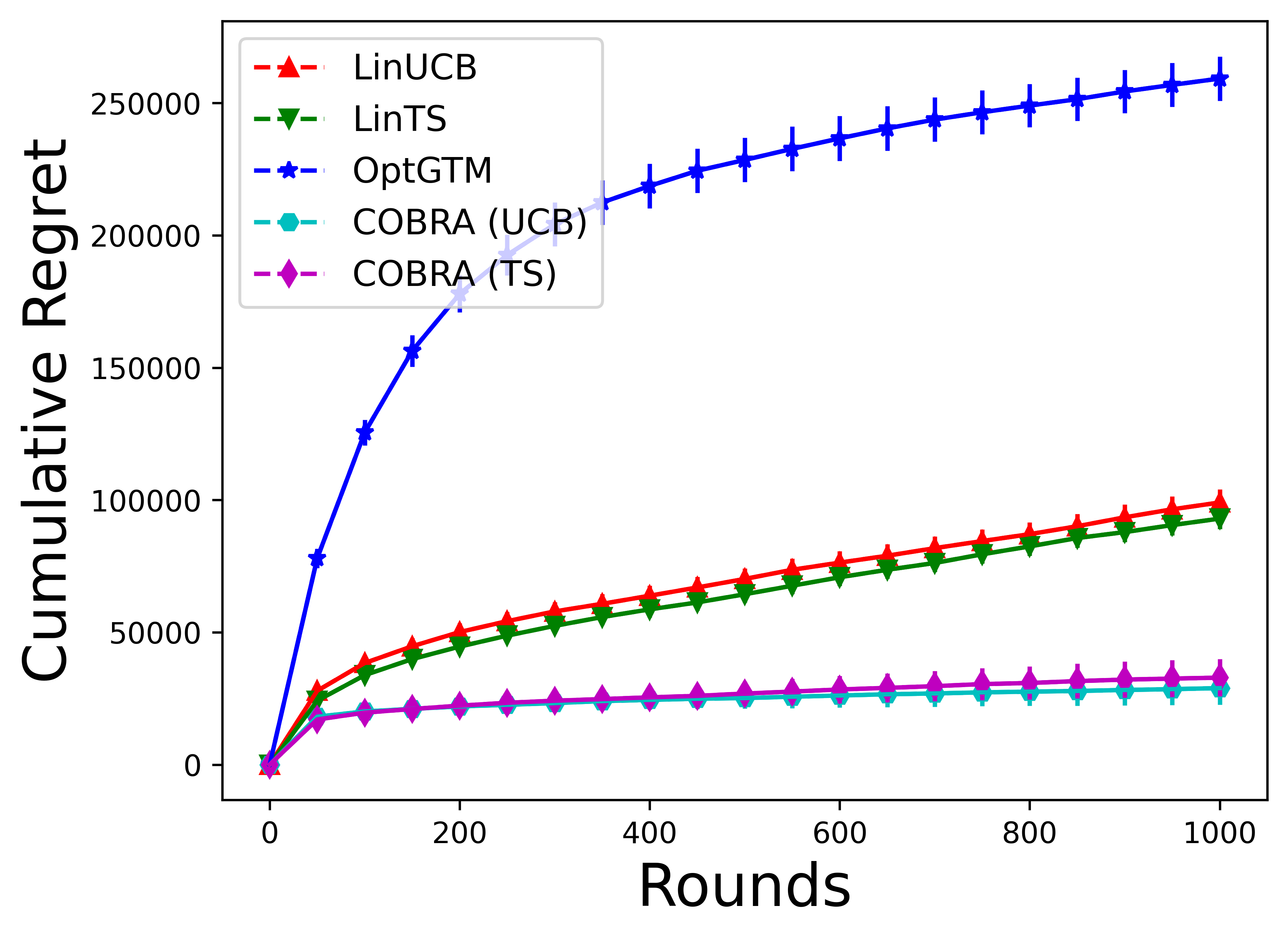}}
    \vspace{-2mm}
    \caption{
        Comparing cumulative regret of \algo{} with baselines using different problem instances.
	}
	\label{fig:compare_cumm_regret}
    \vspace{-3mm}
\end{figure}

In \cref{fig:truthful}, all agents report truthfully under the \textbf{Linear} reward function. Even in this setting, our algorithm outperforms the state-of-the-art {OptGTM} and matches the performance of standard contextual bandit algorithms ({Lin-UCB} and {Lin-TS}).  
As expected, our proposed algorithm \algo{}, based on UCB and TS variants of contextual bandits, outperforms all baselines across different problem instances with \textbf{Linear} and \textbf{Square} reward functions (\cref{fig:prob1}-\ref{fig:prob3}). 
For the \textbf{Square} reward function, we estimate it using kernel regression with a polynomial kernel of degree $2$. 
We observe that the TS-based variants of \algo{} consistently outperform their UCB-based counterparts.

\para{Regret of \algo{} vs. different types of strategic manipulations.}
To simulate the strategic manipulations, we define the over-report feature vector as $x = (1+\rho)x^\star$. Since all features and parameters ($\theta_\star$) are positive, scaling the true feature vector $x^\star$ by a factor of $(1+\rho)$ ensures over-reporting. To control the over-reporting, we uniformly sample $\rho$ from $(\eta,\eta + \epsilon_\eta)$, where changing $\eta$ or $\epsilon_\eta$ leads to different types of manipulations. Specifically, an increased $\eta$ or $\epsilon_\eta$ implies increasing the misreporting amount.
As shown in \cref{fig:eta_ucb}-\ref{fig:eta_noise_ts}, the regret of our \algo{} UCB- and TS-based algorithms increases as the manipulation level ($\eta$ or $\epsilon_\eta$) increases. However, in instances with higher manipulation, the \algo{} TS-based algorithm consistently outperforms the \algo{} UCB-based algorithm.
Additional experimental results and ablations are provided in \cref{asec:more_experiments}.
\begin{figure}[!ht]
    \vspace{-6mm}
	\centering
	\subfloat[Varying $\eta$ (UCB)]{\label{fig:eta_ucb}
		\includegraphics[width=0.24\linewidth]{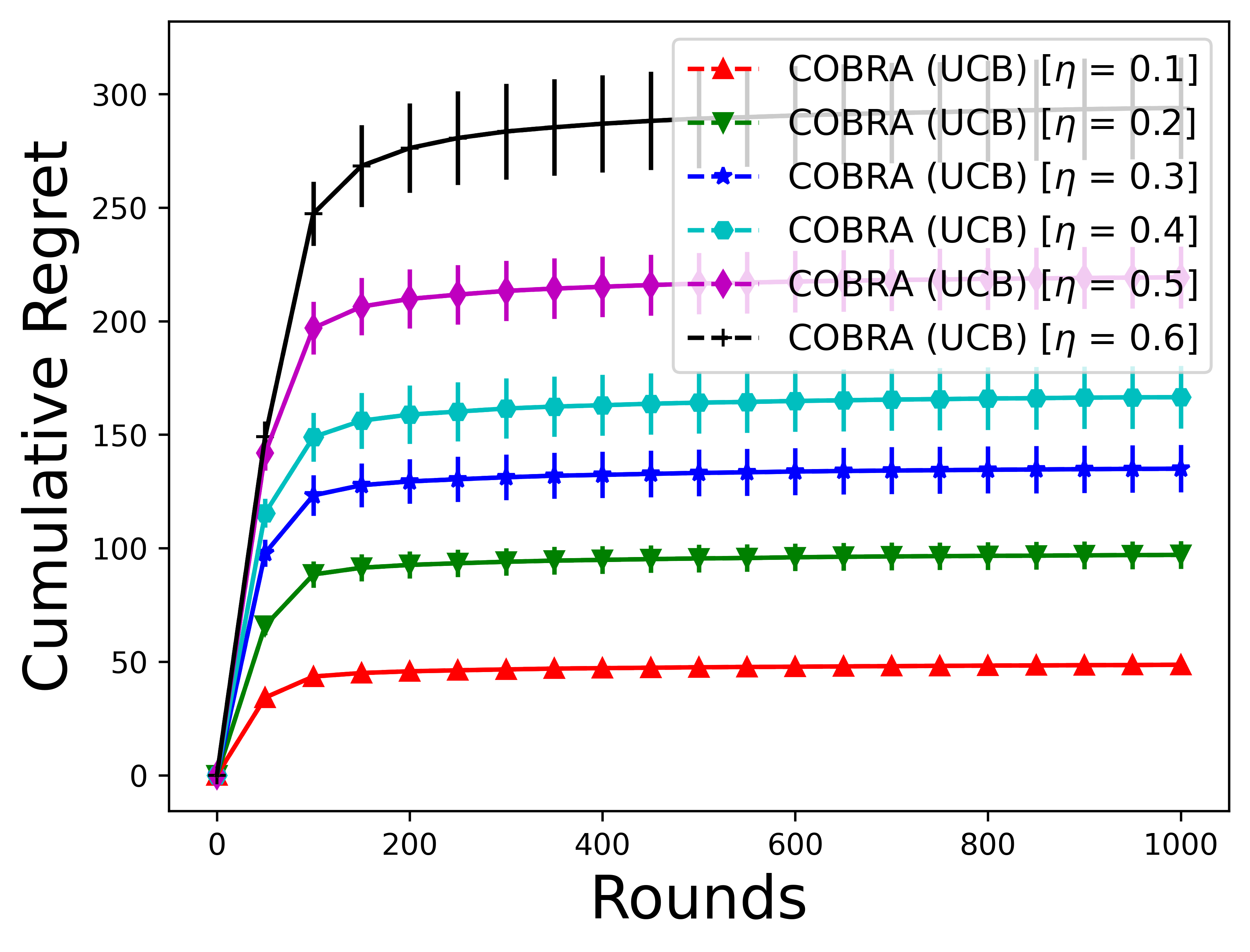}}
	\subfloat[Varying $\eta$ (TS)]{\label{fig:eta_ts}
		\includegraphics[width=0.24\linewidth]{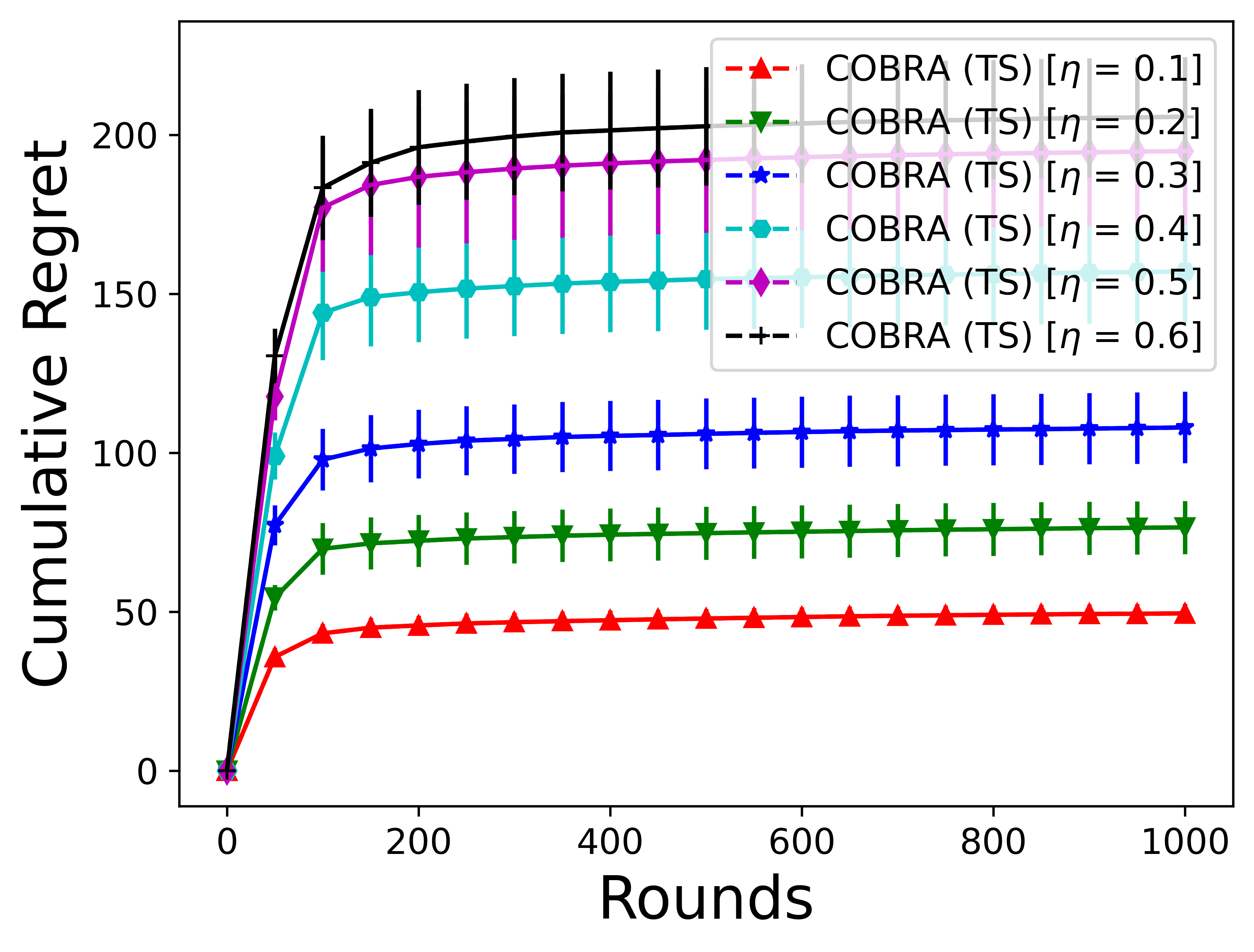}}
	\subfloat[Varying $\epsilon_\eta$ (UCB)]{\label{fig:eta_noise_ucb}
		\includegraphics[width=0.24\linewidth]{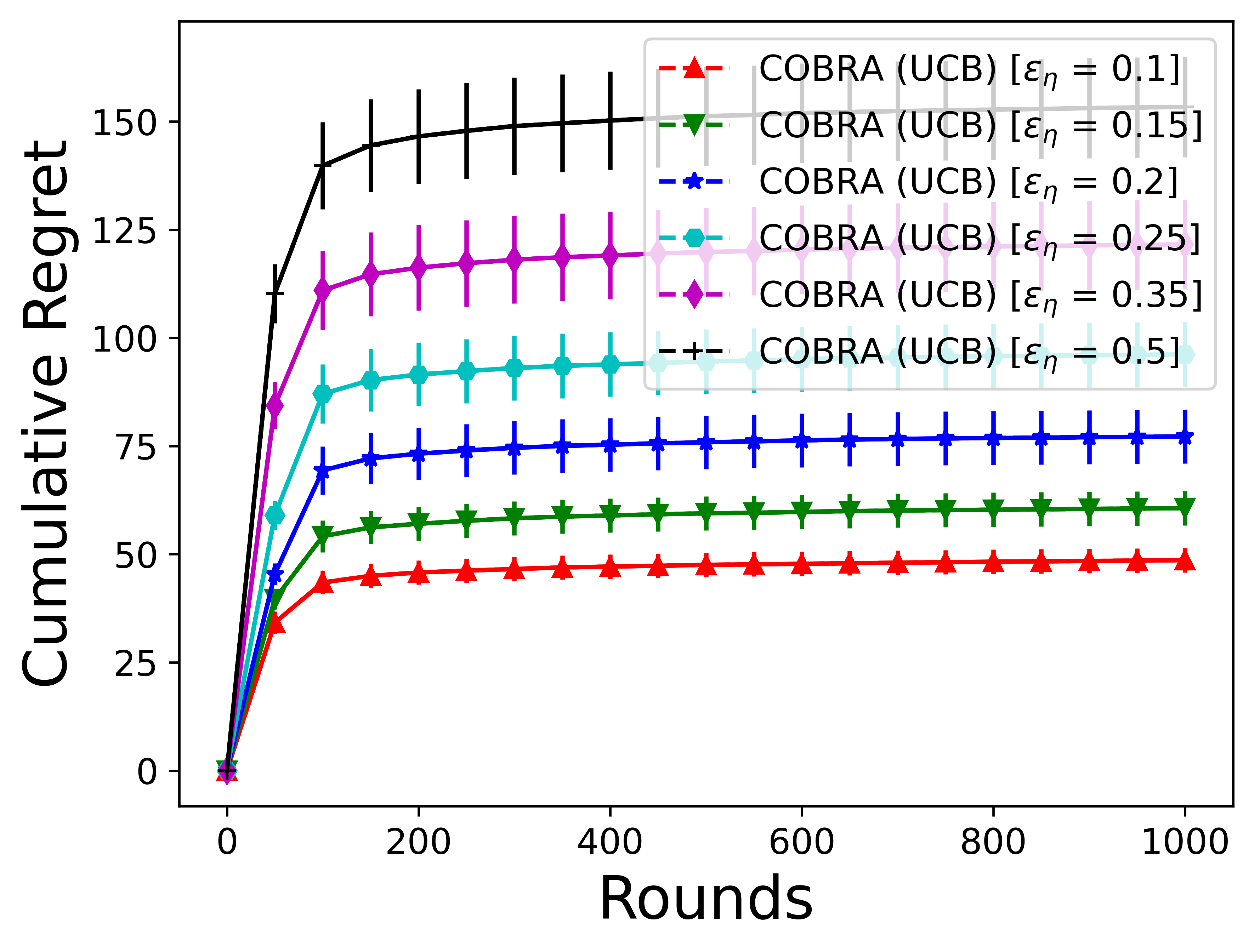}} 
    \subfloat[Varying $\epsilon_\eta$ (TS)]{\label{fig:eta_noise_ts}
		\includegraphics[width=0.24\linewidth]{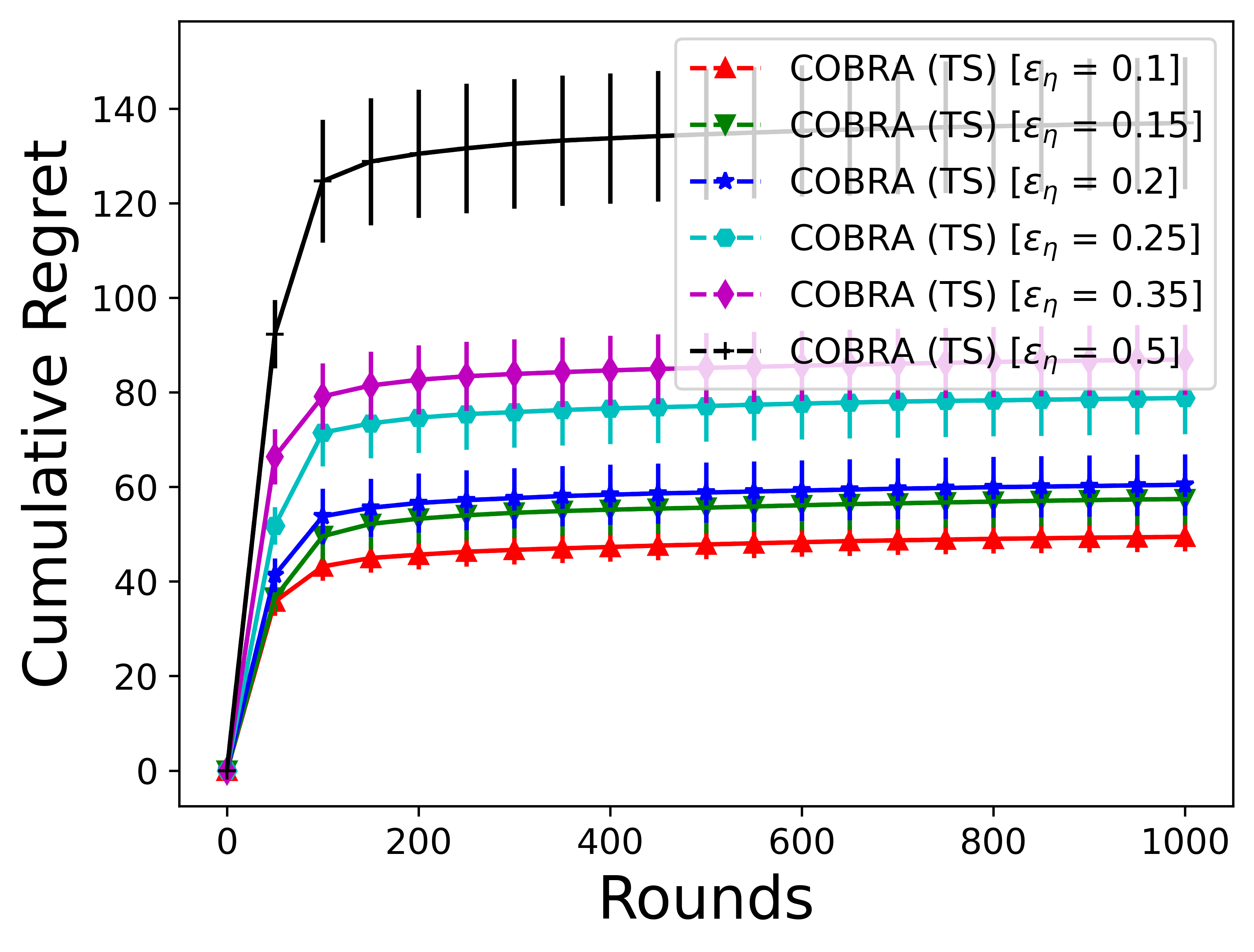}}
    \vspace{-2mm}    
    \caption{
       Cumulative regret of \algo{} vs. different types of strategic manipulations.
	}
	\label{fig:compare_strategic_manipulations}
    \vspace{-4mm}
\end{figure}

Due to space constraints, additional experiments on regret versus number of agents, context-arm feature dimension, two strategic agents, and non-linear reward functions are provided in \cref{asec:more_experiments}.

    \section{Conclusion}
    \label{sec:conclusion}

\vspace{-2mm}

This paper addresses a contextual bandit problem involving strategic agents who may misreport arm features to increase their own utility. To tackle this, we propose \textsc{LOOM}, a mechanism that identifies over-reporting agents by leveraging the reported arm features from other agents. Building on \textsc{LOOM}, we introduce an algorithm, \algo{}, for contextual bandit problems with strategic agents. \algo{} disincentivizes strategic behavior without relying on monetary incentives. Under truthful reporting, it achieves sub-linear regret and, with high probability, never eliminates a truthful agent. Under over-reporting, its regret and the gain of a deviating agent are bounded conditionally on \cref{assu:NEandregret}, with an explicit dependence on the peer-coverage constant $C$ that we show cannot be removed for leave-one-out detection. Our experimental results across different problem instances further demonstrate the performance advantages of our algorithm.
A few promising directions for future work include replacing \cref{assu:NEandregret} and $C$ by verifiable, instance-level conditions, incorporating fairness constraints into the arm selection process, and handling more complex forms of strategic behavior within this setting.

    \subsubsection*{Acknowledgments}
    This research is supported by the National Research Foundation (NRF), Prime Minister’s Office, Singapore under its Campus for Research Excellence and Technological Enterprise (CREATE) programme. The Mens, Manus, and Machina (M3S) is an interdisciplinary research group (IRG) of the Singapore MIT Alliance for Research and Technology (SMART) centre. The authors affiliated with Kyushu University are partially supported by JSPS KAKENHI Grant Number JP21H04979 and JST ERATO Grant Number JPMJER2301.


\subsection*{AI Use Statement}
In this work, we used generative AI tools solely to assist with writing and language refinement, specifically, improving the clarity, grammar, phrasing, and readability of author-drafted text. All technical content, including the problem formulation, methodology, theoretical results and proofs, experimental design, and analysis, was conceived, developed, and written by the authors. AI tools were not used to generate research ideas, derive or verify proofs, produce experimental results, write or debug code, conduct the literature review, or create any figures, tables, or data. Any such tasks requiring disclosure under the policy are therefore not applicable to this work.

We have reviewed all AI-assisted text. Every passage edited or rephrased with AI assistance was checked by the authors for factual accuracy, correct representation of our claims, and consistency with the underlying technical content, and was revised as needed. We take full responsibility for the final content of this work, including all text, claims, and artifacts.

\subsection*{Ethics Statement} 
This work is theoretical, focusing on the design and analysis of algorithms. The proposed methods do not directly involve human subjects, personal data, or real-world deployments. While the framework could potentially be applied in systems that interact with users, we emphasize that ethical considerations, such as fairness, privacy, and informed consent, must be addressed in practical deployments. Our primary goal is to advance the theoretical understanding of incentive-compatible contextual bandit algorithms, and we do not anticipate any  negative societal impacts.

\subsection*{Reproducibility Statement}  
This paper presents theoretical results, including formal proofs of approximate incentive compatibility and regret upper bounds. All assumptions, definitions, and derivations are stated explicitly in the main text (see \cref{sec:cobra} and the Appendix). The details of our experimental setup are provided in \cref{sec:experiments} and the Appendix. 
Additionally, the code used in our experiments has been included in the supplementary material, enabling full reproduction of the results reported in this paper.

    \bibliographystyle{iclr2027_conference}
	\bibliography{references}

    \clearpage
    \appendix

\section{Additional related work}
\label{asec:related_work}

\para{Contextual bandits.}
Contextual bandits \citep{NOW19_slivkins2019introduction, Book_lattimore2020bandit} have many real-life applications, such as online recommendations, advertising, web search, and e-commerce. In this framework, a learner selects an arm and receives a reward for that choice. Given the potentially large or infinite set of arms, the mean reward for each arm is typically modeled as an unknown function, which may be linear \citep{WWW10_li2010contextual, AISTATS11_chu2011contextual, NIPS11_abbasi2011improved, ICML13_agrawal2013thompson}, generalized linear model (GLM) \citep{NIPS10_filippi2010parametric, ICML17_li2017provably, NIPS17_jun2017scalable, ICLR25_verma2025neural}, or non-linear \citep{UAI13_valko2013finite, ICML17_chowdhury2017kernelized, ICML20_zhou2020neural, ICLR21_zhang2020neural}. 
The learner's objective is to identify the optimal action as efficiently as possible, which depends on how tightly the confidence bounds for the reward-function mapping actions to rewards are defined. Several works have explored various sources of information and side observations to enhance the learning process \citep{WWW10_li2010contextual, ICML13_agrawal2013thompson, COLT15_alon2015online, NIPS15_wu2015online, ICML17_li2017provably, NeurIPS21_verma2021stochastic, NeurIPS23_verma2024exploiting}.

\para{Strategic learning.}
There are several works on strategic learning \citep{liu2016bandit, freeman2020no, gast2020linear, zhang2021incentive, harris2022strategic, harris2023strategic} and strategic classification \citep{hardt2016strategic, dong2018strategic, sundaram2023pac}. The strategic classification problem was first introduced by \citet{hardt2016strategic}.
The authors considered a sequential game between a decision-maker selecting a classifier, and a strategic agent who responds by modifying their features.
\cite{chang2024s} address the problem of identifying agents who exhibit the highest degree of strategic manipulation in their inputs, given a dataset of agents and their observed model inputs in an offline setting. 
Our work aligns with this research direction, as it explores the interaction between a strategic agent and a learning algorithm. However, unlike prior studies where agents interact with the learner only once to achieve a desired outcome, our setting involves repeated interactions, forming a repeated game without monetary transactions.
Our main contribution is the development of an approximately incentive-compatible mechanism designed to handle repeated interactions with strategic agents, specifically tailored for contextual bandit problems.

    \section{Discussion about LOOM and Assumption \texorpdfstring{\ref{assu:NEandregret}}{2}}
    \label{asec:disscussion}

\subsection{LOOM-related Discussion}
\label{asec:loom}

\para{Connection between LOOM and VCG mechanism.}
The connection we intend is \emph{structural}, rather than monetary, and we state it precisely to avoid over-claiming. In VCG~\citep{vickrey1961counterspeculation, clarke1971multipart, groves1973incentives}, agent $a$ is evaluated against the outcome computed \emph{without} $a$, preventing $a$ from influencing the quantity used to determine its payment. LOOM adopts this same agent-excluding construction: it evaluates agent $a$ against $\lcb_{t,-a}$, which is constructed from $\cO_{t,-a}$, the observations of all agents except $a$, for the same reason (see \cref{asec:assumptions}). Beyond this exclusion principle, however, the mechanisms differ. LOOM involves no monetary transfers or externality pricing; instead, it creates incentives through \emph{elimination}, and its leave-one-out quantity is a learned statistical estimate rather than an exactly computed welfare measure. Moreover, VCG is dominant-strategy incentive compatible, whereas \algo{} only guarantees that truthfulness is an approximate Nash equilibrium (\cref{def:IC}), with an $\epsilon$ that depends on peer coverage through $C$ and that is established only under \cref{assu:NEandregret} (\cref{thm:regretAllNE} and \cref{thm:regretNonAllNE}). Thus, LOOM is best viewed as a leave-one-out, externality-\emph{style} test that adopts VCG's exclusion principle without its payment mechanism or its incentive guarantee; ``inspired by VCG'' should be read in this narrow sense.

Since strategies are arbitrary history-dependent maps, an agent that observes $c_t$ fully is the strongest adversary covered by our analysis, and the guarantees apply equally to agents with less information.

\para{Homogeneous setting.}
This work considers a homogeneous multi-agent setting where all agents share a common reward function $f$ (\cref{assu:homogeneous}). Whether each agent's arm feature directions are covered by those of the remaining agents is a separate question, quantified by $C$ (\cref{def:advantage}); homogeneity alone does not guarantee it, as Example~2 below shows.
Our proposed LOOM mechanism leverages the reports and rewards of other agents to construct a baseline for detecting misreports. In a heterogeneous setting, where agents may have different reward functions or non-overlapping arm feature space, the pessimistic estimate $\lcb_{t,-a}$ derived from other agents may not accurately bound the expected reward of agent $a$. Consequently, LOOM may become less effective or unable to reliably detect strategic misreporting in such scenarios. In \cref{asec:more_experiments}, we conduct additional experimental studies to illustrate these issues
 and also see the example below for a motivating practical example of LOOM.

\para{Impact of arm feature distribution on LOOM identification.}
The performance of LOOM depends on how well the remaining agents represent the distribution of an over-reporting agent's arm feature, as captured by the term $h(x, \cO_{t,-a})$ in $\text{LCB}_{t,-a}(x)$. 
If the over-reporting agent's distribution is well represented, $h(x, \cO_{t,-a})$ is going to be small, making the over-reporting agent easier to identify; otherwise, a large $h(x, \cO_{t,-a})$ makes identification harder. 
This structural property is central to our setting, as it allows estimating an agent's expected reward using the past observations from other agents under the shared reward function $f$ (see the example below).

\para{Example showing impact of arm feature distribution on LOOM.}
To illustrate how the arm feature distribution of agents plays out in our setting, consider the example of an online e-commerce platform recommending sellers. When multiple sellers offer similar products, such as round-neck T-shirts priced between \$5 and \$15, their corresponding arms will have similar feature vectors. In contrast, sellers offering distinct products, such as one selling T-shirts and another selling smartphones, will likely have arms with very different feature representations. Suppose a seller offers a unique product, for example, being the only seller of Apple products on the online e-commerce platform. Such a seller has no incentive to misreport its features, as no competitors exist. Thus, misreporting becomes strategically beneficial only when sellers offer similar products, in which case their arm features are drawn from similar distributions, allowing LOOM to identify the misreporting agent.

\para{Effect of a deviator on the test of other agents.} \cref{thm:optgtmarmsstayactive} guarantees that truthful agents are not eliminated when \emph{all} agents are truthful. When agent $b$ deviates, its reports enter $\cO_{t,-a}$ for every $a \ne b$, so the leave-one-out estimator $f_{t,-a}$ used to test a truthful agent $a$ is also contaminated. In the linear case, the resulting bias of $\hat\theta_{t,-a}$ is $-V_{t,-a}^{-1}\sum_{s:\,a_s = b} x_{s,b}\eta_s$ (\cref{lem:biasedEst}). For over-reporting ($\eta_s \ge 0$), this bias lowers $\lcb^{(x)}_{t,a}$ and cannot cause a false elimination whenever $x^\top V_{t,-a}^{-1} x_{s,b} \ge 0$; in general, however, $x^\top V_{t,-a}^{-1} x_{s,b}$ can be negative (even for nonnegative features, since $V_{t,-a}^{-1}$ may have negative off-diagonal entries), and an \emph{under}-reporting agent ($\eta_s < 0$, undetectable by LOOM) inflates $f_{t,-a}$. In such cases, a truthful agent $a$ could trip the LOOM condition and be eliminated, which could increase the deviator's selections. Our $\epsilon$-IC statements in \cref{thm:regretAllNE} and \cref{thm:regretNonAllNE} therefore explicitly exclude deviations under which a truthful agent is eliminated. Ruling out such ``sabotage'' deviations, e.g., by requiring a minimum fraction of truthful agents or by testing against estimators that exclude any agent whose reports are statistically inconsistent, is an open problem.

\para{Agent under-reporting.}
Agents have no incentive to under-report, as it reduces their likelihood of being selected by the algorithm. Instead, they are more inclined to over-report to increase their chances of being selected. 
However, as noted by \citet{buening2024strategic}, there are some cases where under-reporting may yield a small gain. 
Our proposed method, LOOM, is designed to detect over-reporting and does not capture under-reporting.
Developing a mechanism that can reliably detect both under-reporting and over-reporting remains an open problem.

\para{Failure Examples of LOOM.}
We consider two examples in which our mechanism, LOOM, is ineffective at detecting the misreport. 

\para{Example 1 (Unidentifiable arm direction).}
Let $\theta_\star=(1,1)$ and consider two arms with true contexts $x^\star_{t,1} = (1/4,\, 0)$ and $x^\star_{t,2} = (0,\, v_t)$, where $v_t \sim \mathrm{Uniform}(0,1)$ i.i.d. Samples from arm~1 follow a Gaussian with mean $1/4$, while no other agent provides observations in arm~2's feature direction $(0,1)$. Suppose agent~2 persistently misreports $x_{t,2} = (0,1)$. Since $\cO_{t,-2}$ contains no samples in this direction, LOOM cannot detect the misreport, and the learner incorrectly estimates arm~2's expected reward as $\EE{v_t} = 1/2 > 1/4$. 
Consequently, the learner plays arm~2 in every round, whereas under the true contexts it would prefer arm~1 whenever $v_t < 1/4$. This misreport goes undetected and induces linear regret. 
The same failure mode arises if the two arms' feature vectors are swapped with any fixed positive probability at each round.

\para{Example 2 (Collinear arm contexts).}
Let $\theta_{\star} = 1$, $x^\star_{t,1} = 1/2$, and $x^\star_{t,2} = 1/4$ for all $t$. Suppose agent~2 misreports $x_{t,2} = 1$. The learner always selects arm~2 because it achieves the highest UCB across the reported contexts. Since arm~1 is never played, the learner collects no samples from it and cannot detect the misreport.
Concretely, $\cO_{t,-2} = \emptyset$, so $\lcb^{(x)}_{t,2} < 0 < \ucb^{(y)}_{t,2}$ and LOOM never fires, while the ridge estimate $\hat\theta_t$ converges to $1/4$, exactly the bias predicted by \cref{lem:biasedEst}. \algo{} therefore incurs regret $1/4$ per round, i.e., $\Theta(T)$ in total, and agent~2, which is suboptimal in every round and would be selected only $\tilde{O}(\sqrt{T})$ times under truthful reporting, gains $\Theta(T)$ selections by this unilateral deviation. \cref{assu:NEandregret} fails here: $f(x_{t,2}) = 1$ while $\ucb_t(1) \to 1/4$.

\para{What the failure examples cost quantitatively.}
The above two examples characterize the regime in which the advantage constant $C$ in \cref{def:advantage} becomes large. In Example~2 ($d=1$, agent~2 misreports and is selected in every round), we have $V_{t,-2}=\lambda$ and $V_t=\lambda+t-1$, so
$
    C_{t,2} = \sqrt{\frac{\norm{x}^2_{V_{t,-2}^{-1}}}{\norm{x}^2_{V_t^{-1}}}} = \sqrt{\frac{\lambda+t-1}{\lambda}} = \Theta(\sqrt{t}),
$
implying that any valid $C$ is $\Theta(\sqrt{T})$ and the bound in \cref{thm:regretAllNE} degrades to the vacuous $\tilde{O}(dT)$. This illustrates why we keep $C$ explicit in the theorem statement: the guarantee degrades continuously with the severity of the coverage gap and can become vacuous when peer coverage is insufficient for detection.

An agent that concentrates its over-reporting in directions never explored by the remaining agents is undetectable by any peer-based test, including LOOM: without peer coverage, there is no independent signal against which to test the report. We state this limitation explicitly. It is illustrated by Examples 1 and 2 and is captured by a large confidence width $h(x,\cO_{t,-a})$.
Two observations qualify this limitation. First, the loss of coverage is continuous and quantifiable. The same quantity that captures this blind spot is the advantage constant $C$ in \cref{def:advantage}, which appears explicitly in \cref{thm:regretAllNE} and \cref{thm:regretNonAllNE}. As peer coverage deteriorates, $C$ increases and the corresponding guarantee degrades smoothly; in Example~2, $C=\Theta(\sqrt{T})$, causing the bound to degrade to the trivial $\tilde{O}(dT)$.
Second, the incentive to exploit a weakly covered direction depends on competition. Over-reporting is beneficial when it changes the agent's selection relative to its competitors. Thus, an agent operating in a direction that no other agent covers may have limited incentive to misreport, since there is no competitor to displace. The above example about the seller illustrates this distinction: a sole supplier of a product category has no rival against which misreporting can improve its selection, whereas sellers of similar products, with similar feature directions, provide precisely the peer coverage that LOOM exploits.
This tension does not eliminate the coverage gap. An agent may deliberately target weakly covered directions while still competing for selection, which remains outside our guarantees. Designing a mechanism with power against such coverage-aware deviations is an open problem.

\para{Alternative to complete removal of over-reporting agent.}
When an agent over-reports its arm features and is completely eliminated, such an event would be extremely rare in practice. If every agent were to persistently misreport, any mechanism prioritizing user experience would eventually stop recommending such agents, thereby driving the platform's reward to zero. This outcome is not a flaw but rather a \emph{safeguard preserving user trust}: continuing to recommend only strategic, untrustworthy agents would ultimately diminish both reward and user engagement.  
An alternative mechanism is that when an agent systematically over-reports, \algo{} flags and temporarily removes it. This mechanism discourages agents from gaming the system while restoring recommendation quality and benefiting users and the platform.
However, deriving theoretical guarantees for such a mechanism may be non-trivial.

\para{Agents with multiple arms.}
Our results generalize to the setting where each agent controls multiple arms. LOOM maintains individual records of each agent's reported arm feature vectors and reward history for their respective arms and applies LOOM's statistical test independently to each agent. This allows LOOM to identify over-reporting agents whose optimistic/pessimistic reward gaps exceed the threshold, even if the agent over-reports its different arms in distinct ways.
As long as \cref{assu:NEandregret} holds for each active agent, the regret and approximate equilibrium guarantees continue to apply. Specifically, the regret bound remains $\tilde{O}(Cd\sqrt{T}+\sqrt{NT})$, where $N$ denotes the number of agents. 
For a very large number of agents ($N$), it may increase computational complexity and affect regret bounds due to the $\sqrt{N}$ term. Note that we can perform all agent-related computations in parallel, as they are independent of each other. 

\para{Sub-optimal agent.}
We highlight that, for agents with suboptimal arms, truthful reporting and over-reporting may lead to similar outcomes: either being ignored by the learner or being eliminated. To address this challenge, a promising direction for future work is to incorporate fairness constraints into the arm selection, thereby ensuring that even suboptimal agents have a chance of being selected.

\para{Colluding agents.}
In the paper, we assume that agents choose their reporting strategies independently. For agent $a \in \cN$, deviations are restricted to unilateral best responses to $\bm{\sigma}_{-a}$, so the only deviations considered are of the form $(\sigma_a^\prime,\bm{\sigma}_{-a})$. In particular, no coalition $\cK \subseteq \cN$ with $|\cK| \ge 2$ jointly deviates.
This assumption is essential for LOOM's peer-based verification. If agents can coordinate, the reference set $\cO_{t,-a}$ for a deviating agent $a$ may contain data systematically biased by colluding agents, thereby weakening or preventing detection. This limitation is inherent to peer-based verification; see \cref{asec:faq} (Question~2) and \cref{asec:assumptions} for a discussion of additional conditions, such as requiring a minimum fraction of truthful agents, that might restore guarantees under coordinated deviations; establishing this is left for future work.

\para{LOOM-compatible contextual bandit algorithm.}
LOOM-compatibility is a property of the \emph{estimator conditional on the observation set it receives}; it does not require concentration to hold under arbitrary corruption of that set. This distinction is important. When analyzing a deviation by agent $a$, the remaining agents are truthful by hypothesis. Hence, $\cO_{t,-a}$ contains only truthful observations, and part (ii) applies directly: $f_{t,-a}$ concentrates without any assumption on the deviating agent's reports. This is precisely the setting used in \cref{thm:optgtmarmsstayactive} and explains why LOOM evaluates agent $a$ against a statistic that agent $a$ cannot influence.
In contrast, part (i) is applied to $\cO_t$, which includes agent $a$'s reports. We therefore do not assume that the concentration guarantee in part (i) continues to hold under over-reporting. \cref{lem:biasedEst} quantifies the resulting effect directly: over-reporting introduces the bias
$
\big(\lambda I + \sum_s x_{s,a_s}x_{s,a_s}^\top\big)^{-1}
\big(\sum_s x_{s,a_s}\eta_s\big)
$
into the ridge estimator, so $\theta_\star$ may lie outside the resulting confidence ellipsoid. For multiple deviators, the additional condition required to recover part (i) is isolated in \cref{assu:NEandregret} and invoked only in \cref{thm:regretAllNE} and \cref{thm:regretNonAllNE}.

\subsection{Discussion about Assumption \ref{assu:NEandregret}}
\label{asec:assumptions}
Two clarifications about the status of this assumption. First, \cref{assu:NEandregret} is a condition on the outcome of the interaction rather than on the problem instance: $f_t$ is built from manipulated data, and under persistent misreporting the bias of \cref{lem:biasedEst} need not vanish while the confidence width does. It holds with high probability under truthful reporting, a sufficient primitive condition under over-reporting is given in \cref{prop:assuSufficient} below, and it fails in Example~2 of \cref{asec:loom}. \cref{thm:regretAllNE} and \cref{thm:regretNonAllNE} are therefore conditional results. Second, we state the remaining scope restriction directly: when several agents deviate simultaneously, $\cO_{t,-a}$ still contains the other deviators' manipulated data even though they are not colluding, so the leave-one-out estimator is no longer clean. This is Case~3 (given below), and it is exactly the regime in which \cref{assu:NEandregret} is doing work.

We believe that addressing strategic behavior in a contextual bandit setting without relying on monetary incentives is a challenging and underexplored problem. There is limited prior work in the literature on this topic, despite its many practical applications, e.g., online platforms where sellers may attempt to manipulate the contextual information of their products to gain an advantage. Our key contribution is the development of an approximately incentive-compatible mechanism inspired by the VCG mechanism.  
However, when an agent overreports arm features, all estimators used by {LOOM} become biased due to the over-reported arm feature vectors as inputs. 
It happens because the misreported features distort the overall feature distribution, creating mismatches between features and their corresponding rewards, which in turn induces bias in the estimators.

In contextual bandits, estimators (such as those for the reward function parameters) rely on the assumption that observed features and rewards are generated according to an honest, stationary process. When agents systematically over-report by deliberately inflating feature values to increase their selection probability, the samples collected by the algorithm are corrupted: the feature vectors in the data do not match the ground truth.
Our next result demonstrates that over-reported arm features increase the bias in estimators used by LOOM that take them as input. 
\begin{lem}[{Bias due to over-reporting}]
    \label{lem:biasedEst}
    Using over-reported arm features increases the bias in all estimators used by LOOM that take these arm features as input.
\end{lem}
\begin{proof}
    Without loss of generality, consider a linear reward function with an unknown true parameter vector $\theta_\star \in \mathbb{R}^d$, where $d \ge 1$ is the dimension of the context-arm feature vector. At each round $t$, the true feature vector associated with agent $a$ is denoted by $x_{t,a}^\star$. However, the agent may strategically misreport their features as $x_{t,a}$ in an attempt to appear more favorable, i.e., to give the impression of a higher expected reward, satisfying $\theta_\star^\top x_{t,a} \ge \theta_\star^\top x_{t,a}^\star$.
    Let $\eta_s = \theta_\star^\top x_{s,a_s} - \theta_\star^\top x_{s,a_s}^\star \ge 0$ denote the difference in reward between the misreported features and the true features, where $a_s$ is the agent selected in round $s$. Note that $\eta_s = 0$ when the agent reports truthfully. The ridge regression estimator based on the observed data $\mathcal{H}_t = \{(x_{s,a_s}, y_s)\}_{s < t}$ is given by
    $$
        \hat\theta_t = \left( \lambda I + \sum_{s<t} x_{s,a_s} x_{s,a_s}^\top \right)^{-1} \left( \sum_{s<t} x_{s,a_s} y_s \right),
    $$
    where $\lambda > 0$ is the regularization parameter that ensures the matrix $\lambda I + \sum_{s<t} x_{s,a_s} x_{s,a_s}^\top$ is invertible.

    Since $y_s = \theta_\star^\top x_{s,a_s}^\star + \epsilon_s$, which depends only on the true context-arm feature vector, we can rewrite it as $y_s = \theta_\star^\top x_{s,a_s} - \eta_s + \epsilon_s.$ Substituting this into the ridge regression estimator, we obtain:
    $$
        \hat\theta_t = \left( \lambda I + \sum_{s < t} x_{s,a_s} x_{s,a_s}^\top \right)^{-1} \left( \sum_{s < t} x_{s,a_s} \left( \theta_\star^\top x_{s,a_s} - \eta_s + \epsilon_s \right) \right).
    $$

    Let $y_s' = \theta_\star^\top x_{s,a_s} + \epsilon_s$ denote the noisy reward for the (possibly misreported) context-arm feature vector $x_{s,a_s}$, generated by the true reward function. Using this notation, we can rewrite the estimator as:
    $$
        \hat\theta_t = \left( \lambda I + \sum_{s < t} x_{s,a_s} x_{s,a_s}^\top \right)^{-1} \left( \sum_{s < t} x_{s,a_s} (y_s' - \eta_s) \right).
    $$
    
    Expanding the expression, we obtain:
    $$
        \hat\theta_t = \left( \lambda I + \sum_{s < t} x_{s,a_s} x_{s,a_s}^\top \right)^{-1} \left( \sum_{s < t} x_{s,a_s} y_s' \right)
        - \left( \lambda I + \sum_{s < t} x_{s,a_s} x_{s,a_s}^\top \right)^{-1} \left( \sum_{s < t} x_{s,a_s} \eta_s \right).
    $$
    
    Let $\tilde\theta_t$ denote the ridge regression estimator computed using the misreported feature vectors $\{x_{s,\cdot}\}_{s < t}$ and the corresponding noisy rewards $\{y_s'\}_{s<t}$, i.e.,
    $$
        \tilde\theta_t = \left( \lambda I + \sum_{s < t} x_{s,a_s} x_{s,a_s}^\top \right)^{-1} \left( \sum_{s < t} x_{s,a_s} y_s' \right).
    $$
    Then, we can express $\hat\theta_t$ as:
    $$
        \hat\theta_t = \tilde\theta_t - \left( \lambda I + \sum_{s < t} x_{s,a_s} x_{s,a_s}^\top \right)^{-1} \sum_{s < t} \left( x_{s,a_s} \eta_s \right),
    $$
    where the term $\left( \lambda I + \sum_{s < t} x_{s,a_s} x_{s,a_s}^\top \right)^{-1} \sum_{s < t} \left( x_{s,a_s} \eta_s \right)$ represents the additional bias introduced due to the misreporting of features.
\end{proof}

\para{Some special cases of Assumption \ref{assu:NEandregret}.}
The bias due to over-reported arm features makes it impossible to derive theoretical guarantees without introducing additional assumptions, such as \cref{assu:NEandregret}. 
To validate its practicality, we consider the following three cases:

\textbf{Case 1.} All agents report truthfully: When reported features are the same as true features, i.e., $x = x^*$ for all $x \in \mathcal{X}$, $\ucb_t(x)$ is an upper bound of $\theta_*^\top x$ with probability at least $1-\delta$ (or with high probability, Lemma~\ref{lem:linConfidenceBound}). As a result, Assumption~\ref{assu:NEandregret} holds (note that Theorem~\ref{thm:regretNE} itself does not require it), and hence the NE and regret bounds of our proposed algorithm, COBRA, are improved by a factor of $\sqrt{N}$ compared to Theorem 5.1 in \cite{buening2024strategic}.

\textbf{Case 2.} One agent can over-report while other agents report truthfully and linear reward function: 
Additional observations from agent $a$ shrink the confidence \emph{width}, since $V_t \succeq V_{t,-a}$ implies $\norm{x}_{V_t^{-1}} \le \norm{x}_{V_{t,-a}^{-1}}$, but this does not imply $\ucb_t(x) \le \ucb_{t,-a}(x)$ in general, because the centers $\hat\theta_t^\top x$ and $\hat\theta_{t,-a}^\top x$ differ; this is why \algo{} enforces this comparison explicitly in Line~\ref{step:check}. If the deviating agent is detected and eliminated by LOOM after finitely many rounds, its total over-reporting gap is bounded and \cref{prop:assuSufficient} applies. If it is never detected, as in Example~2 of \cref{asec:loom}, the bias persists and \cref{assu:NEandregret} can fail. Hence, even with a single deviator, \cref{assu:NEandregret} is not implied by the other conditions and remains an assumption.

\textbf{Case 3.} Multiple agents can over-report: In this case, all estimators used by COBRA become biased, making it impossible to derive theoretical guarantees without additional constraints. Our \cref{thm:regretAllNE} and \cref{thm:regretNonAllNE} hold as long as Assumption~\ref{assu:NEandregret} is satisfied. Notably, we impose no restrictions on how agents report their features, aside from the no collusion assumption, which is a common assumption in VCG-type mechanisms \citep{vickrey1961counterspeculation, clarke1971multipart, groves1973incentives}.
We state the scope restriction of this case explicitly, as it is the key limitation of the leave-one-out construction. Even under no collusion, when multiple agents deviate independently, the reference set $\cO_{t,-a}$ for agent $a$ still contains the manipulated reports of the other deviators. Thus, the leave-one-out estimator $f_{t,-a}$ is no longer based solely on truthful data, and the bias characterized in \cref{lem:biasedEst} can affect it as well. Consequently, detection is guaranteed only for unilateral deviations (Case~2). In Case~3, \cref{assu:NEandregret} serves a different purpose: it ensures a valid optimistic bound $f(\ox_{t,a}) \le \ucb_t(\ox_{t,a})$ despite corrupted observations, which is the property required by the regret analysis. 

\para{A sufficient primitive condition.} The following result gives a condition on the problem primitives under which \cref{assu:NEandregret} holds, at the price of an inflated confidence width.
\begin{prop}
    \label{prop:assuSufficient}
    Consider the linear case, let every agent over-report, and let $\eta_s \doteq \theta_\star^\top(x_{s,a_s} - x^\star_{s,a_s}) \ge 0$. Suppose the cumulative over-reporting gap satisfies $\sum_{s<t}\eta_s \le E$ for all $t \le T$, for a known $E \ge 0$. If $\ucb_t$ uses the inflated radius $\alpha_t + E$ in place of $\alpha_t$, then, with probability at least $1-\delta$, $\theta_\star^\top x \le \ucb_t(x)$ for all $x \in \cX$ and $t \le T$; in particular, \cref{assu:NEandregret} holds.
\end{prop}
\begin{proof}
    By \cref{lem:biasedEst}, $\hat\theta_t = \tilde\theta_t - V_t^{-1}\sum_{s<t} x_{s,a_s}\eta_s$, where $\tilde\theta_t$ is the ridge estimator on $(x_{s,a_s}, \theta_\star^\top x_{s,a_s} + \epsilon_s)_{s<t}$. Since $\tilde\theta_t - \theta_\star = V_t^{-1}(\sum_{s<t} x_{s,a_s}\epsilon_s - \lambda\theta_\star)$ with predictable $x_{s,a_s}$, the self-normalized bound of \citet{NIPS11_abbasi2011improved} gives $|x^\top(\tilde\theta_t - \theta_\star)| \le \alpha_t\norm{x}_{V_t^{-1}}$ for all $x$ and $t$ with probability at least $1-\delta$. Moreover, since $V_t \succeq \lambda I + x_{s,a_s}x_{s,a_s}^\top$, we have $\norm{x_{s,a_s}}_{V_t^{-1}} \le 1$, so
    $|x^\top V_t^{-1}\sum_{s<t}x_{s,a_s}\eta_s| \le \norm{x}_{V_t^{-1}}\sum_{s<t}\eta_s\norm{x_{s,a_s}}_{V_t^{-1}} \le E\norm{x}_{V_t^{-1}}$. Combining, $\theta_\star^\top x \le \hat\theta_t^\top x + (\alpha_t + E)\norm{x}_{V_t^{-1}}$. \qedhere
\end{proof}
\emph{An example where it holds:} if the deviator is eliminated by LOOM after $\tau$ rounds and its per-round over-reporting gap is at most $\eta_{\max}$, then $E \le \tau\eta_{\max}$; the inflation then adds $\tilde{O}(E\sqrt{dT})$ to the regret bound, as in corruption-robust linear bandits. \emph{An example where it fails:} in Example~2 of \cref{asec:loom}, $\eta_s = 3/4$ in every round, so $\sum_{s<t}\eta_s = 3(t-1)/4$ is unbounded, and indeed $\ucb_t(1) \to 1/4 < 1 = f(x_{t,2})$ for the uninflated width.

Relaxing this assumption remains an interesting direction for future work.

\para{Comparison from existing literature.} \cite{buening2024strategic} use agent-specific estimators that detect over-reporting in linear contextual bandits. In contrast, our method takes inspiration from the VCG mechanism \citep{vickrey1961counterspeculation, clarke1971multipart, groves1973incentives} and uses the observations associated with other agents to identify the over-reporting of an agent. This key difference leads us to use $\lcb_{t,-a}(x_{t, a})$ (pessimistic reward estimate using observations of all agents except agent $a$) while \cite{buening2024strategic} use $\lcb_{t,a}(x_{t, a})$ (pessimistic reward estimate only using observation associated with agent $a$) for detection.
A misreporting agent cannot influence the leave-one-out estimator used to test it, and the approach extends beyond linear rewards; the cost is that detection relies on peer coverage, which is made explicit by $C$. Empirically, \algo{} outperforms OptGTM even under truthful reporting (\cref{fig:truthful}).

Theorem 5.2 of \cite{buening2024strategic} holds under their Assumption 2 (holds only for linear reward functions), which has the following consequence:
$
    f(x_{t,a^*}^*) \le \ucb_{t, a_t^*}(x_{t,a^*})\le \ucb_{t,a_t}(x_{t, a_t})
$ 
(see proof of Lemma E.5 on Page 27 in \cite{buening2024strategic}). 
In contrast, our assumptions imply the following: 
$
    f(x_{t,a^*}^*) \le f(x_{t,a^*}) \le \ucb_t(x_{t,a_t^*})\le \ucb_t(x_{t,a_t})\le \ucb_{t,-a_t}(x_{t,a_t})
$
which gives:
\eq{
    \label{eqn:assumption}
    f(x_{t,a^*}^*) \le \ucb_t(x_{t,a_t^*})\le \ucb_{t,-a_t}(x_{t, a_t}) \implies f(x_{t,a^*}^*) \le \ucb_{t,-a_t}(x_{t, a_t}).
}

Assumption 2 of \citep{buening2024strategic} and our Step~\ref{step:check} of \algo{} share a key similarity: they define the conditions under which some theoretical results hold (their Theorem 5.2 and ours Theorem~\ref{thm:regretAllNE}). These assumptions also lead to similar consequences, i.e., the maximum expected reward in any round is upper-bounded by the optimistic reward estimate of the selected arm (or agent) computed using the same agent(s) (i.e., $\ucb_{t, a_t}(x_{t,a})$ and $\ucb_{t,-a_t}(x_{t,a})$) as used in the mechanism for identifying over-reporting agents. We emphasize that our Step~\ref{step:check} of \algo{} is not directly comparable to that of \cite{buening2024strategic}, as they provide conditions for the theoretical guarantees of algorithms based on different underlying mechanisms.

We would like to highlight that detecting over-reporting using only an agent's own observations may be ineffective in practice, particularly when the true parameter $\theta_{*}$ is unknown due to the absence of any external baseline for comparison. In contrast, our VCG-inspired approach leverages observations from other agents to identify over-reporting, making it more practical, as the targeted agent cannot directly influence the detection mechanism.

Furthermore, \textit{we extend our analysis to a class of non-linear contextual bandit algorithms, where the confidence widths satisfy the LOOM-compatibility conditions of \cref{defi:loomCompatible}, and LOOM can be used as a subroutine in any such LOOM-compatible (including non-linear) contextual bandit algorithm to identify strategic agents}. This constitutes a new contribution within this setting.

\para{Equivalence between  Step~\ref{step:check} of \algo{} and Leave-one-out arm selection strategy.}
Our Step~\ref{step:check} of \algo{} is equivalent to replacing the arm selection strategy in Step \ref{step:max_selection} of our proposed algorithm, \algo{}, with a leave-one-out (LOO) arm selection strategy defined as:
$
    a_t = \arg\max_{a \in \mathcal{A}} \ucb_{t,-a} (x_{t,a}),
$
where $\ucb_{t,-a}(x_{t,a})$ denotes the upper confidence bound computed using a LOO estimator that excludes the historical data from agent $a$, i.e., while calculating the \ucb value for agent $a$, we use only the data from other agents. However, the confidence bounds derived from these LOO estimators are generally looser than those obtained using the data from all agents. We next prove the equivalence between \cref{assu:NEandregret} and the use of the LOO arm selection strategy (which is also the fallback rule in Line~\ref{step:fallback} of \algo{}).
\begin{lem}
    \label{lem:eqvAssuLOO}
    When one of the agents overreports and this agent is the optimal agent $a^\star_t$, having Assumption \ref{assu:NEandregret} together with Step~\ref{step:check} of \algo{} results in the same regret as using the LOO arm selection strategy.
\end{lem}
This equivalence is stated for a \emph{single} over-reporting agent, and it does not make \cref{assu:NEandregret} unnecessary: the leave-one-out rule does not satisfy \cref{eqn:looSelect} without further assumptions. The step $f(x_{t,a^\star_t}) \le \ucb_{t,-a^\star_t}(x_{t,a^\star_t})$ requires $\cO_{t,-a^\star_t}$ to contain only truthful data, which holds only when the optimal agent is itself the deviator. Otherwise, $\cO_{t,-a^\star_t}$ contains the deviator's corrupted reports; e.g., in Example~2 of \cref{asec:loom}, $\ucb_{t,-1}(1/2) \to 1/8 < 1/2 = f(x^\star_{t,1})$. We retain the standard selection rule as the default because it offers better performance under truthful reporting.

\begin{proof}
    For completeness, we first derive the consequences of over-reporting under Assumption \ref{assu:NEandregret} (see also \cref{eqn:assumption}) as follows. When an agent over-reports, the following inequality holds for the optimal arm $x_{t,a_t^\star}$ in round $t$:
    $
        f(x_{t,a_t^\star}^\star) \le f(x_{t,a_t^\star}).
    $
    Using Assumption \ref{assu:NEandregret}, i.e., $f(x_{t,a_t^\star}) \le \ucb_t(x_{t,a_t^\star})$, which states that the upper confidence bounds remain valid even when agents over-report, we obtain:
    $
        f(x_{t,a_t^\star}^\star) \le f(x_{t,a_t^\star}) \le \ucb_t(x_{t,a_t^\star}).
    $
    Since COBRA selects arm $a_t$ over $a_t^\star$, i.e., $\ucb_t(x_{t,a_t^\star}) \le \ucb_t(x_{t,a_t})$, we obtain:
    $
        f(x_{t,a_t^\star}^\star) \le \ucb_t(x_{t,a_t}).
    $
    Using Step~\ref{step:check} of \algo{}, i.e., $\ucb_t(x_{t,a_t}) \le \ucb_{t,-a_t}(x_{t,a_t})$, we get:
    \eq{
        \label{eqn:assumption1}
        f(x_{t,a_t^\star}^\star) \le \ucb_{t,-a_t}(x_{t,a_t}). 
    }
    When we use the LOO arm selection strategy $a_t = \arg\max_{a \in \mathcal{A}} \ucb_{t,-a}(x_{t,a})$, the following chain of inequalities holds under over-reporting:
    $
        f(x_{t,a_t^\star}^\star) \le f(x_{t,a_t^\star}) \le \ucb_{t,-a_t^\star}(x_{t,a_t^\star}) \quad (\text{valid since } \cO_{t,-a_t^\star} \text{ is truthful when } a_t^\star \text{ is the deviator}).
    $
    Since arm $a_t$ is selected, i.e., $\ucb_{t,-a_t^\star}(x_{t,a_t^\star}) \le \ucb_{t,-a_t}(x_{t,a_t})$, we conclude:
    \eq{
        \label{eqn:looSelect}
        f(x_{t,a_t^\star}^\star) \le \ucb_{t,-a_t}(x_{t,a_t}).
    }
    Since \cref{eqn:assumption1} and \cref{eqn:looSelect} are identical, \cref{assu:NEandregret} and the leave-one-out (LOO) arm selection strategy lead to the same consequences, hence having the same NE and regret guarantees.
\end{proof}

To obtain an upper bound on the regret, we derive a key result that bounds the difference $\ucb_{t,-a_t}(x_{t,a_t}) - f(x_{t,a_t^\star})$. These results are formalized in \cref{lem:regretSA} for the linear reward function and in \cref{lem:regretNonSA} for the non-linear case. Under truthful reporting, LOOM eliminates no agent with high probability (\cref{thm:optgtmarmsstayactive}), and truthfulness is an approximate Nash equilibrium against over-reporting deviations under \cref{assu:NEandregret} (\cref{thm:regretAllNE}); it is not a dominant strategy. Therefore, we do not adopt the LOO arm selection strategy as the default in COBRA, as the LOO strategy underperforms the standard arm selection strategy currently used in COBRA when all agents report truthfully.
LOOM does not work in cases involving complex strategic behavior by agents, such as when multiple agents collude or misreport together. Addressing these challenges by designing new mechanisms will be a promising area of research at the intersection of mechanism design and contextual bandits. We will leave studying these complex settings to future work, for which our work can serve as a foundation.

    \section{Leftover Proofs}
    \label{asec:proofs}

\para{Scope of the guarantees.}
We first state the scope of our results: (i) LOOM is designed to detect \emph{over-reporting} and does not detect under-reporting; (ii) our incentive-compatibility guarantees consider only \emph{unilateral} deviations, and thus do not allow collusion between agents; (iii) we consider a \emph{homogeneous} setting in which all agents share a common reward function $f$; (iv) because LOOM is a peer-based test, its power depends on how well the remaining agents \emph{cover} the deviating agent's arm-feature directions. This coverage requirement is the binding condition in our guarantees and is captured by the constant $C$. When coverage is poor, $C$ can grow as large as $\Theta(\sqrt{T})$, making the bound in Theorems \ref{thm:regretAllNE} and \ref{thm:regretNonAllNE} vacuous (\cref{asec:loom}). We discuss this residual gap, including the substantial overlap between regions where misreporting is beneficial and regions where LOOM works, in \cref{sec:loom} and \cref{asec:loom}; we discuss possible extensions to heterogeneous agents in \cref{asec:scope} and \cref{asec:faq}; and (v) \cref{thm:regretAllNE} and \cref{thm:regretNonAllNE} are conditional on \cref{assu:NEandregret}, which holds with high probability under truthful reporting but can fail under misreporting (\cref{asec:assumptions}).

\subsection{Confidence schedule}
\label{asec:schedule}
Since elimination is permanent, the per-round guarantee of \cref{thm:optgtmarmsstayactive} must hold uniformly over all rounds and agents. Fix $\delta \in (0,1)$ and let $\delta_0 \doteq \delta/(2(N+1))$. We use confidence parameter $\delta_0$ for the bound on $f_t$ and for the bound on each $f_{t,-a}$ in \cref{defi:loomCompatible} (so $\delta^x_{t,a} = \delta_0$, valid simultaneously for all $t$), and $\delta' = \delta/(2N)$ in \cref{lem:rewOptEst}, i.e., $\delta^y_{t,a} = \delta/(8N\max\{1,S_t(a)\}^2)$. By a union bound over the $N+1$ estimators and the $N$ agents, the event $\cE$ on which all these bounds hold simultaneously for all $t \ge 1$ and $a \in \cN$ satisfies $\Prob{\cE} \ge 1-\delta$. Moreover, $\log(1/\delta^y_{t,a}) \le \log(8NT^2/\delta)$ for all $t \le T$, so the schedule affects only logarithmic factors.

\subsection{Leftover proofs from \texorpdfstring{\cref{sec:loom}}{Section 3}}

\rewOptEst*
\begin{proof}
            Recall, the observed reward in round $t$ is $y_t = f(x_{t,a_t}^\star) + \epsilon_t$, where $\epsilon_t$ is $R$-sub-Gaussian conditionally on the history and on $a_t$ (\cref{assu:homogeneous}). Since $S_t(a)$ is random and adaptively determined, we cannot apply a fixed-sample concentration bound with $S_t(a)$ samples directly. Fix agent $a$. For $k \ge 1$, let $\tau_k$ be the round in which agent $a$ is selected for the $k$-th time, and let $\xi_k \doteq \epsilon_{\tau_k}$ if $\tau_k < \infty$, and $\xi_k$ be an independent $\bN(0,R^2)$ (Gaussian) variable otherwise. Let $\cG_{k-1}$ denote the information available just before $\xi_k$ is revealed. Since the selection $a_s$ is made before $\epsilon_s$ is revealed, $\EE{\exp(\gamma \xi_k) \mid \cG_{k-1}} \le \exp(\gamma^2R^2/2)$ for all $\gamma \in \R$ (optional skipping). By the tower rule, for every fixed $n \ge 1$, $M_n \doteq \sum_{k=1}^n \xi_k$ satisfies $\EE{\exp(\gamma M_n)} \le \exp(n\gamma^2R^2/2)$, and the sub-Gaussian (Hoeffding-type) concentration bound gives
    \eqs{
        \Prob{|M_n| \ge \sqrt{2R^2 n \log(4n^2/\delta')}} \le 2 \cdot \frac{\delta'}{4n^2}.
    }
    A union bound over $n \ge 1$ shows that, with probability at least $1 - \sum_{n\ge1}\delta'/(2n^2) = 1 - \delta'\pi^2/12 \ge 1-\delta'$, we have $|M_n| < \sqrt{2R^2 n\log(4n^2/\delta')}$ for all $n \ge 1$ simultaneously. On this event, for every $t$ with $S_t(a) = n \ge 1$, we have $\sum_{s\le t, a_s = a} \epsilon_s = M_n$ and $\delta^y_{t,a} = \delta'/(4n^2)$; for $S_t(a) = 0$ both sums below are empty. Hence, with probability at least $1-\delta'$, simultaneously for all $t \ge 1$,
    \eq{
        \label{eqn:rewardLCB}
        \sum_{s\le t, a_s = a} y_s \ge \sum_{s\le t, a_s = a}  f(x_{s,a_s}^\star) - \sqrt{2R^2S_t(a)\log(1/\delta_{t,a}^y)}
    }
    and
    \eq{
        \label{eqn:rewardUCB}
        \sum_{s\le t, a_s = a} y_s \le \sum_{s\le t, a_s = a} f(x_{s,a_s}^\star) + \sqrt{2R^2S_t(a)\log(1/\delta_{t,a}^y)}.
    }
    Rearranging \cref{eqn:rewardLCB} gives the claim. \qedhere
    \end{proof}

\optgtmarmsstayactive*
\begin{proof}
    Since all agents report truthfully, for all $t \ge 1, a \in \cA_t:\tx = \ox$. 
    Note that we are estimating the reward function $f$ using available observations, context-arm features, and rewards. 
    Recall that we use $\cO_{t,-a}$ to denote the observations from all agents except agent $a$ and
    $f_{t,-a}$ represents the estimate of reward function $f$ using $\cO_{t,-a}$ at the end of round $t$.
    Even if other agents report truthfully, noisy reward feedback may lead to an inaccurate estimator.
    Let the confidence bound $|f_{t,-a}(x) - f(x)| \le h(x, \cO_{t,-a})$ hold with probability $1-\delta_{t,a}$.
    Then, for any $x \in \cX$, $\lcb_{t,-a}(x)= f_{t,-a}(x) - h(x, \cO_{t,-a})$ is a pessimistic estimate of the expected reward for $x$ that also holds with probability $1-\delta_{t,a}$.
    Furthermore, $f(x) \ge \lcb_{t,-a}(x)$ (see \cref{lem:nonLinLooConfidenceBound} in \cref{asec:non_linear_analysis} for more details).
    Using this, for any $x_{t,a_t} \in \cX$, we have
    $
        f(x_{t,a_t}^\star) = f(x_{t,a_t}) \ge \lcb_{t,-a}(\ox_{t,a_t}) 
        \implies f(x_{t,a_t}^\star) \ge \lcb_{t,-a}(\ox_{t,a_t}). 
    $
    Next, we can lower bound the sum of true rewards in terms of the lower confidence bound on estimated rewards using observed context-arm feature vectors as follows:
    \al{
        &\sum_{s\le t, a_s = a}f(x_{s,a_s}^\star) \ge  \sum_{s\le t, a_s = a} \lcb_{t,-a}(\ox_{s, a_s}) \nonumber \\
        \implies &\sum_{s\le t, a_s = a} \lcb_{t,-a}(\ox_{s, a_s}) \le \sum_{s\le t, a_s = a}f(x_{s,a_s}^\star). \label{eqn:testLoomTrueRew}
    }

    For brevity, we assume the above bound holds with probability at least $1-\delta_{t,a}^x$ in the round $t$. Note that $\delta_{t,a}^x$ can be computed exactly when applying the union bound.
    Since the true reward is unknown, we instead first use the upper bound provided in \cref{lem:rewOptEst}, which holds with probability at least $1-\delta_{t,a}^y$, to modify \cref{eqn:testLoomTrueRew}. We use the definitions of $\lcb_{t,a}^{(x)}$ and $\ucb_{t,a}^{(y)}$ to get:
    \al{
        \label{eqn:testSA}
        &\sum_{s\le t, a_s = a} \lcb_{t,-a}(\ox_{s, a_s}) \le \sum_{s\le t,a_s = a} y_s + \sqrt{2R^2S_t(a)\log(1/\delta_{t,a}^y)} \nonumber \\
        \implies & \lcb_{t,a}^{(x)} \le \ucb_{t,a}^{(y)}.
    }
    If the sum of the lower bound of estimated rewards is at most the upper bound of observed rewards for an agent, then the LOOM condition does not hold and the agent is not eliminated. 
    However, if any agent violates \cref{eqn:testSA}, i.e., $\lcb_{t,a}^{(x)} > \ucb_{t,a}^{(y)}$, then, with probability at least $1-\delta^x_{t,a}-\delta^y_{t,a}$, that agent is not truthful.  
    The probability that the LOOM condition holds for a truthful agent is upper bounded by $\delta_{t,a}^x + \delta_{t,a}^y$.
    Since this condition is used as a criterion in \algo{} to identify the strategic agent, \algo{} does not eliminate a truthful agent with probability at least $1-\delta_{t,a}^x - \delta_{t,a}^y$.

    Under the confidence schedule of \cref{asec:schedule}, on the event $\cE$ the leave-one-out lower bounds (with $\delta^x_{t,a} = \delta_0$) and \cref{lem:rewOptEst} (with $\delta' = \delta/(2N)$) hold simultaneously for all $t \ge 1$ and all $a \in \cN$. Under truthful reporting, the argument above therefore applies to every agent in every round, so LOOM never eliminates any agent and $\cN_t = \cN_1$ for all $t \ge 1$. Since $\Prob{\cE} \ge 1-\delta$, this proves the second claim. In particular, the event that $\cN_{t+1} = \emptyset$ for some $t$ has probability at most $\delta$.
\end{proof}

We want to emphasize that the quantifier $\bm{\sigma}_{-a} = \bm{\sigma}^\star_{-a}$ is exactly the one required by Nash incentive compatibility, under which $\cO_{t,-a}$ contains only truthful observations. Accordingly, the proof above relies only on $\lcb_{t,-a}$, \cref{lem:rewOptEst}, and \cref{lem:nonLinLooConfidenceBound}, all evaluated at the truthful profile, and makes no assumption about corrupted observations. Moreover, the proof never compares $h(x,\cO_t)$ and $h(x,\cO_{t,-a})$, and therefore does not depend on the advantage constant $C$ in \cref{def:advantage}.

\subsection{Leftover proofs from \texorpdfstring{\cref{sec:cobra}}{Section 4}}

The following lemmas are fundamental to the proof of our theoretical results. We use the following notation throughout the proof: the arm is represented by $a$, and $-a$ represents other than arm $a$'s estimate.
We use $\norm{x}_A$ to denote the weighted $l_2$-norm of vector $x$ with respect to matrix $A$.
We next state the result that gives the confidence ellipsoid with center at $\hat\theta_t$ or confidence set for the case when the reward function is linear. We will use this result to prove our bounds in \cref{sec:cobra}.

\begin{lem}
    \label{lem:linConfidenceBound}
    Let $\delta \in (0,1)$, $\lambda>0$, $R > 0$, $\hat\theta_t = V_t^{-1} \sum_{s=1}^{t-1} x_{s,a_s}y_s$, $V_t = \lambda I + \sum_{s=1}^{t-1} x_{s,a_s} x_{s,a_s}^\top$. 
    Then, with probability at least $1-\delta$, for all $t\ge 1$, $\theta_\star$ lies in the following confidence set:
    \eqs{
        C_t = \left\{ \theta \in \R^d \colon \norm{\hat\theta_t - \theta}_{V_t} \leq \alpha_t \right\}, \text{ where }
        \alpha_t = \Lp R\sqrt{d\log\left( \frac{1+ \Lp{tL^2}/{\lambda}\Rp}{\delta}\right)} + \lambda^{\frac{1}{2}}S\Rp.
    }
    Furthermore, with probability at least $1-\delta$,
    \eqs{   
        \forall x \in \cX: \slin{x} \le \ucb_t(x)= \tlin{x} + \alpha_t \norm{x}_{V_t^{-1}}.
    }
    Similarly, with probability at least $1-\delta$,
    \eqs{   
        \forall x \in \cX: \slin{x} \ge \lcb_t(x)= \tlin{x} - \alpha_t \norm{x}_{V_t^{-1}}.
    }
\end{lem}

\begin{proof}
    The proof of the first part of the results directly follows from Theorem 2 of \cite{NIPS11_abbasi2011improved}. The proof of the second part follows from the first part with some simple algebraic manipulations as follows:
    \als{
        &\slin{x} - \tlin{x} \le |\tlin{x} - \slin{x}| \\
        \implies &\slin{x} - \tlin{x} \le \norm{\hat\theta_t - \theta_\star}_{V_t}\norm{x}_{V_t^{-1}} \\
        \implies &\slin{x} \le \tlin{x} + \alpha_t\norm{x}_{V_t^{-1}} \\
        \implies &\slin{x} \le \ucb_t(x).
    }

    Similarly, the last part also follows from the first part with some simple algebraic manipulations as follows:
    \eqs{
        |\tlin{x} - \slin{x}| \le \norm{x}_{V_t^{-1}}\norm{\hat\theta_t - \theta_\star}_{V_t}.
    }

    After reversing the above inequality, we have
    \als{
        &\norm{x}_{V_t^{-1}}\norm{\hat\theta_t - \theta_\star}_{V_t} \ge |\tlin{x} - \slin{x}| \ge \tlin{x} - \slin{x} \\
        \implies &\norm{x}_{V_t^{-1}}\norm{\hat\theta_t - \theta_\star}_{V_t} \ge \tlin{x} - \slin{x} \\
        \implies &\slin{x} \ge \tlin{x} - \norm{\hat\theta_t - \theta_\star}_{V_t}\norm{x}_{V_t^{-1}} \\
        \implies &\slin{x} \ge \tlin{x} - \alpha_t\norm{x}_{V_t^{-1}} \\
        \implies &\slin{x} \ge \lcb_t(x).  \qedhere
    }
\end{proof}

Note that $\theta_\star$ may lie outside the confidence ellipsoid $C_t$ when agents misreport. However, when all agents are truthful, i.e., $\ox = \tx$, \cref{lem:linConfidenceBound} applies directly.  
We make the manipulation setting explicit because it is the regime for which LOOM is designed. \cref{lem:linConfidenceBound} is derived for observations generated by the reported features, and we do not assume that its confidence guarantee remains valid when those features are corrupted. In particular, when an agent over-reports, \cref{lem:biasedEst} in \cref{asec:assumptions} shows that the ridge estimator acquires the bias
$
    \big(\lambda I + \sum_{s<t} x_{s,a_s}x_{s,a_s}^\top\big)^{-1}\big(\sum_{s<t} x_{s,a_s}\eta_s\big),
$
with $\eta_s \ge 0$, so $\theta_\star$ may lie outside the resulting confidence ellipsoid. Thus, the confidence guarantee for the full observation set relies on truthful reporting. In contrast, \cref{lem:linLooConfBound} is stated for $\cO_{t,-a}$ and is the key ingredient in the detection guarantee: under a unilateral deviation, the remaining agents are truthful, so its assumptions hold. For multiple deviators, the additional condition required for \cref{lem:linConfidenceBound} is isolated explicitly in \cref{assu:NEandregret}.

Recall the following definitions from the main paper (note that we use the ridge regression closed-form solution excluding the information of agent $a$):
\eqs{
    \hat\theta_{t,-a} = V_{t,-a}^{-1} \sum_{s=1,a_s \ne a}^{t-1} x_{s,a_s} y_s, \text{ with } V_{t,-a} = \lambda I + \sum_{s=1,a_s\ne a}^{t-1} x_{s,a_s}x_{s,a_s}^\top.
}

\begin{lem}
    \label{lem:linLooConfBound}
    Let $\delta \in (0,1)$, $\lambda>0$, and $R > 0$. Then, with probability at least $1-\delta$, for all $t \ge 1$,
    \eqs{
        \norm{\hat\theta_{t,-a} - \theta_\star}_{V_{t,-a}} \le \Lp R\sqrt{ d\log \Lp \frac{1+ (t-S_t(a))L^2/\lambda}{\delta}\Rp} + \lambda^{\frac{1}{2}}S \Rp = \alpha_{t,-a}.
    }
    Furthermore, with probability at least $1-\delta$, the upper bound of $\slin{x}$ is given by
    \eqs{   
        \forall x \in \cX: \slin{x} \le \ucb_{t,-a}(x)=  \tlina{x} + \alpha_{t,-a} \norm{x}_{V_{t,-a}^{-1}}.
    }
    Similarly, with probability at least $1-\delta$, the lower bound of $\slin{x}$ is given by
    \eqs{   
        \forall x \in \cX:\ \slin{x} \ge \lcb_{t,-a}(x)= \tlina{x} - \alpha_{t,-a} \norm{x}_{V_{t,-a}^{-1}}.
    }
\end{lem}

\begin{proof}
    The first part of the proof follows from \cref{lem:linConfidenceBound} as we are not using observations associated with agent $a$, reducing to the standard confidence bound restricted to observations of all agents except $a$. 
    The proof of the second part follows from the first part with some simple algebraic manipulations as follows:
    \als{
        &\slin{x} - \tlina{x} \le |\tlina{x} - \slin{x}| \\
        \implies &\slin{x} - \tlina{x} \le \norm{\hat\theta_{t,-a} - \theta_\star}_{V_{t,-a}}\norm{x}_{V_{t,-a}^{-1}} \\
        \implies &\slin{x} \le \tlina{x} + \alpha_{t,-a}\norm{x}_{V_{t,-a}^{-1}} \\
        \implies &\slin{x} \le \ucb_{t,-a}(x).
    }

    Similarly, the last part follows from the first part with some algebraic manipulations as follows:
    \eqs{
        |\tlina{x} - \slin{x}| \le \norm{x}_{V_{t,-a}^{-1}}\norm{\hat\theta_{t,-a} - \theta_\star}_{V_{t,-a}}.
    }

    After reversing the above inequality, we have
    \als{
        &\norm{x}_{V_{t,-a}^{-1}}\norm{\hat\theta_{t,-a} - \theta_\star}_{V_{t,-a}} \ge |\tlina{x} - \slin{x}| \ge \tlina{x} - \slin{x} \\
        \implies &\norm{x}_{V_{t,-a}^{-1}}\norm{\hat\theta_{t,-a} - \theta_\star}_{V_{t,-a}} \ge \tlina{x} - \slin{x} \\
        \implies &\slin{x} \ge \tlina{x} - \norm{\hat\theta_{t,-a} - \theta_\star}_{V_{t,-a}}\norm{x}_{V_{t,-a}^{-1}} \\
        \implies &\slin{x} \ge \tlina{x} - \alpha_{t,-a}\norm{x}_{V_{t,-a}^{-1}} \\
        \implies &\slin{x} \ge \lcb_{t,-a}(x).  \qedhere
    }

\end{proof}

\subsubsection{Proof of \texorpdfstring{\cref{thm:regretNE}}{Theorem 2}}

\regretNE*
\begin{proof}
    When all agents report truthfully, our algorithm is the same as Lin-UCB \citep{AISTATS11_chu2011contextual} with a mechanism for identifying strategic agents that eliminates no agent on the event $\cE$, which holds with probability at least $1-\delta$ (\cref{thm:optgtmarmsstayactive}). For completeness, we first prove the regret upper bound of \algo{} as follows:
    \al{
        \label{eqn:regretUBTrue}
        \Regret_T\Lp\algo, \bm{\sigma^\star}\Rp &= \sum_{t=1}^T (\slin{x_{t,a_t^\star}^\star} - \slin{x_{t,a_t}^\star}).
    }

    Since the true feature vector is the same as the reported context-arm feature vector (i.e., $x_{t,a}^\star = x_{t,a}$), we can start with upper bounding the difference $\slin{x_{t,a_t^\star}^\star} - \slin{x_{t,a_t}^\star}$ as follows:
    \al{
        \slin{x_{t,a_t^\star}^\star} - \slin{x_{t,a_t}^\star} 
        &=\slin{x_{t,a_t^\star}} - \slin{x_{t,a_t}} \nonumber\\
        &\le \ucb_t{(x_{t,a_t^\star})} - \slin{x_{t,a_t}} \nonumber\\
        &\le \ucb_t{(x_{t,a_t})} - \slin{x_{t,a_t}} \hspace{5mm} \Lp \text{as $\ucb_t{(x_{t,a_t^\star})} \le \ucb_t{(x_{t,a_t})}$} \Rp \nonumber\\
        &= \tlin{\ox_{t, a_t}} + \alpha_t\norm{x_{t,a_t}}_{V_t^{-1}} - \slin{x_{t,a_t}} \nonumber\\
        &= \tlin{\ox_{t, a_t}} - \slin{x_{t,a_t}} + \alpha_t\norm{x_{t,a_t}}_{V_t^{-1}} \nonumber\\
        &\le \norm{\theta_\star - \hat\theta_{t}}_{V_t} \norm{x_{t,a_t}}_{V_t^{-1}} + \alpha_t\norm{x_{t,a_t}}_{V_t^{-1}} \nonumber\\
        &\le \alpha_t\norm{x_{t,a_t}}_{V_t^{-1}} + \alpha_t\norm{x_{t,a_t}}_{V_t^{-1}} \nonumber\\
        \implies \slin{x_{t,a_t^\star}^\star} - \slin{x_{t,a_t}^\star} &\le 2\alpha_t\norm{x_{t,a_t}}_{V_t^{-1}}.  \label{eqn:gapUB}
    }
    Note that $\hat\theta_t$ is an estimator of $\theta_\star$ as the true feature vector is the same as the reported context-arm feature vector.
    After using the upper bound given in \cref{eqn:gapUB} into \cref{eqn:regretUBTrue}, we get an upper bound on the regret as follows:
    \al{
        \Regret_T\Lp\algo, \bm{\sigma^\star}\Rp &=\sum_{t=1}^T\big(\slin{x_{t,a_t^\star}^\star} - \slin{x_{t,a_t}^\star}\big) \nonumber \\ 
        &\le \sum_{t=1}^T 2\alpha_t\norm{x_{t,a_t}}_{V_t^{-1}} \nonumber \\ 
        &= 2\sum_{t=1}^T \alpha_t\norm{x_{t,a_t}}_{V_t^{-1}} \nonumber\\
        &\le 2\sqrt{T} \sqrt{\sum_{t=1}^T \Lb \alpha_t\norm{x_{t,a_t}}_{V_t^{-1}} \Rb^2 } \nonumber\\
        &\le 2\sqrt{T} \sqrt{\sum_{t=1}^T \Lb \alpha_T \norm{x_{t, a_t}}_{V_t^{-1}} \Rb^2 } \nonumber\\
        &= 2\sqrt{T} \sqrt{ \alpha_T^2 \sum_{t=1}^T \norm{x_{t, a_t}}_{V_t^{-1}}^2 } \nonumber\\
        &= 2\alpha_T\sqrt{T} \sqrt{\sum_{t=1}^T \norm{x_{t, a_t}}_{V_t^{-1}}^2 } \nonumber\\
        &\le  2\alpha_T \sqrt{T} \sqrt{2\log\frac{\textnormal{det}(V_T)}{\textnormal{det} (\lambda I_d)}} \nonumber\\
        \implies \Regret_T\Lp\algo, \bm{\sigma^\star}\Rp &\le 2\alpha_T \sqrt{2dT\log (1 + TL^2/(d\lambda))} = \tilde{O}(d\sqrt{T})  \label{eqn:regretSumUB}.
    }
    The first inequality directly follows from \cref{eqn:gapUB}. The second inequality is due to using the Cauchy-Schwarz inequality, where the third inequality follows from the fact that $\alpha_t$ increases with $t$.
    The last two inequalities follow from Lemma 11 and Lemma 10 of \cite{NIPS11_abbasi2011improved}, respectively, and $\alpha_T = O\big(R\sqrt{d\log(TL^2/(\lambda\delta))} + \sqrt{\lambda}S\big) = \tilde{O}(\sqrt{d})$.  
    
    We now lower bound the number of selections of every agent under the truthful profile; this alone does not establish an approximate Nash equilibrium (see the scope paragraph after this proof). Recall, $S_T(a)$ denotes the number of times agent $a$ is selected by \algo{}, which is given as follows:
    \al{
        S_T(a) &=\sum_{t=1}^T\ind{a_t= a} \nonumber \\
        &=\sum_{t=1}^T\ind{a_t= a, a_t^\star = a} +  \sum_{t=1}^T\ind{a_t= a, a_t^\star \ne a} \nonumber \\
        &\ge \sum_{t=1}^T\ind{a_t^\star = a} - \sum_{t=1}^T\ind{a_t^\star = a, a_t \ne a} \nonumber \\
        &\ge \sum_{t=1}^T\ind{a_t^\star = a} - \sum_{t=1}^T\ind{a_t \ne a_t^\star} \nonumber \\
        \implies S_T(a) &\ge S^\star_T(a) - \sum_{t=1}^T\ind{a_t \ne a_t^\star} \label{eqn:pullsLB}. 
    }

    To get the lower bound $S_T(a)$, we get an upper bound $\sum_{t=1}^T\ind{a_t \ne a_t^\star}$. Let $\Delta_{a_t} = \Lp \slin{x_{t,a_t^\star}^\star} -   \slin{x_{t,a_t}^\star}\Rp > 0$ for $a_t \ne a_t^\star$. 
    We multiply and divide $\ind{a_t \ne a_t^\star}$ by $\Delta_{a_t}$ and then use inequality in \cref{eqn:gapUB}, i.e., $\Delta_{a_t} \le 2\alpha_t\norm{x_{t,a_t}}_{V_t^{-1}}$ as follows:
    \als{
        \sum_{t=1}^T\ind{a_t \ne a_t^\star} &=\sum_{t=1}^T\ind{a_t \ne a_t^\star}\frac{\Delta_{a_t}}{\Delta_{a_t}} \\
        &\le \sum_{t=1}^T\ind{a_t \ne a_t^\star}\frac{2\alpha_t \norm{x_{t,a_t}}_{V_t^{-1}}}{\Delta_{a_t}} \hspace{5mm} (\text{as } x_{t,a}^\star = x_{t,a})\\
        &\le \sum_{t=1}^T\frac{2\alpha_t \norm{x_{t,a_t}}_{V_t^{-1}}}{\Delta_{a_t}} \\ 
        &=\sum_{t=1}^T\frac{1}{\Delta_{a_t}} 2\alpha_t \norm{x_{t,a_t}}_{V_t^{-1}} \\
        &\le  \sqrt{\sum_{t=1}^T\Lp\frac{1}{\Delta_{a_t}}\Rp^2 \sum_{t=1}^T \Lp 2\alpha_t \norm{x_{t,a_t}}_{V_t^{-1}}\Rp^2}\\
        &\le  \sqrt{\sum_{t=1}^T\Lp\frac{1}{\Delta_{a_t}}\Rp^2 \sum_{t=1}^T\Lp 2\alpha_T \norm{x_{t,a_t}}_{V_t^{-1}}\Rp^2} \\
        &= \sqrt{\sum_{t=1}^T\Lp\frac{1}{\Delta_{a_t}}\Rp^2} \sqrt{\sum_{t=1}^T\Lp 2\alpha_T \norm{x_{t,a_t}}_{V_t^{-1}}\Rp^2} \\
        &= \sqrt{\sum_{t=1}^T\Lp\frac{1}{\Delta_{a_t}}\Rp^2} \sqrt{\Lp 2\alpha_T\Rp^2\sum_{t=1}^T \norm{x_{t,a_t}}_{V_t^{-1}}^2}\\
        &= 2\alpha_T \sqrt{\sum_{t=1}^T\Lp\frac{1}{\Delta_{a_t}}\Rp^2} \sqrt{\sum_{t=1}^T \norm{x_{t,a_t}}_{V_t^{-1}}^2} \\
        &\le 2\alpha_T \sqrt{\sum_{t=1}^T\Lp\frac{1}{\Delta_{a_t}}\Rp^2} \sqrt{2\log\frac{\textnormal{det}(V_T)}{\textnormal{det} (\lambda I_d)}} \\
        &\le 2\alpha_T \sqrt{\sum_{t=1}^T\Lp\frac{1}{\Delta_{a_t}}\Rp^2} \sqrt{2d\log (1 + TL^2/(d\lambda))} \\
        &\le 2\alpha_T \sqrt{\sum_{t=1}^T\Lp\frac{1}{\Delta_{\min}}\Rp^2} \sqrt{2d\log (1 + TL^2/(d\lambda))} \\
        &\le \frac{2}{\Delta_{\min}} \alpha_T \sqrt{T} \sqrt{2d\log (1 + TL^2/(d\lambda))} \\
        \implies \sum_{t=1}^T\ind{a_t \ne a_t^\star} &\le \frac{2}{\Delta_{\min}}\Lp R\sqrt{d\log\left( \frac{1+ \Lp{TL^2}/{\lambda}\Rp}{\delta}\right)} + \lambda^{\frac{1}{2}}S\Rp \sqrt{2dT\log (1 + TL^2/(d\lambda))}.
    }

    Note that $\Delta_{\min} \le \Delta_{a_t}$ for every round with $a_t \ne a^\star_t$ by the gap condition in \cref{sec:cobra}. 
    Although using $\Delta_{\min}$ loosens the upper bound, we use this to get dependence on $T$. Let $\tilde{O}$ hide the dependence on logarithmic terms, then we have the following result:
    \eq{
        \label{eqn:subOptPulls}
        \sum_{t=1}^T\ind{a_t \ne a_t^\star} \le \tilde{O}\Lp d\sqrt{T}/\Delta_{\min} \Rp.
    }
    
    Using this upper bound in \cref{eqn:pullsLB}, we get the following bound for any agent $a \in \cN$:
    \eq{
        \label{eqn:neLB}
        S_T(a) \ge S^\star_T(a) - \tilde{O}\Lp d\sqrt{T}/\Delta_{\min}\Rp.
    }
    
    All bounds above hold on the event $\cE$, which has probability at least $1-\delta$, and on which no agent is eliminated (\cref{thm:optgtmarmsstayactive}). This completes the proof.
    \end{proof}

\para{Scope of \cref{thm:regretNE}.} The bound in \cref{eqn:subOptPulls} is derived under $x_{t,a} = x^\star_{t,a}$ for all agents, using the confidence ellipsoid of $\hat\theta_t$ (\cref{lem:linConfidenceBound}). When an agent misreports, this ellipsoid can be invalid (\cref{lem:biasedEst}), so \cref{eqn:subOptPulls} cannot be used to bound the gain of a deviating agent. Example~2 in \cref{asec:loom} shows that such a deviation can gain $\Theta(T)$ selections when peer coverage is absent. Deviations are handled in \cref{thm:regretAllNE}, whose bound depends on $C$. The proof of \cref{thm:regretNE} is entirely in terms of $V_t$ and does not involve $C$.

\subsubsection{Proof of \texorpdfstring{\cref{thm:regretAllNE}}{Theorem 4}}
\label{asec:thm4proof}
To prove \cref{thm:regretAllNE}, we will use the following 
result
that upper bounds the total amount of regret that an agent $a$ can exert before being identified by LOOM. 
\begin{lem}
    \label{lem:regretSA}
    Let $\ucb_{t,-a}(\ox_{s, a})  = \tlina{\ox_{s, a}} + \alpha_{t,-a}\norm{\ox_{s, a}}_{V_{t,-a}^{-1}}$, where $x_{s,a}$ is the arm feature vector associated with agent $a$ in the round $s$. Then, with probability at least $1-\delta_{t,a}^x - \delta_{t,a}^y$,
    \als{
        \sum_{s \le t\colon a_s = a} \Lp \ucb_{t,-a}(\ox_{s,a}) - \slin{x_{s,a}^\star} \Rp \leq \sum_{s\le t, a_s = a} 2\alpha_{t,-a}\norm{x_{s,a}}_{V_{t,-a}^{-1}} + 2\sqrt{2R^2S_t(a)\log(1/\delta_{t,a}^y)}.
    } 
\end{lem}
\begin{proof}
    Using \cref{eqn:testSA} with \cref{lem:linLooConfBound} for linear reward function that holds with probability at least $1-\delta_{t,a}^x$,
    we get:
    \als{
        \sum_{s\le t, a_s = a} \Lp \tlina{\ox_{s, a}} - \alpha_{t,-a}\norm{x_{s,a}}_{V_{t,-a}^{-1}} \Rp &\le \sum_{s\le t,a_s = a} y_s + \sqrt{2R^2S_t(a)\log(1/\delta_{t,a}^y)} \\
        &\le \sum_{s\le t, a_s = a} \slin{x_{s, a}^\star} + \sqrt{2R^2S_t(a)\log(1/\delta_{t,a}^y)} \\
        &\qquad + \sqrt{2R^2S_t(a)\log(1/\delta_{t,a}^y)}\\
        \implies \sum_{s\le t, a_s = a} \Lp \tlina{\ox_{s, a}} - \slin{x_{s, a}^\star}\Rp &\le \sum_{s\le t, a_s = a} \alpha_{t,-a}\norm{x_{s,a}}_{V_{t,-a}^{-1}} + 2\sqrt{2R^2S_t(a)\log(1/\delta_{t,a}^y)}.
    }
    The second inequality follows from \cref{eqn:rewardUCB} by using upper bound (as the reward function is linear) on $\sum_{s\le t,a_s = a} y_s$ that holds with probability $1-\delta_{t,a}^y$.
    Now we prove the second part of the result by replacing $\tlina{\ox_{s, a}}$ by $\ucb_{t,-a}(\ox_{s,a})  - \alpha_{t,-a}\norm{x_{s,a}}_{V_{t,-a}^{-1}}$ and we get
    \als{
        &\sum_{s\le t, a_s = a} \Lp \ucb_{t,-a}(\ox_{s,a})  - \alpha_{t,-a}\norm{x_{s,a}}_{V_{t,-a}^{-1}} - \slin{x_{s, a}^\star}\Rp \\
        &\qquad\le \sum_{s\le t, a_s = a} \alpha_{t,-a}\norm{x_{s,a}}_{V_{t,-a}^{-1}} + 2\sqrt{2R^2S_t(a)\log(1/\delta_{t,a}^y)} \\
        \implies &\sum_{s\le t, a_s = a} \Lp \ucb_{t,-a}(\ox_{s,a}) - \slin{x_{s, a}^\star}\Rp \\
        &\qquad\le \sum_{s\le t, a_s = a} 2\alpha_{t,-a}\norm{x_{s,a}}_{V_{t,-a}^{-1}} + 2\sqrt{2R^2S_t(a)\log(1/\delta_{t,a}^y)}. \qedhere
    }
\end{proof}

We first restate the main assumptions needed to prove \cref{thm:regretAllNE}.
\Assumption*
We want to highlight that the homogeneity (\cref{assu:homogeneous}) requirement is stronger than what LOOM actually needs. LOOM only requires the \emph{directions} played by an over-reporting agent to be sufficiently covered by the reports of the remaining agents. This is a property of the \emph{design}, rather than of agent identity. Accordingly, we formulate detectability in terms of the confidence width $h(x, \cO_{t,-a})$, rather than the identity of the agent, and the same quantity determines the advantage constant $C$ in \cref{def:advantage}. Possible extensions to genuinely heterogeneous agents, including clustered agent types, similarity-based weighting, and shared latent representations, are discussed in \cref{asec:faq} (Question~3).

We now have all the results that will be used to prove \cref{thm:regretAllNE}.
\regretAllNE*
\begin{proof}
    Recall $\cA_t$ denotes the set of arms reported by active agents in the round $t$, whose feature vectors are $\{\ox_{t,a}\}_{a\in\cA_t}$. The regret of \algo{} for any over-reporting profile $\bm{\sigma}$ satisfying \cref{assu:NEandregret} is given as follows:
    \al{
        \label{eqn:regretUB}
        \Regret_T\Lp\algo, \bm{\sigma}\Rp = \sum_{t=1}^T \Lp \slin{x_{t,a_t^\star}^\star} - \slin{x_{t,a_t}^\star}\Rp = \sum_{t=1}^T \Lp \max_{a \in \cA_t}\slin{x_{t,a}^\star} - \slin{x_{t,a_t}^\star}\Rp.
    }
    
    Under \cref{assu:NEandregret}, if \algo{} selects $a_t \in \cA_t \colon a_t \ne a_t^\star$,
    we have $\slin{x_{t,a_t^\star}^\star} \le \ \slin{x_{t,a_t^\star}} \le \ucb_{t}(x_{t,a_t^\star}) \le \ucb_{t}(x_{t,a_t})$.
    Using $\slin{x_{t,a_t^\star}^\star} \le \ucb_{t}(x_{t,a_t})$ inequality with \cref{lem:regretSA}, we have
    \al{
        \Regret_T\Lp\algo, \bm{\sigma}\Rp &= \sum_{t=1}^T \Lp \slin{x_{t,a_t^\star}^\star} - \slin{x_{t,a_t}^\star}\Rp \nonumber \\
        &\le \sum_{t=1}^T \Lp \slin{x_{t,a_t^\star}} - \slin{x_{t,a_t}^\star}\Rp  \hspace{5mm}(\text{agents are over-reporting}) \nonumber \\
        &\le \sum_{t=1}^T \Lp \ucb_t(x_{t,a_t^\star}) - \slin{x_{t,a_t}^\star}\Rp \hspace{5mm} (\text{\cref{assu:NEandregret}}) \nonumber \\
        &\le \sum_{t=1}^T \Lp \ucb_t(x_{t,a_t}) - \slin{x_{t,a_t}^\star} \Rp \hspace{5mm} (\text{as selected arm is $a_t$}) \nonumber \\
        &\le \sum_{t=1}^T \Lp \ucb_{t,-a_t}(x_{t,a_t}) - \slin{x_{t,a_t}^\star} \Rp \hspace{5mm} (\text{ensured by the Step~\ref{step:check} of \algo{}}) \nonumber \\
        &= \sum_{t=1}^T \sum_{a \in \cN} \ind{a_t = a} \Lp \ucb_{t,-a}(x_{t,a}) - \slin{x_{t,a}^\star} \Rp \nonumber \\
        &=  \sum_{a=1}^N \sum_{t\le T, a_t = a} \Lp \ucb_{t,-a}(x_{t,a}) - \slin{x_{t,a}^\star} \Rp \nonumber \\
        &\le  \sum_{a=1}^N \Lp \sum_{t\le T, a_t = a} 2\alpha_{t,-a}\norm{x_{t,a}}_{V_{t,-a}^{-1}} + 2\sqrt{2R^2S_{T}(a)\log(1/\delta_{T,a}^y)} \Rp \hspace{5mm} (\text{\cref{lem:regretSA}}) \nonumber \\
        &= \sum_{a=1}^N\sum_{t\le T, a_t = a} 2\alpha_{t,-a}\norm{x_{t,a}}_{V_{t,-a}^{-1}} + \sum_{a=1}^N 2\sqrt{2R^2S_{T}(a)\log(1/\delta_{T,a}^y)} \nonumber \\
        &= \sum_{t=1}^T 2\alpha_{t,-a_t}\norm{x_{t,a_t}}_{V_{t,-a_t}^{-1}} + \sum_{a=1}^N 2\sqrt{2R^2S_{T}(a)\log(1/\delta_{T,a}^y)}. \label{eqn:regretSumUBSA}
    }

    First, we will upper bound the first part of the above inequality, i.e., $\sum_{t=1}^T\alpha_{t,-a_t}\norm{x_{t,a_t}}_{V_{t,-a_t}^{-1}}$, as follows:
    \al{
        \sum_{t=1}^T 2\alpha_{t,-a_t}\norm{x_{t,a_t}}_{V_{t,-a_t}^{-1}} \nonumber 
        &\le 2\sqrt{T} \sqrt{\sum_{t=1}^T \Lb \alpha_{t,-a_t}\norm{x_{t,a_t}}_{V_{t,-a_t}^{-1}} \Rb^2 } \nonumber\\
        &\le 2\sqrt{T} \sqrt{\sum_{t=1}^T \Lb \alpha_T \norm{x_{t, a_t}}_{V_{t,-a_t}^{-1}} \Rb^2 } \nonumber \\
        &= 2\sqrt{T} \sqrt{ \alpha_T^2 \sum_{t=1}^T \norm{x_{t, a_t}}_{V_{t,-a_t}^{-1}}^2 } \nonumber \\
        &= 2\alpha_T\sqrt{T} \sqrt{\sum_{t=1}^T \norm{x_{t, a_t}}_{V_{t,-a_t}^{-1}}^2} \nonumber \\
        &= 2\alpha_T\sqrt{T} \sqrt{\sum_{t=1}^T \norm{x_{t, a_t}}_{V_{t,-a_t}^{-1}}^2 \frac{\norm{x_{t, a_t}}_{V_{t}^{-1}}^2}{\norm{x_{t, a_t}}_{V_{t}^{-1}}^2} } \nonumber\\
        &= 2\alpha_T\sqrt{T} \sqrt{\sum_{t=1}^T \norm{x_{t, a_t}}_{V_{t}^{-1}}^2 \frac{\norm{x_{t, a_t}}_{V_{t,-a_t}^{-1}}^2}{\norm{x_{t, a_t}}_{V_{t}^{-1}}^2} } \nonumber\\
        &\le 2 \alpha_TC\sqrt{T} \sqrt{\sum_{t=1}^T \norm{x_{t, a_t}}_{V_{t}^{-1}}^2 } \nonumber\\
        &\le  2\alpha_TC \sqrt{T} \sqrt{2\log\frac{\textnormal{det}(V_T)}{\textnormal{det} (\lambda I_d)}} \nonumber\\
        \implies \Regret_T\Lp\algo, \bm{\sigma}\Rp &\le 2\alpha_TC \sqrt{2dT\log (1 + TL^2/(d\lambda))} = \tilde{O}(Cd\sqrt{T})  \label{eqn:regretSumLOOUB}.
    }

    The first inequality is due to using Cauchy-Schwarz inequality, where the second inequality follows from the fact that $\alpha_{t, -a}$ increases with $t$.
    The third inequality follows from the fact that there exists a constant $C$ (\cref{def:advantage}) such that $C \ge \max\limits_{a} \frac{\norm{x_{t, a_t}}_{V_{t,-a}^{-1}}}{\norm{x_{t, a_t}}_{V_{t}^{-1}}}$ for all $t \le T$.
    The last two inequalities follow from Lemma 11 and Lemma 10 of \cite{NIPS11_abbasi2011improved}, respectively, and $\alpha_T = \tilde{O}(\sqrt{d})$.      
    For first part of \cref{eqn:regretSumUBSA}, we have $\sum_{t=1}^T 2\alpha_{t,-a_t}\norm{x_{t,a_t}}_{V_{t,-a_t}^{-1}} \le 2\alpha_T C\sqrt{2dT\log (1 + TL^2/(d\lambda))}$ from \cref{eqn:regretSumLOOUB}, and then using Jensen’s inequality for the second part with the fact that $\sum_{a=1}^NS_{T}(a) \le T$, and $\log(1/\delta^y_{t,a}) \le \log(8NT^2/\delta)$ under the confidence schedule. Then, we have
    \al{
        &\Regret_T\Lp\algo, \bm{\sigma}\Rp \le 2\alpha_TC \sqrt{2dT\log (1 + TL^2/(d\lambda))} +  2\sqrt{2R^2NT\log(8NT^2/\delta)} \nonumber \\
        \implies &\Regret_T\Lp\algo, \bm{\sigma}\Rp \le  \tilde{O}(Cd\sqrt{T} + \sqrt{NT}). \label{eqn:regretUBSA}
    }
    
    We now prove that being truthful is an approximate Nash equilibrium for \algo{}. Recall \cref{eqn:pullsLB}, $S_T(a)$ denotes the number of times agent $a$ is selected by \algo{}, which is given as follows:
    \eq{
        \label{eqn:pullsLBSB}
        S_T(a) =\sum_{t=1}^T\ind{a_t= a} \ge S^\star_T(a) - \sum_{t=1}^T\ind{a_t \ne a_t^\star}. 
    }

    To get the lower bound $S_T(a)$, we get the upper bound $\sum_{t=1}^T\ind{a_t \ne a_t^\star}$ when any agent can behave strategically. Recall $\Delta_{a_t} = \Lp \slin{x_{t,a_t^\star}^\star} -   \slin{x_{t,a_t}^\star}\Rp \le \ucb_{t,-a_t}(x_{t,a_t}) - \slin{x_{t,a_t}^\star}$ for $a_t \ne a_t^\star$. We multiply and divide $\ind{a_t \ne a_t^\star}$ by $\Delta_{a_t}$ as follows:
    \al{
        \sum_{t=1}^T\ind{a_t \ne a_t^\star} &=\sum_{t=1}^T\ind{a_t \ne a_t^\star}\frac{\Delta_{a_t}}{\Delta_{a_t}} \nonumber \\
        &=\sum_{a \in \cN}\sum_{t\le T, a_t = a}\ind{a_t \ne a_t^\star, a_t = a}\frac{\Delta_{a_t}}{\Delta_{a_t}} \nonumber \\
        &\le \sum_{a \in \cN}\sum_{t\le T, a_t = a}\ind{a_t \ne a_t^\star, a_t = a}\frac{\ucb_{t,-a}(x_{t,a}) - \slin{x_{t,a}^\star}}{\Delta_{a_t}} \nonumber \\
        &\le \sum_{a \in \cN}\sum_{t\le T, a_t = a}\frac{\ucb_{t,-a}(x_{t,a}) - \slin{x_{t,a}^\star}}{\Delta_{a_t}} \nonumber \\
        \intertext{Using the gap condition $\Delta_{a_t} \ge \Delta_{\min}$, we get}
        &\le \frac{1}{\Delta_{\min}}\sum_{a \in \cN}\sum_{t\le T, a_t = a} \big(\ucb_{t,-a}(x_{t,a}) - \slin{x_{t,a}^\star}\big) \nonumber \\
        \intertext{Using \cref{eqn:regretSumUBSA} with its upper bound, we have}
        \sum_{t=1}^T\ind{a_t \ne a_t^\star} &\le \frac{2}{\Delta_{\min}}\Lp 2\alpha_TC \sqrt{2dT\log (1 + TL^2/(d\lambda))} +  2\sqrt{2R^2NT\log(8NT^2/\delta)} \Rp \nonumber \\
        \implies \sum_{t=1}^T\ind{a_t \ne a_t^\star} &\le \tilde{O}\Lp (Cd\sqrt{T} + \sqrt{NT})/\Delta_{\min}\Rp \label{eqn:subOptPullsSA}.
    }

    Using this upper bound in \cref{eqn:pullsLBSB}, we get the following bound for any agent $a \in \cN$:
    \eq{
        \label{eqn:neLBSA}
        S_T(a) \ge S^\star_T(a) - \tilde{O}\Lp (Cd\sqrt{T} + \sqrt{NT})/\Delta_{\min}\Rp.
    }

    Now, we consider the case where an agent $a$ deviates from the truthful strategy by over-reporting, while all other agents are truthful, and suppose that \cref{assu:NEandregret} holds and no truthful agent is eliminated under $(\sigma'_a, \bm{\sigma}^\star_{-a})$. Splitting the rounds according to whether $a^\star_t = a$, the number of times agent $a$ is selected by \algo{} satisfies
    \eq{
        \label{eqn:pullsUBSA}
        S_T(a) =\sum_{t=1}^T\ind{a_t= a} \le \sum_{t=1}^T\ind{a_t^\star = a} +  \sum_{t=1}^T\ind{a_t \ne a_t^\star}. 
    }

    Using \cref{eqn:subOptPullsSA} in \cref{eqn:pullsUBSA}, we get
    \eq{
        \label{eqn:neUBSA}
        S_T(a) \le S^\star_T(a) + \tilde{O}\Lp (Cd\sqrt{T} + \sqrt{NT})/\Delta_{\min}\Rp.
    }
    
    Here $S^\star_T(a)$ is the same as under the truthful profile: the contexts, and hence the true features, are exogenous, and, since no truthful agent is eliminated, the set of agents competing with $a$ in each round in which $a$ is active is the same as under $\bm{\sigma}^\star$ (if the deviator $a$ itself is eliminated, its count only decreases). Under $\bm{\sigma}^\star$, \cref{thm:regretNE} gives $S_T(a) \ge S^\star_T(a) - \tilde{O}(d\sqrt{T}/\Delta_{\min})$ on $\cE$. Combining this with \cref{eqn:neUBSA}, and using $0 \le S_T(a) \le T$ to account for $\cE^c$ (with $\delta = 1/T$), gives $\EE{S_T(a)\mid\sigma'_a, \bm{\sigma}^\star_{-a}} \le \EE{S_T(a) \mid \bm{\sigma}^\star} + \tilde{O}\big((Cd\sqrt{T} + \sqrt{NT})/\Delta_{\min}\big)$, i.e., \algo{} is $\epsilon$-IC against such deviations.
\end{proof}

The constant $C$ is used in the third inequality in \cref{eqn:regretSumLOOUB}, where $\sum_t \norm{x_{t,a_t}}^2_{V_{t,-a}^{-1}}$ is bounded in terms of $\sum_t \norm{x_{t,a_t}}^2_{V_t^{-1}}$. It then propagates into the $\epsilon$-\neql{} bound through \cref{eqn:subOptPullsSA}. 
We keep $C$ explicit at this step rather than absorbing it into $\tilde{O}(\cdot)$, making the dependence on the mechanism's blind spot transparent: as peer coverage deteriorates, the guarantee degrades continuously. In Example~2 of \cref{asec:loom}, this dependence is tight in the sense that $V_{t,-2}=\lambda$ and $V_t=\lambda+t-1$, yielding $C_{t,2}=\Theta(\sqrt{t})$ and hence $C=\Theta(\sqrt{T})$; the resulting bound degrades to the trivial $\tilde{O}(dT)$.

\para{Advantage constant $C$ in the linear setting.}
The following results hold for every over-reporting strategy profile $\bm{\sigma}$ satisfying \cref{assu:NEandregret}, allowing multiple agents to deviate simultaneously. In this setting, converting the leave-one-out norm to the full-design norm introduces a factor that was previously absorbed into $\tilde{O}(\cdot)$. We instead make this dependence explicit through the advantage constants $C_{t,a}$ and $C$ from \cref{def:advantage}. In the linear setting considered here,
$C_{t,a}= \sup_{x \ne 0}\frac{\norm{x}_{V_{t,-a}^{-1}}}{\norm{x}_{V_t^{-1}}}$ is the advantage of agent $a$ at round $t$; any $C_a \ge C_{t,a}$ for all $t \le T$ is an advantage of agent $a$; and $C \ge \max_{a \in \cN} C_a \ge 1$ is the maximum advantage.

Let $G_{t,a} = \sum_{s<t,\,a_s=a} x_{s,a_s}x_{s,a_s}^\top$, so that $V_t = V_{t,-a}+G_{t,a}$, the linear case admits the exact characterization
\eq{    
    \label{eqnCchar}
    C_{t,a}^2 = \lambda_{\max}\!\Lp V_{t,-a}^{-1}V_t \Rp = 1 + \lambda_{\max}\!\Lp V_{t,-a}^{-1/2} G_{t,a} V_{t,-a}^{-1/2} \Rp.
}
As $V_{t,-a} \succeq \lambda I$ and $\norm{x}_2 \le L$, we have $V_t \preceq \Lp 1 + S_t(a)L^2/\lambda \Rp V_{t,-a}$, yielding the crude bound $C_{t,a} \le \sqrt{1+S_t(a)L^2/\lambda}$.
Thus, $C$ is large precisely when the directions explored by agent $a$ are poorly covered by the remaining agents' observations, capturing the coverage condition underlying LOOM's detection power in \cref{sec:loom}.
Obtaining a sharper theoretical bound on $C$ is challenging because $G_{t,a}$ is policy-dependent, noise-dependent, and generally non-i.i.d. We therefore evaluate $C$ directly in our experiments (\cref{fig:vary_agents_C2strategic_ucb}--\ref{fig:fix_agents_C2strategic_ts} for two over-reporting agents, and \cref{fig:vary_large_agentsC_50_ucb}--\ref{fig:vary_large_dimsC_100_ucb} for $N \in \{100,\ldots,500\}$ and $d \in \{60,\ldots,100\}$). Across these settings, $C$ grows slowly with $T$ and remains nearly flat. These observations are empirical and do not establish a theoretical bound, and these instances have favorable peer coverage since all agents' features are drawn from a common distribution.

\para{When is $C$ bounded?} The crude bound above grows with $S_t(a)$, and it is attained up to constants when agent $a$'s features are orthogonal to those of all other agents: in the $K$-armed case with canonical features, $C_{t,a}^2 = 1 + S_t(a)/\lambda$ even under truthful reporting. The following coverage condition is sufficient for $C = O(1)$.
\begin{prop}[Coverage condition]
    \label{prop:coverage}
    Suppose there is $\kappa > 0$ such that $\lambda_{\min}(V_{t,-a}) \ge \lambda + \kappa(t-1)$ for all $a \in \cN$ and $1 \le t \le T$. Then $C_{t,a} \le \max\{1, L/\sqrt{\kappa}\}$ for all such $a$ and $t$, so one can take $C = \max\{1, L/\sqrt{\kappa}\}$.
\end{prop}
\begin{proof}
    For any $x \ne 0$, $\norm{x}^2_{V_{t,-a}^{-1}} \le \norm{x}_2^2/\lambda_{\min}(V_{t,-a})$ and $\norm{x}^2_{V_t^{-1}} \ge \norm{x}_2^2/\lambda_{\max}(V_t)$, where $\lambda_{\max}(V_t) \le \lambda + (t-1)L^2$ since $\norm{x_{s,a_s}}_2 \le L$. Hence $C^2_{t,a} \le (\lambda + (t-1)L^2)/(\lambda + \kappa(t-1)) \le \max\{1, L^2/\kappa\}$. \qedhere
\end{proof}
The condition holds, for example, when the other agents are selected in a constant fraction of rounds and their (reported) features have a covariance matrix bounded away from zero, in which case it holds with high probability by matrix concentration. Example~2 in \cref{asec:loom} violates it, since there $V_{t,-2} = \lambda$ for all $t$, and $C = \Theta(\sqrt{T})$. When the condition fails, \cref{thm:regretAllNE} and \cref{thm:regretNonAllNE} remain valid but may be vacuous; Example~2 shows that this cannot be avoided by any leave-one-out test, since misreporting in a direction not covered by the other agents is undetectable from their data. Deriving verifiable, instance-level conditions under which $C$ is bounded for the adaptively collected data of \algo{} remains an open problem.

    \section{Non-linear Reward Function}
    \label{asec:non_linear}

\begin{table}[H]
	\caption{Examples of different $h(x, \cO_t)$ values for some LOOM-compatible contextual bandit algorithms, using notations from the original papers.}
    \label{table:hfunc}
	\label{table}
	\centering
	\setlength{\arrayrulewidth}{0.25mm}
	\setlength{\tabcolsep}{2pt}
	\renewcommand{\arraystretch}{2}		
	\begin{tabular}{|c|c|}
		\hline
		Contextual bandit algorithm     & $h(x, \cO_t)$     \\
		\hline
		Lin-UCB \citep{AISTATS11_chu2011contextual}    & $\Lp R\sqrt{d\log\left( \nicefrac{1+\frac{tL^2}{\lambda}}{\delta}\right)} + \lambda^{\nicefrac{1}{2}}S \Rp \norm{x}_{{V}_{t}^{-1}}$     \\
		\hline 
		GLM-UCB \citep{ICML17_li2017provably} & $\sqrt{\frac{d}{2} \log(1 + 2t/d) + \log(1/\delta)} \frac{\norm{x}_{{V}_{t}^{-1}}}{\kappa}$  \\
		\hline 
		IGP-UCB \citep{ICML17_chowdhury2017kernelized}     & $\left(\sqrt{2(\gamma_{t-1} + 1 +  \log(1/\delta))} + B\right)\sigma_{t-1}(x)$ \\
		\hline
	\end{tabular}
	
\end{table}

\subsection{Theoretical Results}
\label{asec:non_linear_analysis}
We first derive results similar to \cref{lem:linConfidenceBound} and \cref{lem:linLooConfBound} for contextual bandit problems with non-linear reward functions. For brevity, we ignore $\kA$ in $f_t^{\kA}$ and use only $f_t$.

\begin{lem}
    \label{lem:nonLinConfidenceBound}
    Let $\kA$ be a LOOM-compatible contextual bandit algorithm for which $|f_t(x) - f(x)| \le h(x, \cO_t)$ holds with probability at least $1-\delta$ for any $x \in \cX$.
    Then, for all $t\ge 1$, 
    \begin{enumerate}
        \item With probability at least $1-\delta$,  $  
            \forall x \in \cX: f(x) \le \ucb_t(x)= f_t(x) + h(x, \cO_t).
        $

        \item Similarly, with probability at least $1-\delta$,
        $  
            \forall x \in \cX:f(x) \ge \lcb_t(x)= f_t(x) - h(x, \cO_t).
        $
    \end{enumerate}

\end{lem}

\begin{proof}
    The proofs of these results follow directly from the first part of \cref{defi:loomCompatible}. For completeness, we provide the proof of the first part, which follows from straightforward algebraic simplifications of $|f_t(x) - f(x)| \le h(x, \cO_t)$:
    \als{
        &|f_t(x) - f(x)| \le h(x, \cO_t) \\
        \implies &|f(x) - f_t(x)| \le h(x, \cO_t) \\
        \implies &f(x) - f_t(x) \le h(x, \cO_t) \\
        \implies &f(x) \le f_t(x) + h(x, \cO_t).
    }

    Similarly, the proof of the second part follows with some simple algebraic simplifications of $|f_t(x) - f(x)| \le h(x, \cO_t)$:
    \als{
        &|f_t(x) - f(x)| \le h(x, \cO_t) \\
        \implies &f_t(x) - f(x) \le h(x, \cO_t) \\
        \implies &f_t(x) - h(x, \cO_t)\le  f(x) \\
        \implies &f(x) \ge f_t(x) - h(x, \cO_t).  \qedhere
    }
\end{proof}

\begin{lem}
    \label{lem:nonLinLooConfidenceBound}
    Let $\kA$ be a LOOM-compatible contextual bandit algorithm for which $|f_{t,-a}(x) - f(x)| \le h(x, \cO_{t,-a})$ holds with probability at least $1-\delta$ for any $x \in \cX$.
    Then, for all $t\ge 1$, 
    \begin{enumerate}
        \item With probability at least $1-\delta$,  
        \eqs{   
            \forall x \in \cX: f(x) \le \ucb_{t,-a}(x)= f_{t,-a}(x) + h(x, \cO_{t,-a}).
        }

    \item Similarly, with probability at least $1-\delta$,
        \eqs{   
            \forall x \in \cX:f(x) \ge \lcb_{t,-a}(x)= f_{t,-a}(x) - h(x, \cO_{t,-a}).
        }
    \end{enumerate}
\end{lem}

\begin{proof}
    The proofs of these results follow directly from the second part of \cref{defi:loomCompatible}. For completeness, we provide the proof of the first part, which follows from straightforward algebraic simplifications of $|f_{t,-a}(x) - f(x)| \le h(x, \cO_{t,-a})$:
    \als{
        &|f_{t,-a}(x) - f(x)| \le h(x, \cO_{t,-a}) \\
        \implies &|f(x) - f_{t,-a}(x)| \le h(x, \cO_{t,-a}) \\
        \implies &f(x) - f_{t,-a}(x) \le h(x, \cO_{t,-a}) \\
        \implies &f(x) \le f_{t,-a}(x) + h(x, \cO_{t,-a}).
    }

    Similarly, the proof of the second part follows with some simple algebraic simplifications of $|f_{t,-a}(x) - f(x)| \le h(x, \cO_{t,-a})$:
    \als{
        &|f_{t,-a}(x) - f(x)| \le h(x, \cO_{t,-a}) \\
        \implies &f_{t,-a}(x) - f(x) \le h(x, \cO_{t,-a}) \\
        \implies &f_{t,-a}(x) - h(x, \cO_{t,-a})\le  f(x) \\
        \implies &f(x) \ge f_{t,-a}(x) - h(x, \cO_{t,-a}).  \qedhere
    }
\end{proof}

\regretNonNE* 
\begin{proof}
    When all agents report truthfully, our algorithm is the same as the contextual bandit algorithm $\kA$ with a mechanism for identifying strategic agents that eliminates no agent on the event $\cE$, which holds with probability at least $1-\delta$ (\cref{thm:optgtmarmsstayactive}). For completeness, we first recall the definition of the regret of \algo{} as follows:
    \eq{
        \label{eqn:regretNonUBTrue}
        \Regret_T\Lp\algo(\kA), \bm{\sigma^\star}\Rp = \sum_{t=1}^T \Lp f(x_{t,a_t^\star}^\star) - f(x_{t,a_t}^\star) \Rp.
    }

    Since the true feature vector is the same as the reported context-arm feature vector (i.e., $x_{t,a}^\star = x_{t,a}$), we can start by upper-bounding the difference $f(x_{t,a_t^\star}^\star) - f(x_{t,a_t}^\star)$ as follows:
    \al{
        f(x_{t,a_t^\star}^\star) - f(x_{t,a_t}^\star) 
        &\le \ucb_t{(x_{t,a_t^\star})} -f(x_{t,a_t}^\star) \nonumber\\
        &\le \ucb_t{(x_{t,a_t})} -f(x_{t,a_t}^\star) \hspace{5mm} \Lp \text{as $\ucb_t{(x_{t,a_t^\star})} \le \ucb_t{(x_{t,a_t})}$} \Rp \nonumber\\
        &= f_t(x_{t,a_t}) + h(x_{t,a_t}, \cO_t) -f(x_{t,a_t}^\star) \nonumber\\
        &\le |f_t(x_{t,a_t}) - f(x_{t,a_t}^\star)| + h(x_{t,a_t}, \cO_t) \nonumber\\
        &\le h(x_{t,a_t}, \cO_t) + h(x_{t,a_t}, \cO_t) \nonumber\\
        \implies f(x_{t,a_t^\star}^\star) - f(x_{t,a_t}^\star)  &\le 2h(x_{t,a_t}, \cO_t).  \label{eqn:gapNonUB}
    }
    Note that $f_t$ is an estimator of the reward function $f$ as the true feature vector is the same as the reported context-arm feature vector.
    After using the upper bound given in \cref{eqn:gapNonUB} into \cref{eqn:regretNonUBTrue}, we get an upper bound on the regret as follows:
    \al{
        \Regret_T\Lp\algo(\kA), \bm{\sigma^\star}\Rp &= \sum_{t=1}^T \Lp f(x_{t,a_t^\star}^\star) - f(x_{t,a_t}^\star) \Rp \nonumber \\ 
        &\le \sum_{t=1}^T 2h(x_{t,a_t}, \cO_t) \nonumber \\ 
        &= 2\sum_{t=1}^T h(x_{t,a_t}, \cO_t) \nonumber\\
        &\le 2 \sqrt{T\sum_{t=1}^T \Lb h(x_{t,a_t}, \cO_t) \Rb^2} \nonumber\\
        \implies \Regret_T\Lp\algo(\kA), \bm{\sigma^\star}\Rp &\le 2\sqrt{T} \sqrt{\sum_{t=1}^T \Lb h(x_{t,a_t}, \cO_t) \Rb^2} = \tilde{O}\Lp \tilde{d}\sqrt{T} \Rp.
    }
    The first inequality directly follows from \cref{eqn:gapNonUB}. The second inequality is due to using the Cauchy-Schwarz inequality. The last equality is due to $\sqrt{\sum_{t=1}^T \Lb h(x_{t,a_t}, \cO_t) \Rb^2 } = \tilde{O}(\tilde{d}\log T)$.      
    
    We now lower bound the number of selections of every agent under the truthful profile (this does not by itself establish an approximate Nash equilibrium; see \cref{thm:regretNonAllNE}). Recall, $S_T(a)$ denotes the number of times agent $a$ is selected by \algo{}, which is given as follows:
    \al{
        S_T(a) &=\sum_{t=1}^T\ind{a_t= a} \nonumber \\
        &=\sum_{t=1}^T\ind{a_t= a, a_t^\star = a} +  \sum_{t=1}^T\ind{a_t= a, a_t^\star \ne a} \nonumber \\
        &\ge \sum_{t=1}^T\ind{a_t^\star = a} - \sum_{t=1}^T\ind{a_t^\star = a, a_t \ne a} \nonumber \\
        &\ge \sum_{t=1}^T\ind{a_t^\star = a} - \sum_{t=1}^T\ind{a_t \ne a_t^\star} \nonumber \\
        \implies S_T(a) &\ge S^\star_T(a) - \sum_{t=1}^T\ind{a_t \ne a_t^\star} \label{eqn:pullsNonLB}. 
    }

    To get the lower bound $S_T(a)$, we get an upper bound $\sum_{t=1}^T\ind{a_t \ne a_t^\star}$. Let $\Delta_{a_t} = \Lp f(x_{t,a_t^\star}^\star) -   f(x_{t,a_t}^\star)\Rp > 0$ for $a_t \ne a_t^\star$. 
    We multiply and divide $\ind{a_t \ne a_t^\star}$ by $\Delta_{a_t}$ and then use inequality in \cref{eqn:gapNonUB}, i.e., $\Delta_{a_t} \le 2h(x_{t,a_t}, \cO_t)$ as follows:
    \al{
        \sum_{t=1}^T\ind{a_t \ne a_t^\star} &=\sum_{t=1}^T\ind{a_t \ne a_t^\star}\frac{\Delta_{a_t}}{\Delta_{a_t}} \nonumber \\
        &\le \sum_{t=1}^T\ind{a_t \ne a_t^\star}\frac{2h(x_{t,a_t}, \cO_t)}{\Delta_{a_t}} \hspace{5mm} (\text{as } x_{t,a}^\star = x_{t,a}) \nonumber\\
        &\le \sum_{t=1}^T\frac{2h(x_{t,a_t}, \cO_t)}{\Delta_{a_t}} \nonumber \\ 
        &= \sum_{t=1}^T\frac{1}{\Delta_{a_t}} 2h(x_{t,a_t}, \cO_t) \nonumber \\
        &\le \sum_{t=1}^T\frac{1}{\Delta_{\min}} 2h(x_{t,a_t}, \cO_t) \nonumber \\
        &= \frac{2}{\Delta_{\min}}\sum_{t=1}^T h(x_{t,a_t}, \cO_t) \nonumber \\
        &\le  \frac{2}{\Delta_{\min}} \sqrt{T}\sqrt{\sum_{t=1}^T \Lb h(x_{t,a_t}, \cO_t)\Rb^2} \nonumber \\
        \implies \sum_{t=1}^T\ind{a_t \ne a_t^\star} &\le \tilde{O}\Lp \tilde{d}\sqrt{T}/\Delta_{\min} \Rp \label{eqn:subOptPullsNon}.
    }

    Note that $\Delta_{\min} \le \Delta_{a_t}$ for every round with $a_t \ne a^\star_t$ by the gap condition in \cref{sec:cobra}, which gives the last inequality. 
    Although using $\Delta_{\min}$ loosens the upper bound, we use this to get dependence on $T$. 
    Using this upper bound in \cref{eqn:pullsNonLB}, we get the following bound for any agent $a \in \cN$:
    \eq{
        \label{eqn:neNonLB}
        S_T(a) \ge S^\star_T(a) - \tilde{O}\Lp \tilde{d}\sqrt{T}/\Delta_{\min}\Rp.
    }
    This completes the proof. As in the linear case, \cref{eqn:subOptPullsNon} is derived under truthful reporting and cannot be used to bound the gain of a deviating agent; deviations are covered by \cref{thm:regretNonAllNE}.
\end{proof}

The proof of \cref{thm:regretNonNE} relies entirely on $h(x,\cO_t)$ and, like the proof of \cref{thm:regretNE}, never compares it with $h(x,\cO_{t,-a})$. Thus, \cref{thm:optgtmarmsstayactive}, \cref{thm:regretNE}, and \cref{thm:regretNonNE}, which establish the soundness of the LOOM test and the guarantees under truthful reporting, are independent of the advantage constant $C$ in \cref{def:advantage}.
We need the following result that upper bounds the total amount of regret that an agent $a$ can exert before being identified by LOOM. 
\begin{lem}
    \label{lem:regretNonSA}
    Let $\ucb_{t,-a}(\ox_{s,a})  = f_{t,-a}(x_{s,a}) + h(x_{s,a}, \cO_{t,-a})$, where $x_{s,a}$ is the arm feature vector associated with agent $a$ in the round $s$. Then, with probability at least $1-\delta_{t,a}^x - \delta_{t,a}^y$,
    \als{
        \sum_{s \le t\colon a_s = a} \Lp \ucb_{t,-a}(\ox_{s,a}) - f(x_{s,a}^\star) \Rp \leq \sum_{s\le t, a_s = a} 2h(x_{s,a}, \cO_{t,-a}) + 2\sqrt{2R^2S_t(a)\log(1/\delta_{t,a}^y)}.
    } 
\end{lem}

\begin{proof}
    Using \cref{eqn:testSA} with \cref{lem:nonLinLooConfidenceBound} for non-linear reward function that holds with probability at least $1-\delta_{t,a}^x$,
    we get:
    \als{
        \sum_{s\le t, a_s = a} \Lp f_{t,-a}(x_{s,a}) - h(x_{s,a}, \cO_{t,-a}) \Rp &\le \sum_{s\le t,a_s = a} y_s + \sqrt{2R^2S_t(a)\log(1/\delta_{t,a}^y)} \\
        &\le \sum_{s\le t, a_s = a} f(x_{s, a}^\star) + \sqrt{2R^2S_t(a)\log(1/\delta_{t,a}^y)} \\
        &\qquad + \sqrt{2R^2S_t(a)\log(1/\delta_{t,a}^y)}\\
        \implies \sum_{s\le t, a_s = a} \Lp f_{t,-a}(x_{s,a}) - f(x_{s, a}^\star) \Rp &\le \sum_{s\le t, a_s = a} h(x_{s,a}, \cO_{t,-a}) + 2\sqrt{2R^2S_t(a)\log(1/\delta_{t,a}^y)}.
    }
    The second inequality follows from \cref{eqn:rewardUCB} by using an upper bound on $\sum_{s\le t,a_s = a} y_s$ that holds with probability $1-\delta_{t,a}^y$.
    Now we prove the second part of the result by replacing $f_{t,-a}(x_{s,a})$ by $\ucb_{t,-a}(\ox_{s,a})  - h(x_{s,a}, \cO_{t,-a})$ and we get
    \als{
        &\sum_{s\le t, a_s = a} \Lp \ucb_{t,-a}(\ox_{s,a})  - h(x_{s,a}, \cO_{t,-a}) - f(x_{s, a_s}^\star)\Rp \\
        &\qquad\le \sum_{s\le t, a_s = a} h(x_{s,a}, \cO_{t,-a}) + 2\sqrt{2R^2S_t(a)\log(1/\delta_{t,a}^y)} \\
        \implies &\sum_{s\le t, a_s = a} \Lp \ucb_{t,-a}(\ox_{s,a}) - f(x_{s, a_s}^\star)\Rp \\
        &\qquad\le \sum_{s\le t, a_s = a} 2h(x_{s,a}, \cO_{t,-a}) + 2\sqrt{2R^2S_t(a)\log(1/\delta_{t,a}^y)}. \qedhere
    }
\end{proof}

\regretNonAllNE*
\begin{proof}
    Recall $\cA_t$ denotes the set of arms' features corresponding to the active agents in the round $t$. The regret of \algo{} for an over-reporting profile $\bm{\sigma}$ satisfying \cref{assu:NEandregret} is given as follows:
    \al{
        \label{eqn:regretNonUB}
        \Regret_T\Lp\algo(\kA), \bm{\sigma}\Rp = \sum_{t=1}^T \Lp f(x_{t,a_t^\star}^\star) - f(x_{t,a_t}^\star)\Rp = \sum_{t=1}^T \Lp \max_{a \in \cA_t} f(x_{t,a}^\star) - f(x_{t,a_t}^\star)\Rp.
    }
    
    Under \cref{assu:NEandregret}, if \algo{} selects $a_t \in \cA_t \colon a_t \ne a_t^\star$,
    we have $f(x_{t,a_t^\star}^\star) \le f(x_{t,a_t^\star}) \le \ucb_{t}(x_{t,a_t^\star}) \le \ucb_{t}(x_{t,a_t})$.
    Using $f(x_{t,a_t^\star}^\star) \le \ucb_{t}(x_{t,a_t})$ inequality with \cref{lem:regretNonSA}, we have
    \al{
        \Regret_T\Lp\algo(\kA), \bm{\sigma}\Rp &= \sum_{t=1}^T \Lp f(x_{t,a_t^\star}^\star) - f(x_{t,a_t}^\star)\Rp \nonumber \\
        &\le \sum_{t=1}^T \Lp f(x_{t,a_t^\star}) - f(x_{t,a_t}^\star)\Rp  \hspace{5mm}(\text{agents are over-reporting}) \nonumber \\
        &\le \sum_{t=1}^T \Lp \ucb_t(x_{t,a_t^\star}) - f(x_{t,a_t}^\star)\Rp \hspace{5mm} (\text{\cref{assu:NEandregret}}) \nonumber \\
        &\le \sum_{t=1}^T \Lp \ucb_t(x_{t,a_t}) - f(x_{t,a_t}^\star) \Rp \hspace{5mm} (\text{as selected arm is $a_t$}) \nonumber \\
        &\le \sum_{t=1}^T \Lp \ucb_{t,-a_t}(x_{t,a_t}) - f(x_{t,a_t}^\star) \Rp \hspace{5mm} (\text{ensured by the Step~\ref{step:check} of \algo{}}) \nonumber \\
        &= \sum_{t=1}^T \sum_{a\in\cN} \ind{a_t = a} \Lp \ucb_{t,-a}(x_{t,a}) - f(x_{t,a}^\star) \Rp \nonumber \\
        &=  \sum_{a=1}^N \sum_{t\le T, a_t = a} \Lp \ucb_{t,-a}(x_{t,a}) - f(x_{t,a}^\star) \Rp \nonumber \\
        &\le  \sum_{a=1}^N \Lp \sum_{t\le T, a_t = a} 2h(x_{t,a}, \cO_{t,-a}) + 2\sqrt{2R^2S_{T}(a)\log(1/\delta_{T,a}^y)} \Rp \nonumber \\
        &= \sum_{a=1}^N\sum_{t\le T, a_t = a} 2h(x_{t,a}, \cO_{t,-a}) + \sum_{a=1}^N 2\sqrt{2R^2S_{T}(a)\log(1/\delta_{T,a}^y)} \nonumber \\
        &= \sum_{t=1}^T 2h(x_{t,a_t}, \cO_{t,-a_t}) + \sum_{a=1}^N 2\sqrt{2R^2S_{T}(a)\log(1/\delta_{T,a}^y)} \nonumber \\
        &\le 2\sqrt{T}\sqrt{\sum_{t=1}^T \Lb h(x_{t,a_t}, \cO_{t,-a_t}) \Rb^2} + 2\sum_{a=1}^N \sqrt{2R^2S_{T}(a)\log(1/\delta_{T,a}^y)} \nonumber \\
        &\le 2\sqrt{T}\sqrt{\sum_{t=1}^T \Lb h(x_{t,a_t}, \cO_{t,-a_t}) \Rb^2} +  2\sqrt{2R^2NT\log(8NT^2/\delta)} \nonumber \\
        &\le 2C\sqrt{T}\sqrt{\sum_{t=1}^T \Lb h(x_{t,a_t}, \cO_{t}) \Rb^2} +  2\sqrt{2R^2NT\log(8NT^2/\delta)} \nonumber \\
        \implies \Regret_T\Lp\algo(\kA), \bm{\sigma}\Rp &= \tilde{O} \Lp C\tilde{d}\sqrt{T} + \sqrt{NT} \Rp. \label{eqn:regretNonSumUBSA}
    }
    The fourth-last inequality follows from \cref{lem:regretNonSA}.
    The third-last inequality is due to using Cauchy-Schwarz inequality.
    The second-last inequality follows from Jensen's inequality with the fact that $\sum_{a=1}^NS_T(a) \le T$ and from $\log(1/\delta^y_{T,a}) \le \log(8NT^2/\delta)$ under the confidence schedule (\cref{asec:schedule}).
    The last inequality is the non-linear counterpart of the third inequality in \cref{eqn:regretSumLOOUB} and is the \emph{only} step in the proof that compares the leave-one-out confidence width with the full-design width. It therefore introduces the advantage constant $C$ from \cref{def:advantage}. Specifically, for all $a$ and $t \ge 1$, $C \ge \sup_{x \in \cX}\frac{h(x,\cO_{t,-a})}{h(x,\cO_t)}$, which implies $\sum_t \big[h(x_{t,a_t},\cO_{t,-a})\big]^2 \le C^2 \sum_t \big[h(x_{t,a_t},\cO_t)\big]^2$.
    As in the linear case, we keep this dependence on $C$ explicit in the final regret bound. 

    We now prove that being truthful is an approximate Nash equilibrium for \algo{}. Recall \cref{eqn:pullsNonLB}, $S_T(a)$ denotes the number of times agent $a$ is selected by \algo{}, which is given as follows:
    \eq{
        \label{eqn:pullsNonLBSB}
        S_T(a) =\sum_{t=1}^T\ind{a_t= a} \ge S^\star_T(a) - \sum_{t=1}^T\ind{a_t \ne a_t^\star}. 
    }

    To get the lower bound $S_T(a)$, we get the upper bound $\sum_{t=1}^T\ind{a_t \ne a_t^\star}$ when any agent can behave strategically. Recall $\Delta_{a_t} = \Lp f(x_{t,a_t^\star}^\star) -   f(x_{t,a_t}^\star)\Rp \le \ucb_{t,-a_t}(x_{t,a_t}) - f(x_{t,a_t}^\star)$ for $a_t \ne a_t^\star$. We multiply and divide $\ind{a_t \ne a_t^\star}$ by $\Delta_{a_t}$ as follows:
    \al{
        \sum_{t=1}^T\ind{a_t \ne a_t^\star} &=\sum_{t=1}^T\ind{a_t \ne a_t^\star}\frac{\Delta_{a_t}}{\Delta_{a_t}} \nonumber \\
        &=\sum_{a \in \cN}\sum_{t\le T, a_t = a}\ind{a_t \ne a_t^\star, a_t = a}\frac{\Delta_{a_t}}{\Delta_{a_t}} \nonumber \\
        &\le \sum_{a \in \cN}\sum_{t\le T, a_t = a}\ind{a_t \ne a_t^\star, a_t = a}\frac{\ucb_{t,-a}(x_{t,a}) - f(x_{t,a}^\star)}{\Delta_{a_t}} \nonumber \\
        &\le \sum_{a \in \cN}\sum_{t\le T, a_t = a}\frac{\ucb_{t,-a}(x_{t,a}) - f(x_{t,a}^\star)}{\Delta_{a_t}} \nonumber \\
        \intertext{Using the gap condition $\Delta_{a_t} \ge \Delta_{\min}$, we get}
        &\le \frac{1}{\Delta_{\min}}\sum_{a \in \cN}\sum_{t\le T, a_t = a} \big(\ucb_{t,-a}(x_{t,a}) - f(x_{t,a}^\star)\big) \nonumber \\
        \intertext{Using \cref{eqn:regretNonSumUBSA} with its upper bound, we have}
        \sum_{t=1}^T\ind{a_t \ne a_t^\star} &\le \frac{2}{\Delta_{\min}}\Lp 2\sqrt{T}\sqrt{\sum_{t=1}^T \Lb h(x_{t,a_t}, \cO_{t,-a_t}) \Rb^2} +  2\sqrt{2R^2NT\log(8NT^2/\delta)} \Rp \label{eqn:subOptPullsNonSA} \\
        \implies \sum_{t=1}^T\ind{a_t \ne a_t^\star} &= \tilde{O}\Lp (C\tilde{d}\sqrt{T} + \sqrt{NT})/\Delta_{\min}\Rp \nonumber.
    }

    Using this upper bound in \cref{eqn:pullsNonLBSB}, we get the following bound for any agent $a \in \cN$:
    \eq{
        \label{eqn:neNonLBSA}
        S_T(a) \ge S^\star_T(a) - \tilde{O}\Lp (C\tilde{d}\sqrt{T} + \sqrt{NT})/\Delta_{\min}\Rp.
    }

    Now, we consider the case where an agent $a$ deviates from the truthful strategy by over-reporting while all other agents are truthful, and suppose that \cref{assu:NEandregret} holds and no truthful agent is eliminated. Splitting the rounds according to whether $a^\star_t = a$, the number of times agent $a$ is selected by \algo{} satisfies
    \eq{
        \label{eqn:pullsNonUBSA}
        S_T(a) =\sum_{t=1}^T\ind{a_t= a} \le \sum_{t=1}^T\ind{a_t^\star = a} +  \sum_{t=1}^T\ind{a_t \ne a_t^\star}. 
    }

    Using \cref{eqn:subOptPullsNonSA} in \cref{eqn:pullsNonUBSA}, we get
    \eq{
        \label{eqn:neNonUBSA}
        S_T(a) \le S^\star_T(a) + \tilde{O}\Lp (C\tilde{d}\sqrt{T} + \sqrt{NT})/\Delta_{\min}\Rp.
    }
    
    Combining \cref{eqn:neNonUBSA} with the lower bound of \cref{thm:regretNonNE} under $\bm{\sigma}^\star$, exactly as in the proof of \cref{thm:regretAllNE}, shows that \algo{}$(\kA)$ is $\tilde{O}\big((C\tilde{d}\sqrt{T} + \sqrt{NT})/\Delta_{\min}\big)$-IC against such deviations.
\end{proof}

\begin{rem}
    Assuming sufficient peer coverage (i.e., $C = O(1)$; without it, Example~2 in \cref{asec:loom} shows a $\Theta(T)$ gain), a strategic agent may attempt to introduce a small but persistent misreport that remains within LOOM's confidence bounds over finite horizons. However, any systematic bias accumulates linearly in the LOOM test statistic, so to avoid elimination with high probability, the misreport magnitude must decay as $O(1/\sqrt{S_t(a)})$. This implies that the total additional utility gained through such deviations is at most $\tilde O(\sqrt{T})$. Consequently, although arbitrarily small misreports cannot be statistically ruled out, they yield vanishing per-round gains, preserving approximate IC. \end{rem}

In the non-linear case, the constant $C$ enters at the last inequality of \cref{eqn:regretNonSumUBSA}, where
$\sum_t [h(x_{t,a_t},\cO_{t,-a})]^2$ is bounded in terms of $\sum_t [h(x_{t,a_t},\cO_t)]^2$.
This is the direct non-linear analog of the corresponding step in the linear analysis above.
In summary, $C$ is needed only for the guarantees under over-reporting, \cref{thm:regretAllNE} and \cref{thm:regretNonAllNE}. It does not enter \cref{thm:optgtmarmsstayactive}, which establishes soundness of the test, or \cref{thm:regretNE} and \cref{thm:regretNonNE}, which establish regret and selection guarantees under truthful reporting. Their proofs depend entirely on $V_t$ in the linear case and on $h(x,\cO_t)$ in the non-linear case, without comparing either quantity to its leave-one-out counterpart.

    \section{Additional Experiments and Details}
    \label{asec:more_experiments}

This section provides additional details from \cref{sec:experiments}, followed by further experimental results.

\para{TS-based algorithm.}
Motivated by the empirical advantages of Thompson Sampling (TS) over UCB-based bandit algorithms \citep{ICLR21_zhang2020neural,NIPS11_chapelle2011empirical, ICML13_agrawal2013thompson}, we also propose a TS-based variant, \algo{} (TS). 
This algorithm closely mirrors \algo{} (UCB), differing only in the arm selection step (Line~\ref{step:max_selection}). To obtain a TS-based reward estimate, the algorithm first samples a reward function parameter $\tilde\theta_t \sim \bN\left(\hat\theta_t, \beta_t^2 V_t^{-1} \right)$, where $\bN$ denotes the normal distribution and $\beta_t=R\sqrt{9d\log\Lp t/\delta \Rp}$ \citep{ICML13_agrawal2013thompson}. 
Using $\tilde\theta_t$, the TS-based reward estimate, i.e., $\text{TS}_t(x_{t,a}) = x_{t,a}^\top \tilde\theta_t$, replaces $\text{UCB}_{t}(x_{t,a})$ when computing the optimistic reward in Line \ref{step:max_selection}.
We compare and demonstrate the superior empirical performance of \algo{} (TS) in \cref{sec:experiments}.

\para{Why synthetic problem instances.} 
Evaluating whether \algo{} actually {incentivizes} truthful reporting requires simulating strategic behavior, including access to each agent's true features $x^\star_{t,a}$ and the ability to choose a misreport $\ox_{t,a}$ in response to the deployed algorithm. To the best of our knowledge, no publicly available real-world dataset contains the counterfactual information required for this evaluation, namely, both the true and reported features of a strategic agent. Synthetic instances, which are standard in the bandit literature, therefore provide a controlled setting for isolating and validating our theoretical predictions. We evaluate regret against Lin-UCB, Lin-TS, and OptGTM (\cref{fig:compare_cumm_regret}), the effect of manipulation strength (\cref{fig:compare_strategic_manipulations}), scaling with the number of agents $N$ and dimension $d$ (\cref{fig:compare_NandD} and \cref{fig:compareLargeExp}), two simultaneously strategic agents and the corresponding measured advantage $C$ (\cref{fig:compare_overReport}), and non-linear reward functions (\cref{fig:compare_nonlinear_cumm_regret}).

\para{Strategic over-reporting.}
We consider two types of strategic manipulations:
\emph{(I) Feature adaptation:} The strategic agent updates the arm features it reports based on past selection outcomes using a finite-difference stochastic gradient ascent update. Specifically, the agent receives a binary feedback signal: $1$ if it was selected in the previous round and $0$ otherwise. The agent uses this feedback to iteratively adjust its reported features to increase the probability of being selected in future rounds.
\emph{(II) Systematic over-reporting:} To maximize the likelihood of selection, the strategic agent over-reports its feature vector according to $x = (1 + \Delta_x) x^\star$, where $x^\star$ denotes the true arm features. The agent maintains an estimate of the optimal over-reporting factor, $\hat{\Delta}_x$, which guides the extent of over-reporting.

\para{Agents' observation space.}
Each agent observes only a binary signal indicating whether it was selected in the previous round, together with its share of the realized reward (e.g., an e-commerce platform may retain a fee before passing the remaining revenue to the agent). An agent does \emph{not} observe other agents' reports, the learner's estimates $f_t$ or $f_{t,-a}$, or the underlying reward function $f$. This is precisely the information available to the strategic agents in our experiments: the type~I agent in \cref{sec:experiments} adapts its reports using finite-difference stochastic gradient ascent based solely on the binary selection signal (see ``Strategic manipulations via feature adaptation'' below).
Importantly, $f$ is unknown to both the learner and the agents and is not required as part of the agents' information. Hence, our guarantees apply even to agents with strictly less information, including the type~\textit{I} agent described above.

\para{Strategic manipulations via feature adaptation.}
We want to highlight that our proposed algorithm, \algo{}, operates without prior knowledge of the specific nature of these manipulations, which we model as equivalent to over-reporting. Under the assumption that the agent engages only in strategic over-reporting, the objective is to identify any such over-reporting behavior. 
For feature adaptation, the agent can strategically manipulate and optimize its features against the deployed algorithm using only binary feedback (whether it was selected or not) in each round as follows: Let $x_a^\star$ be the true arm-features of agent $a$. The agent can over-report a feature $\tilde{x} = x_a^\star + \nu\Delta$, where $\Delta \in \mathbb{R}^d$ is a bounded perturbation such that $\norm{\Delta}_2 \le \Delta_{\max}$ and $\nu$ is the learning rate (distinct from the manipulation level $\eta$ in \cref{fig:compare_strategic_manipulations}). The agent's goal is to learn a manipulation strategy $\Delta$ using binary feedback that increases its probability of being selected. To do so, the agent uses a finite-difference stochastic gradient ascent update.

\para{Experiments with non-linear reward.}
We also compare the performance of our proposed algorithm, \algo{}, for contextual bandit problems with non-linear reward functions.
For this experiment, we adapt problem instances with non-linear reward functions from those used for linear functions in \cref{sec:experiments}. We apply a polynomial kernel of degree $2$ to transform the context-arm feature vectors to introduce non-linearity.
The constant terms (i.e., the $1$'s) resulting from this transformation are removed. As an example, a sample $4$-d feature vector $x = (x_1, x_2, x_3,x_4)$ is transformed into a $14$-d feature vector: $x^\prime = (x_1, x_2, x_3,x_4, x_1x_2, x_1x_3, x_1x_4, x_2x_3, x_2x_4, x_3x_4, x_1^2, x_2^2, x_3^2, x_4^2)$.
As expected, our algorithms \algo{} based on UCB and TS-based contextual non-linear bandit algorithms (prefixed with `n') outperform all the baselines in the strategic instances (adapted to non-linear setting, also prefixed with `n') as shown in \cref{fig:compare_nonlinear_cumm_regret}.
These results are observed across various problem instances, where only the reward function is varied while all other parameters remain unchanged, except for the number of rounds, which is set to $T=2000$.
We further observe that \algo{} with TS outperforms its UCB-based counterpart in the strategic instances.

\begin{figure}[!ht]
	\vspace{-5mm}
    \centering
	\subfloat[Square (Truthful setting)]{\label{fig:nonlinear_truthful}
		\includegraphics[width=0.245\linewidth]{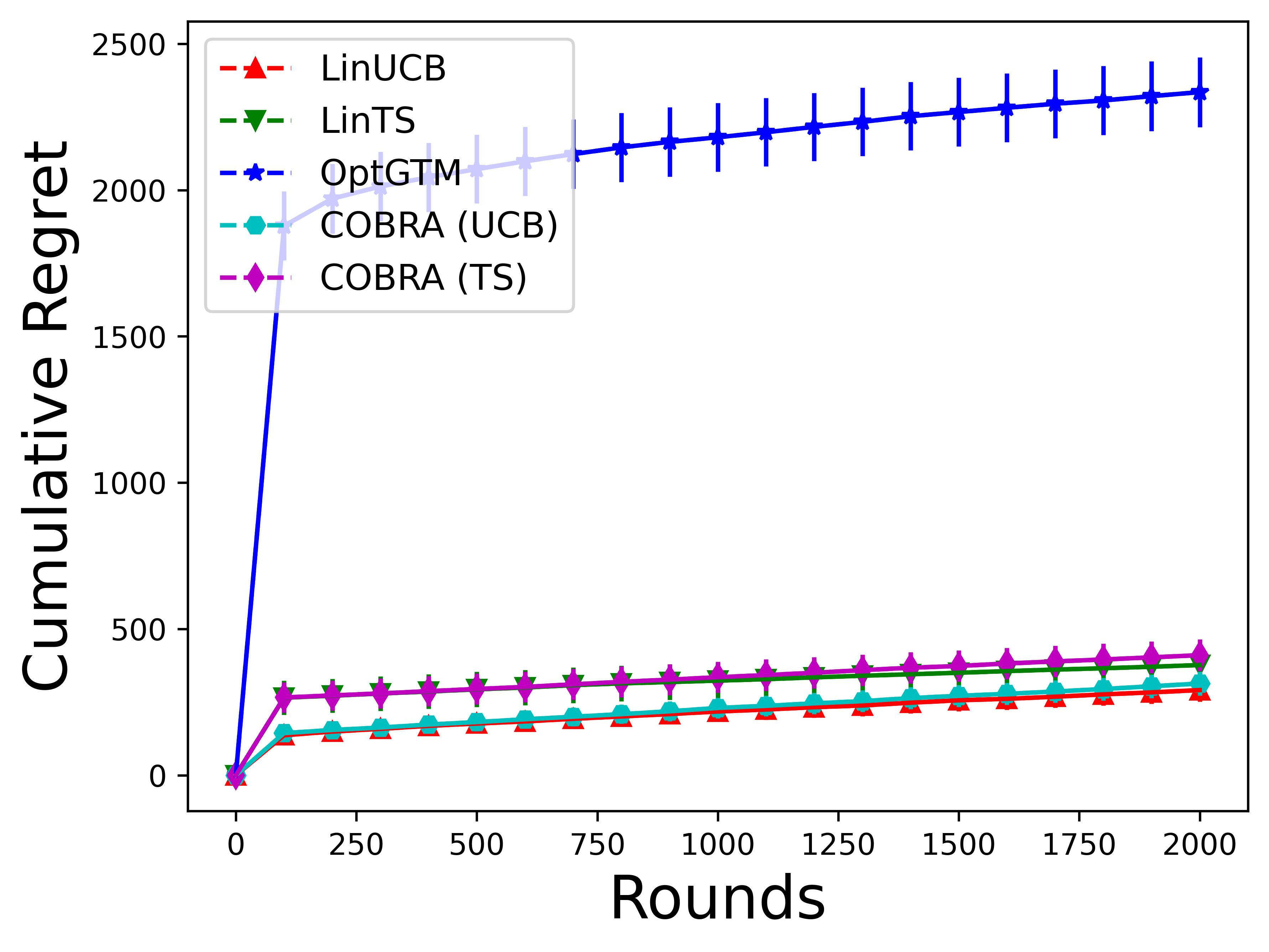}}
    \subfloat[Linear (Agent type I)]{\label{fig:nonlinear_prob1}
		\includegraphics[width=0.245\linewidth]{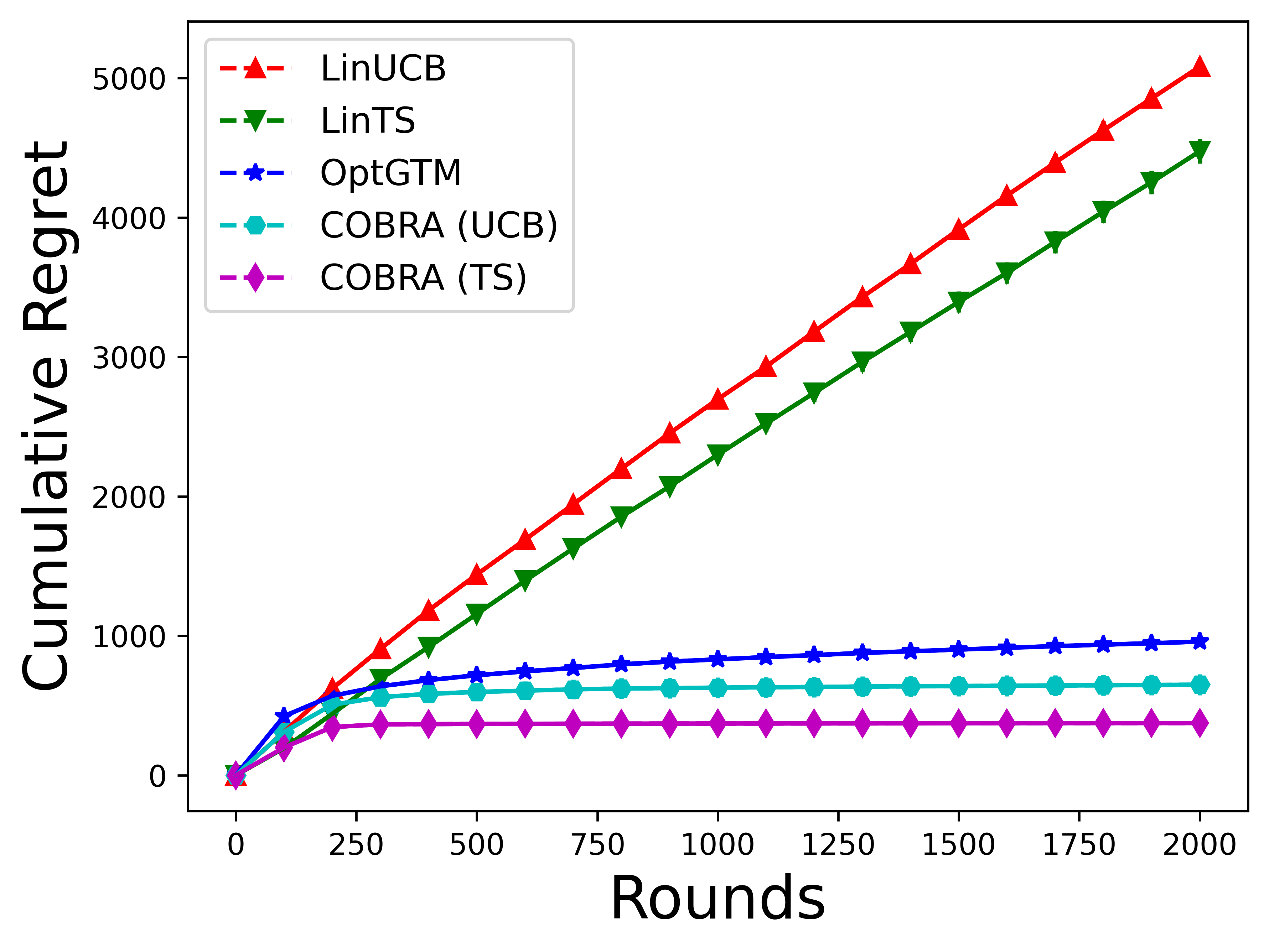}}
	\subfloat[Linear (Agent type II)]{\label{fig:nonlinear_prob2}
		\includegraphics[width=0.245\linewidth]{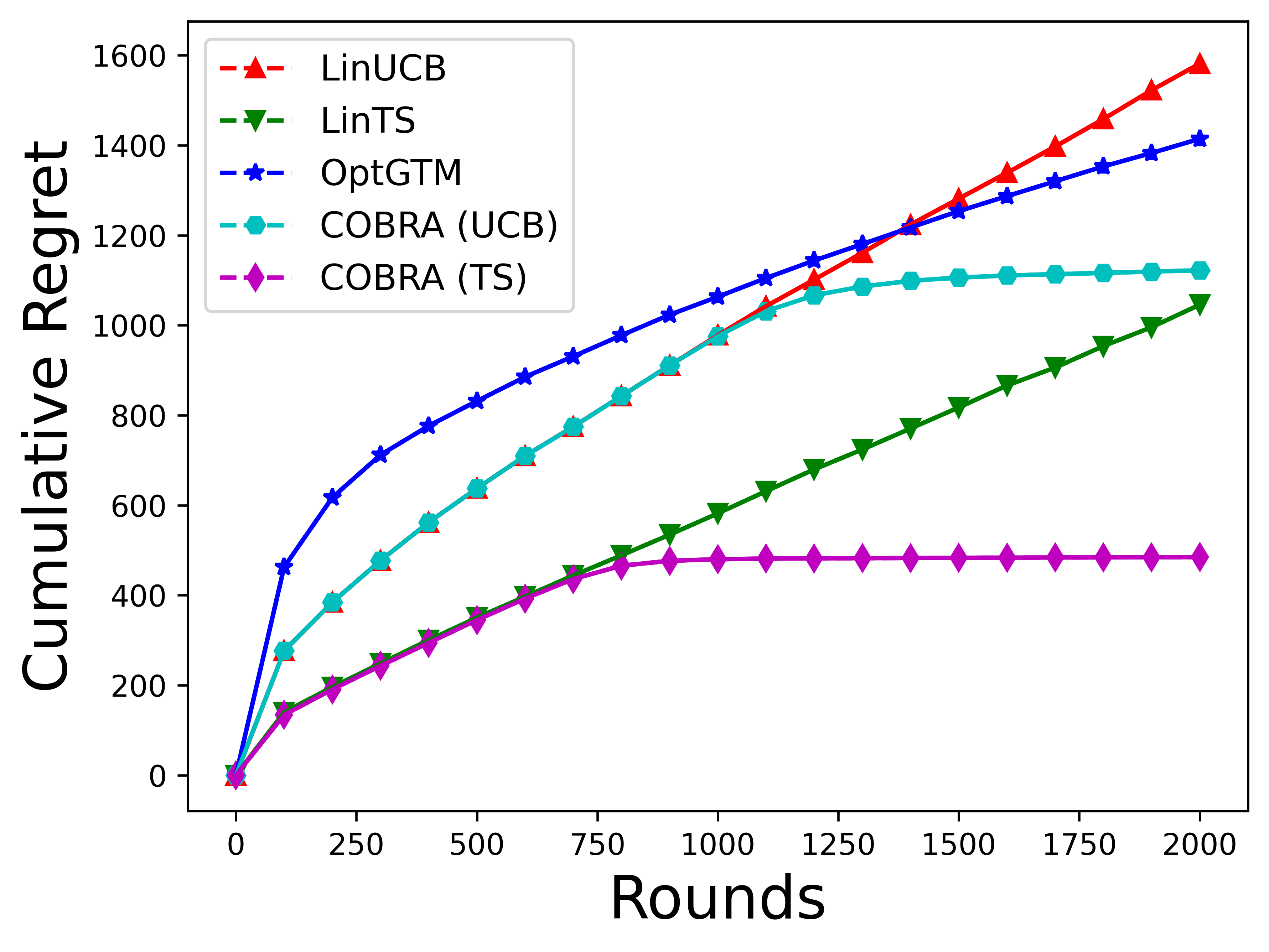}}
	\subfloat[Square (Agent type II)]{\label{fig:nonlinear_prob3}
		\includegraphics[width=0.245\linewidth]{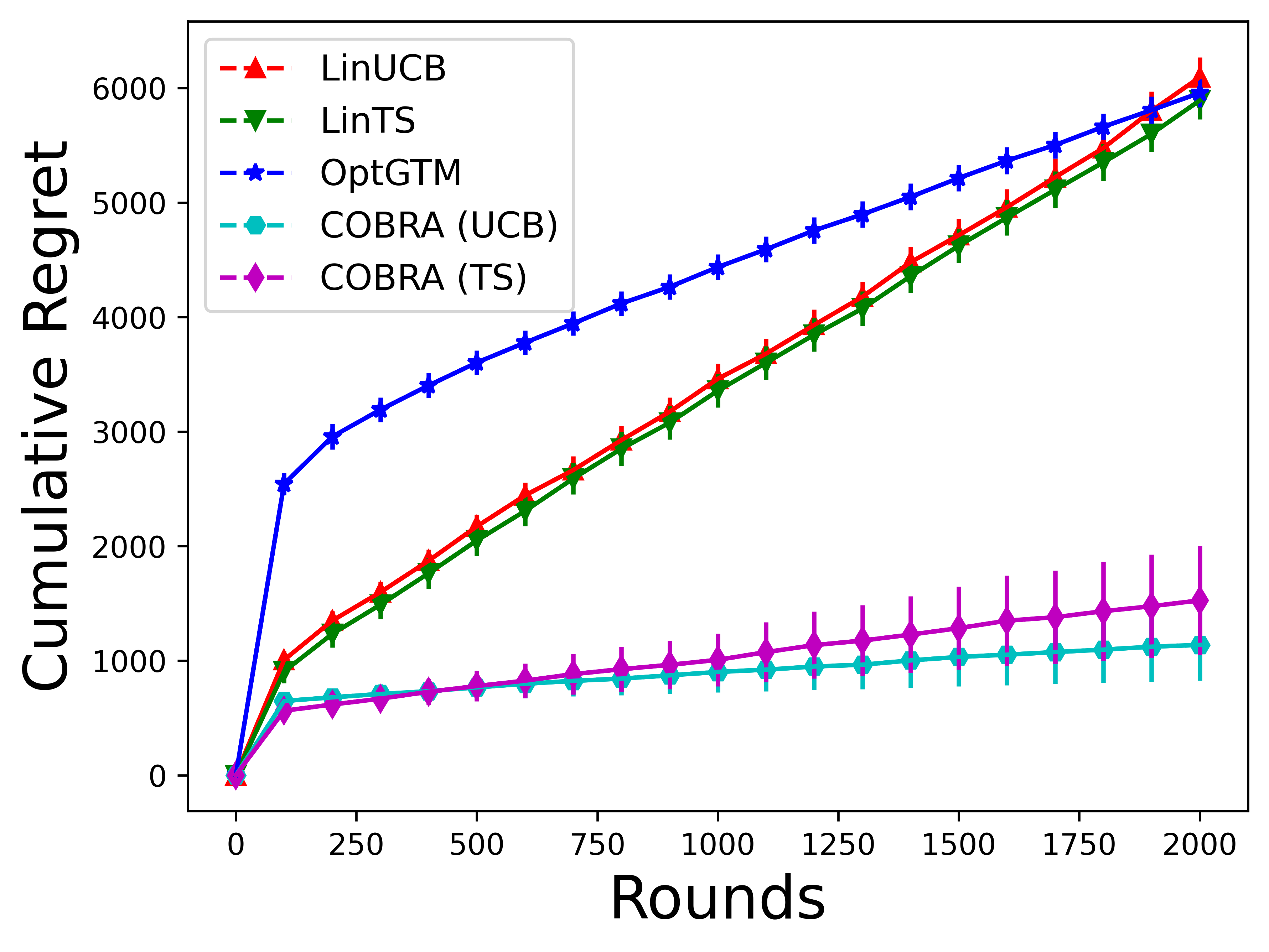}}
	\caption{
        Comparing the cumulative regret of \algo{} with different baselines for problem instances with non-linear reward functions.
	}
	\label{fig:compare_nonlinear_cumm_regret}
\end{figure}

\para{Regret of \algo{} vs. number of agents ($N$) and dimension $(d)$.}
The number of agents $(N)$ and the dimension of the context-arm feature vector $(2d)$, where in this section $d \doteq d_c = d_n$, in the contextual bandit problem control the difficulty. As their values increase, the problem becomes more difficult, making it harder to allocate the context to the best agent.
We want to verify this by observing how the regret of our proposed algorithms changes while varying  $N$ and $d$ in the contextual bandit problem. 
To see this in our experiments, we use the linear reward function (i.e., $f(x) = 5x^\top\theta_\star$), $2000$ contexts, $N=10$ when varying dimension,  $d=20$ while varying the number of agents.
As shown in \cref{fig:vary_agents_ucb} and \cref{fig:vary_agents_ts}, the regret of our \algo{} UCB- and TS- based algorithms increases as we increase the number of agents, i.e., $N = \{10,  20, 30, 40, 50\}$. 
We also observe the same trend when we increase the dimension $d$ of the context-arm feature vectors from $d =\{5, 10, 15, 20, 25\}$ as shown in \cref{fig:vary_dims_ucb} and \cref{fig:vary_dims_ts}. 
We also observe that the \algo{} TS-based algorithm performs better than its \algo{} UCB-based counterpart (as seen in \cref{fig:vary_agents_ucb}-\ref{fig:vary_dims_ts} by comparing the regret of both algorithms). 

\begin{figure}[!ht]
    \vspace{-5mm}
	\centering
	\subfloat[Vary agents/arms (UCB)]{\label{fig:vary_agents_ucb}
		\includegraphics[width=0.24\linewidth]{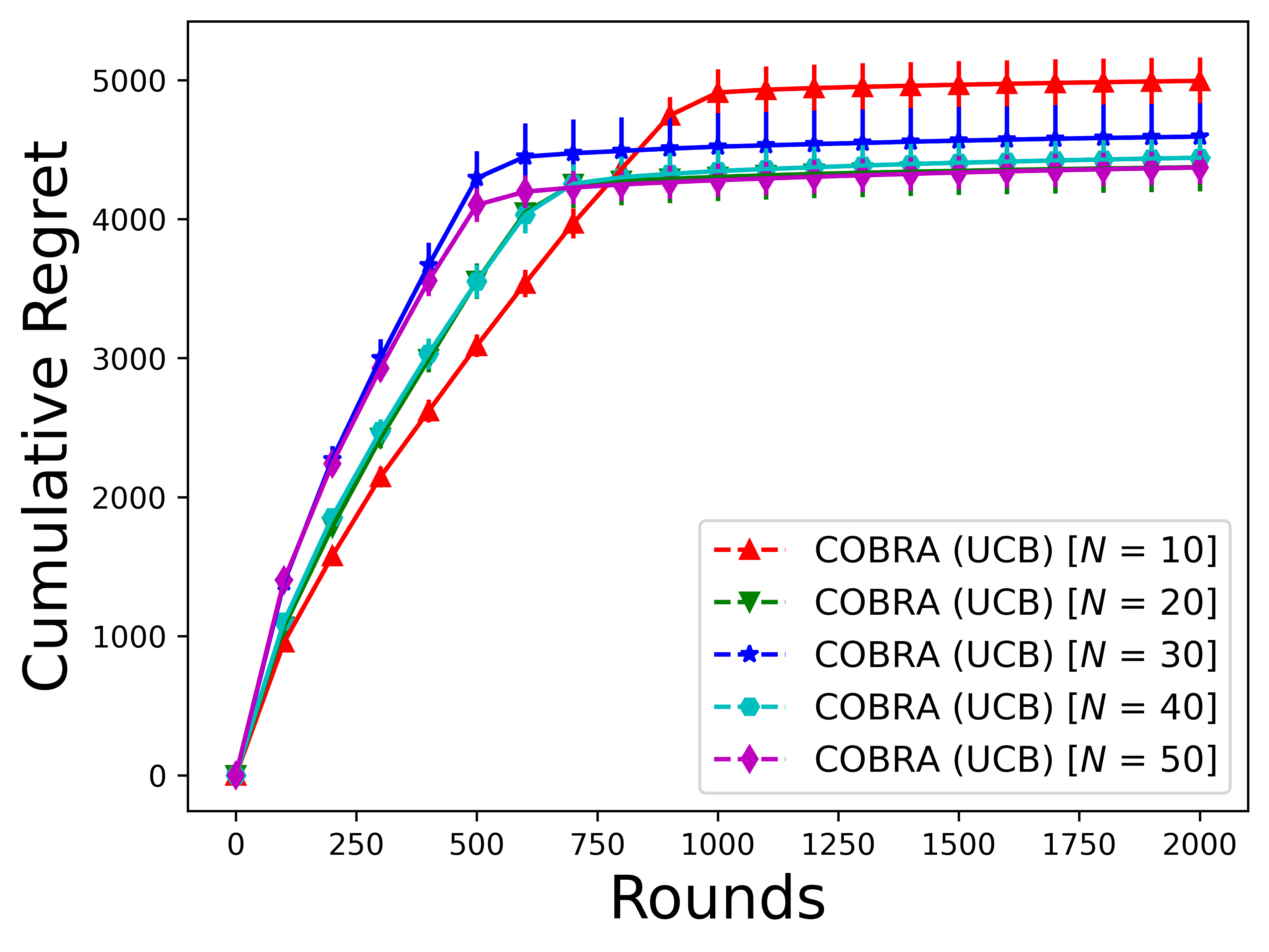}}
	\subfloat[Vary agents/arms (TS)]{\label{fig:vary_agents_ts}
		\includegraphics[width=0.24\linewidth]{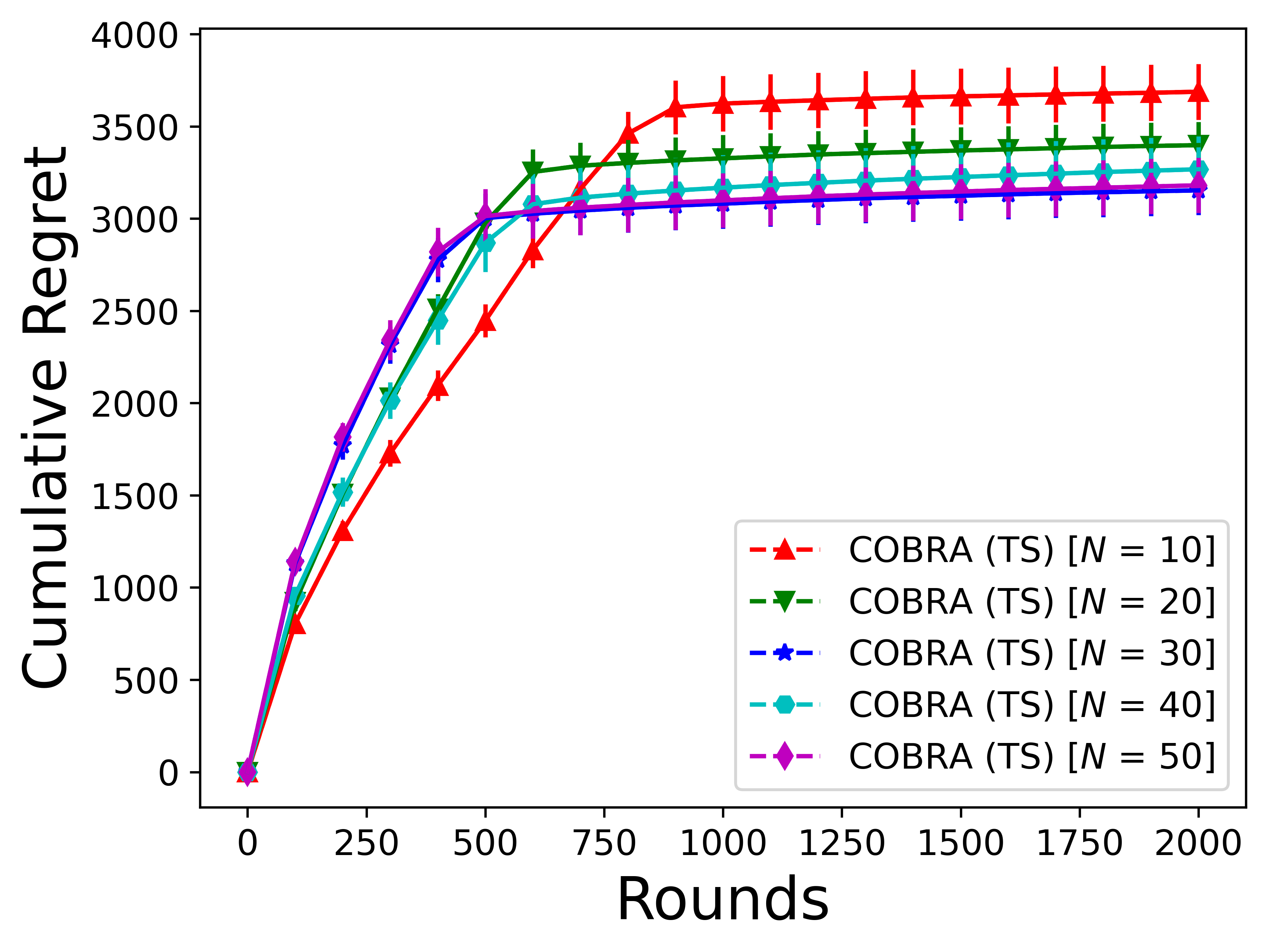}}
	\subfloat[Vary dimension (UCB)]{\label{fig:vary_dims_ucb}
		\includegraphics[width=0.24\linewidth]{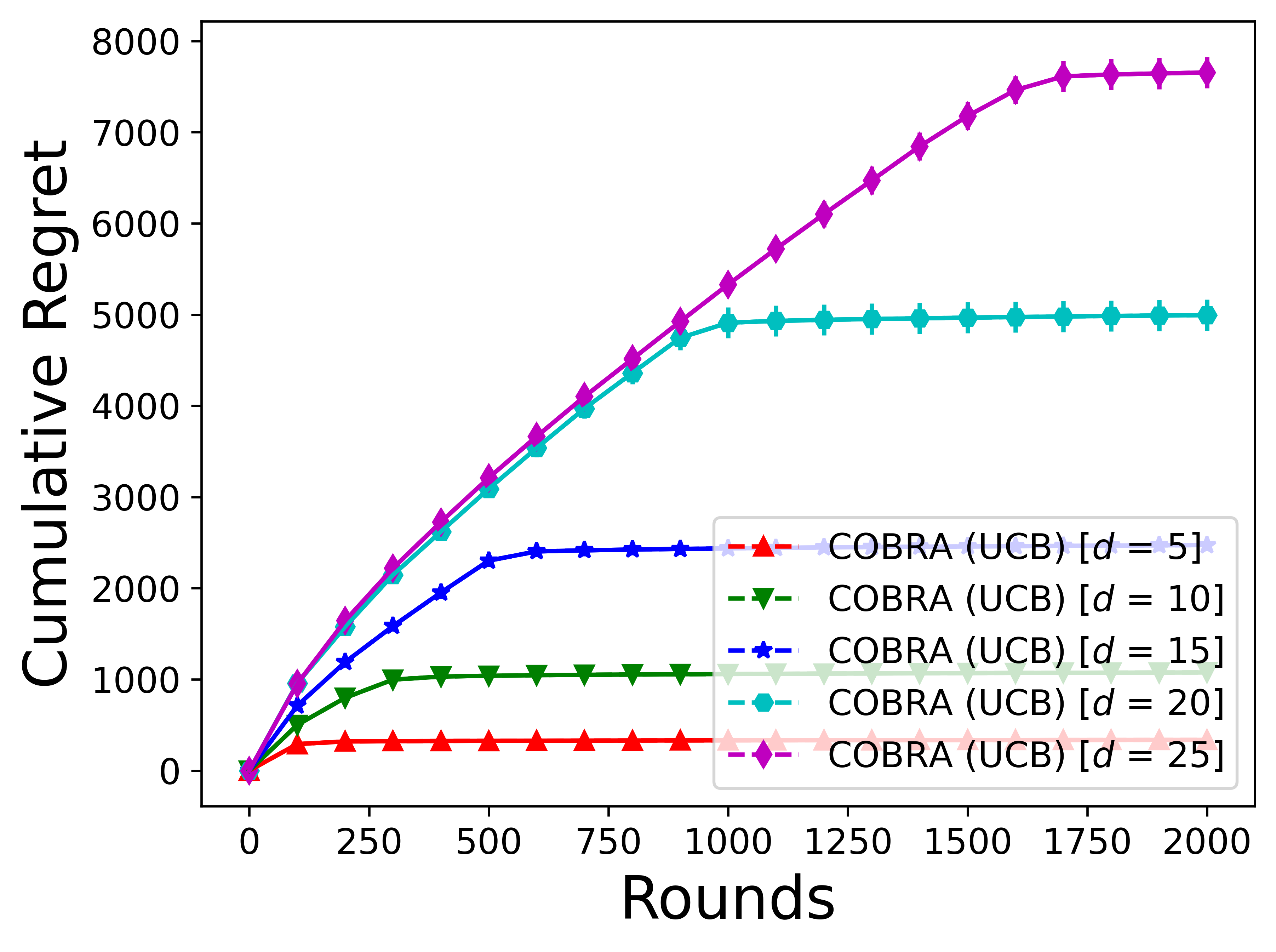}}
	\subfloat[Vary dimension (TS)]{\label{fig:vary_dims_ts}
		\includegraphics[width=0.24\linewidth]{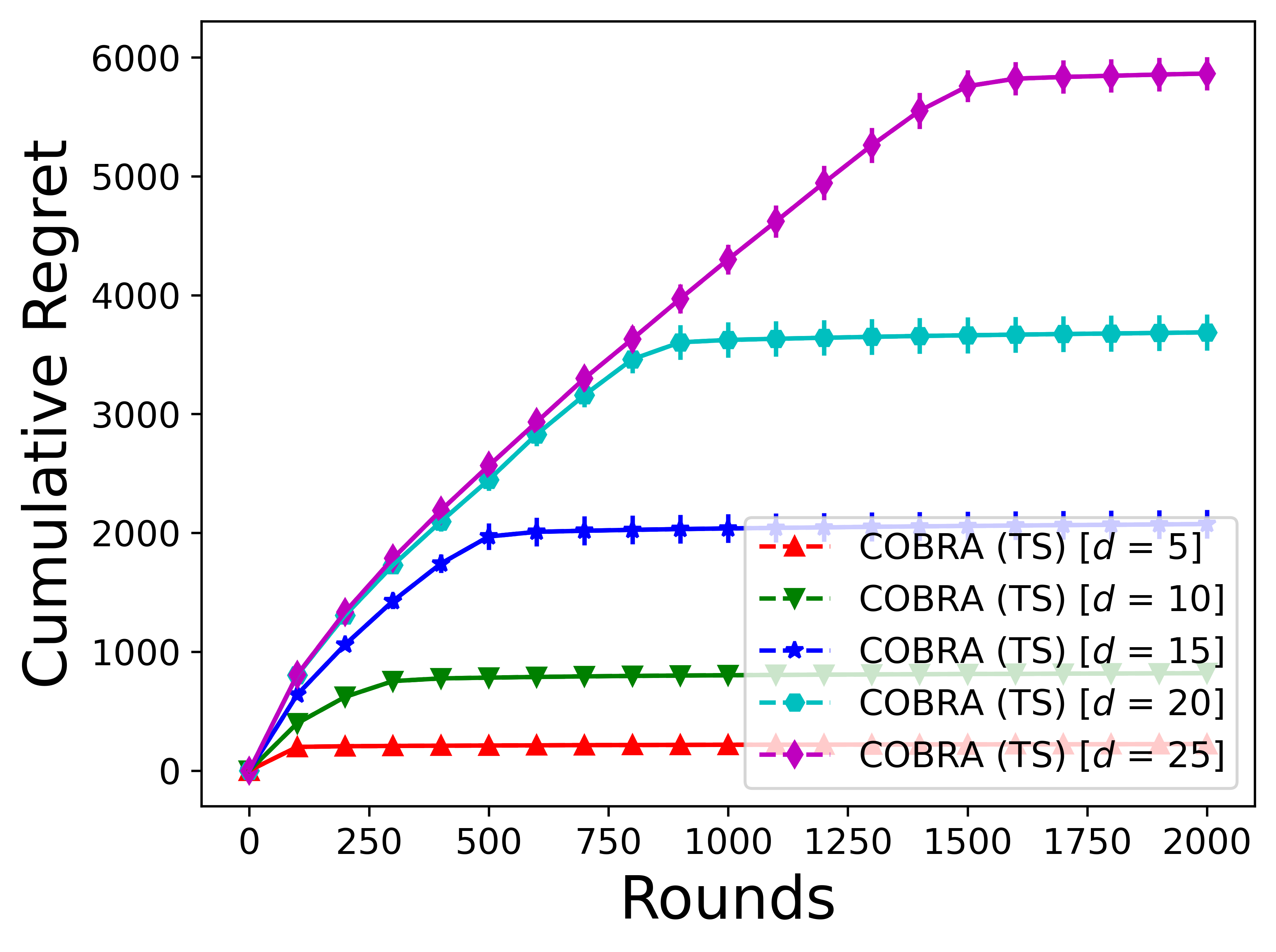}}
    \caption{
       Cumulative regret of \algo{} vs. different values of $N$ and $d$. 
	}
	\label{fig:compare_NandD}
\end{figure}

\para{Experiments with two strategic agents.}
We also conducted additional experiments involving two type II strategic agents. In these experiments, we control the degree of over-reporting for one of the two strategic agents as $\hat{\Delta}_x \sim \bN(\text{Scale}\cdot\Delta_x^\star, \sigma_\Delta^2)$. 
The results (\cref{fig:vary_agents_2strategic_ucb}–\ref{fig:fix_agents_2strategic_ts}) show that as the controlled over-reporting of the agent increases, the regret of COBRA also increases. The higher regret is because the strategic behavior becomes harder to detect, and the over-reporting biases the estimator used in arm selection. For the same settings, we also plot the maximum value of $C$ (the maximum advantage of any agent, \cref{def:advantage}) under different levels of over-reporting for the two strategic agents. As expected, $C$ grows very slowly with the number of rounds and remains almost constant, as shown in \cref{fig:vary_agents_C2strategic_ucb}–\ref{fig:fix_agents_C2strategic_ts}.
This is the constant carried explicitly in \cref{thm:regretAllNE} and \cref{thm:regretNonAllNE}. These measurements provide empirical evidence that $C$ remains small in the homogeneous setting, but do not constitute a theoretical bound; these instances have favorable peer coverage, since all agents' features are drawn from a common distribution, whereas $C$ can grow as $\Theta(\sqrt{T})$ without coverage (Example~2 in \cref{asec:loom}; see also \cref{prop:coverage}).

\begin{figure}[!ht]
    \vspace{-5mm}
	\centering
	\subfloat[Vary agents/arms (UCB)]{\label{fig:vary_agents_2strategic_ucb}
		\includegraphics[width=0.24\linewidth]{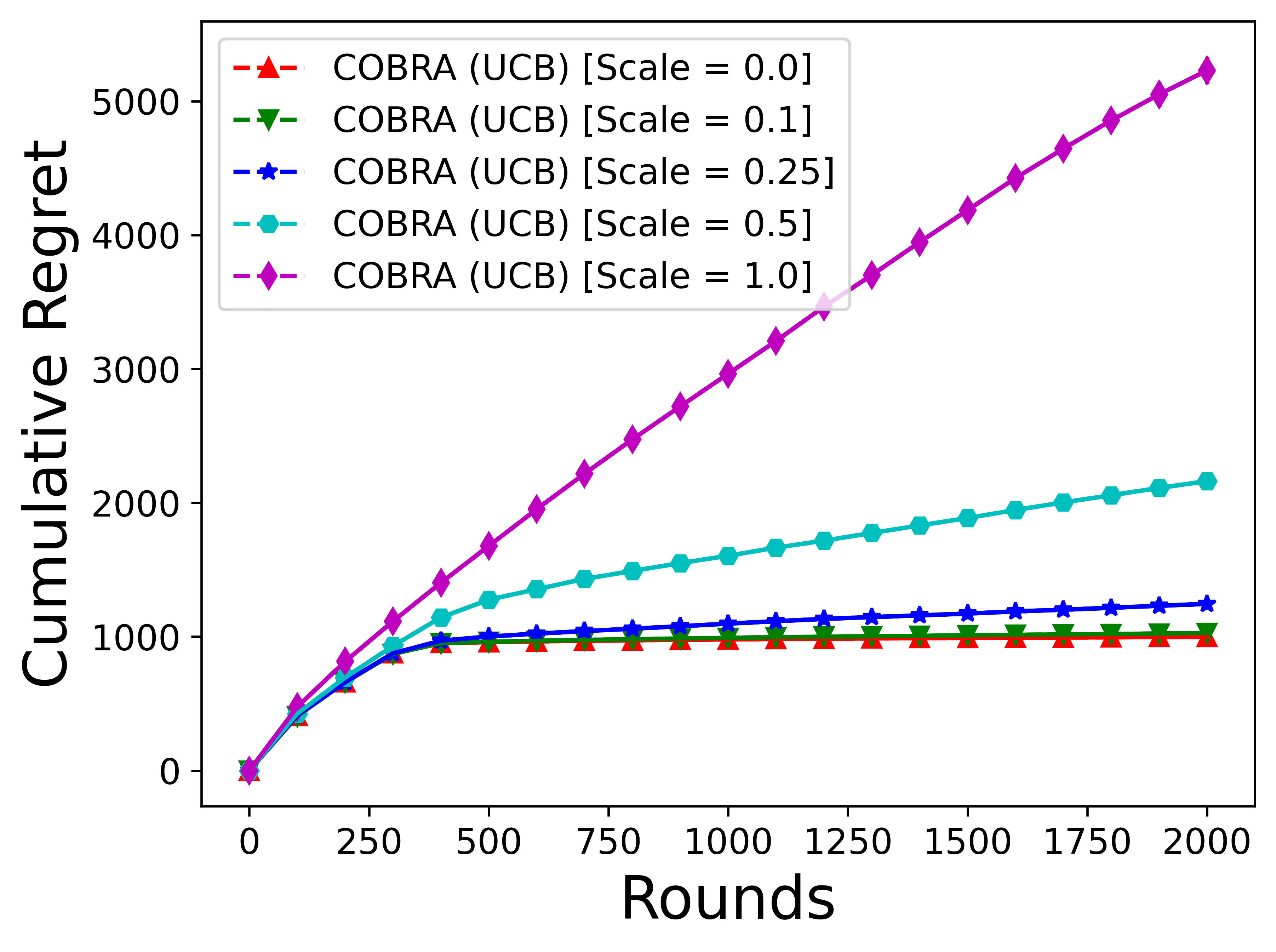}}
	\subfloat[Vary agents/arms (TS)]{\label{fig:vary_agents_2strategic_ts}
		\includegraphics[width=0.24\linewidth]{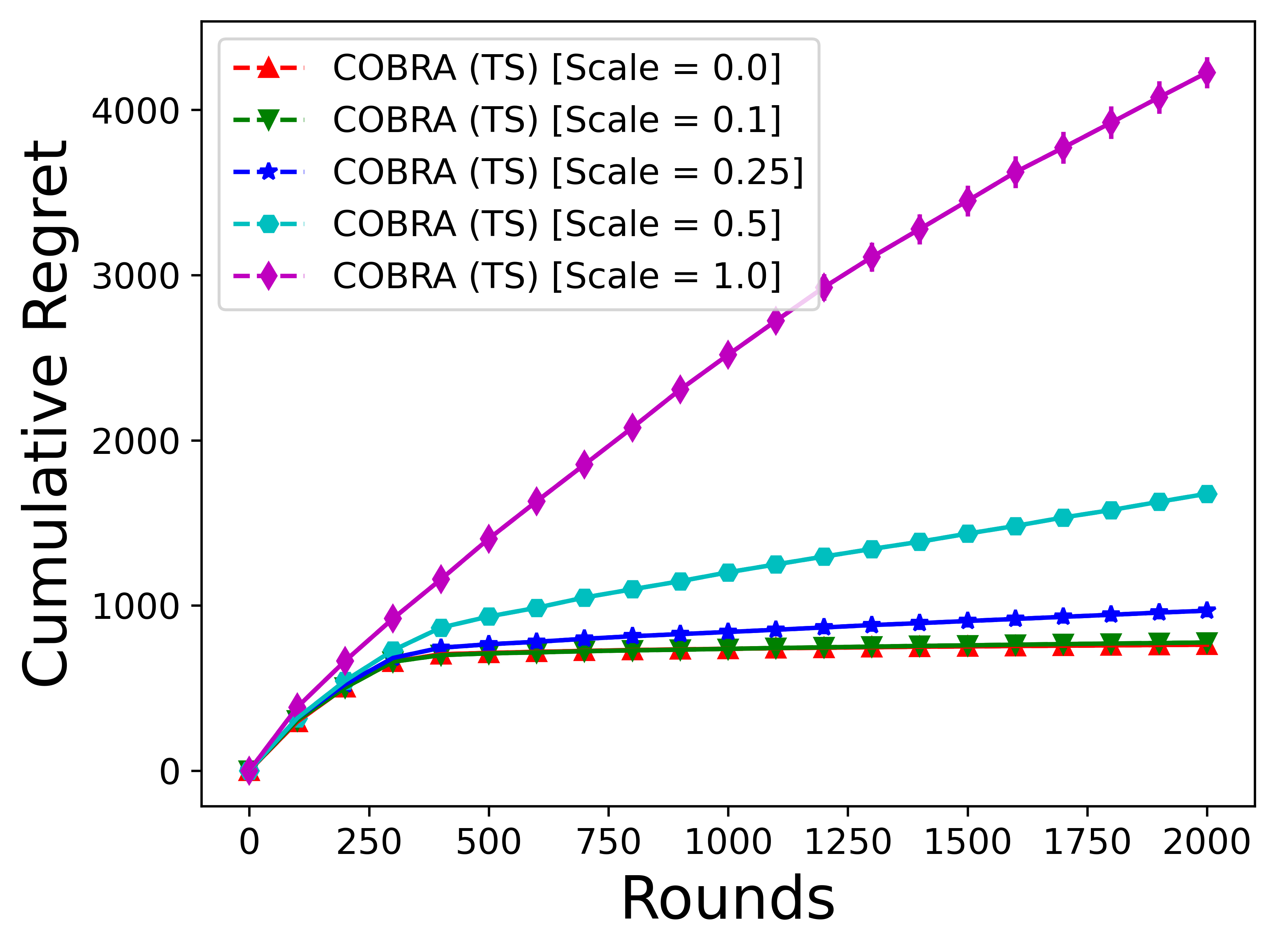}}
	\subfloat[Fix agents/arms (UCB)]{\label{fig:fix_agents_2strategic_ucb}
		\includegraphics[width=0.24\linewidth]{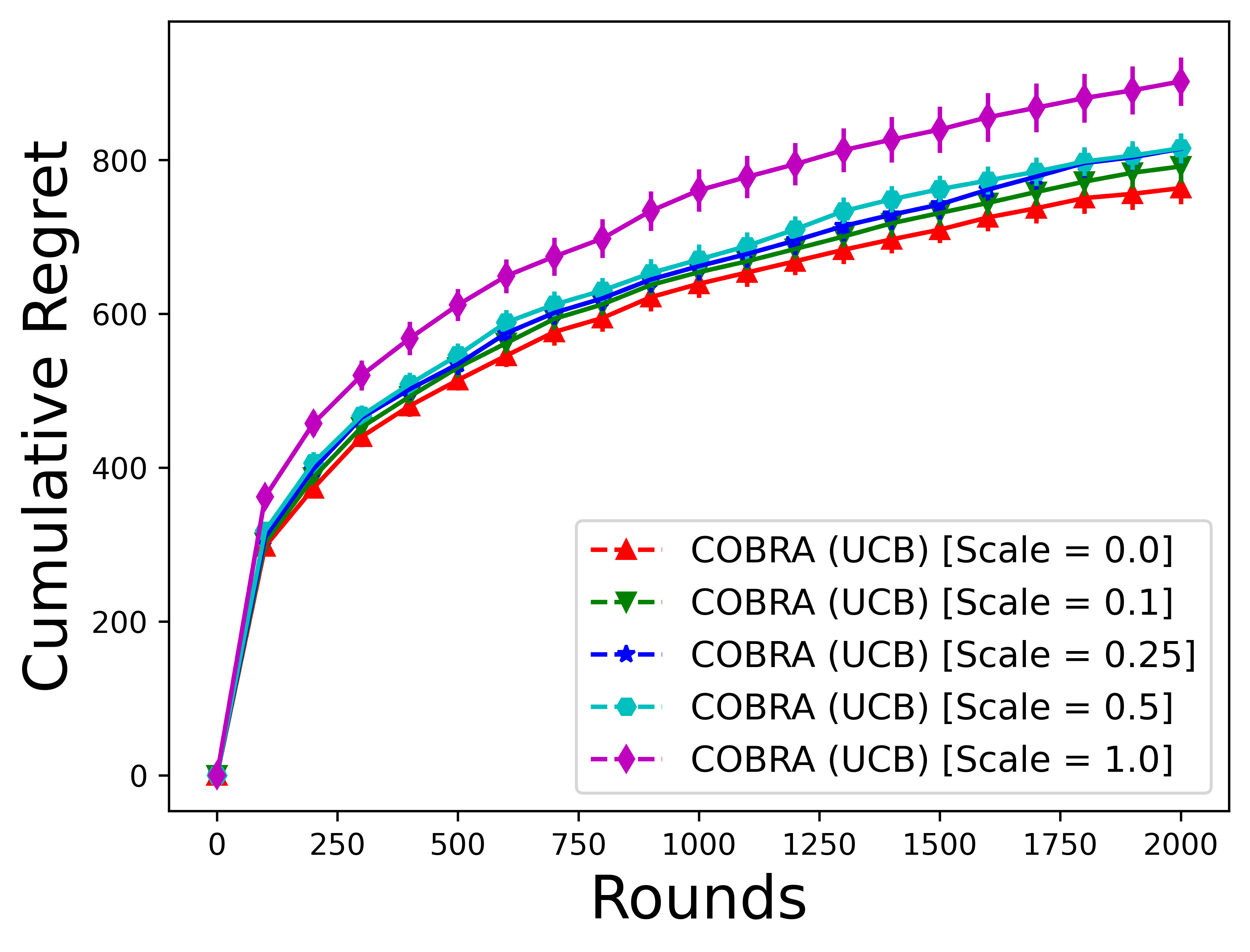}}
	\subfloat[Fix agents/arms (TS)]{\label{fig:fix_agents_2strategic_ts}
		\includegraphics[width=0.24\linewidth]{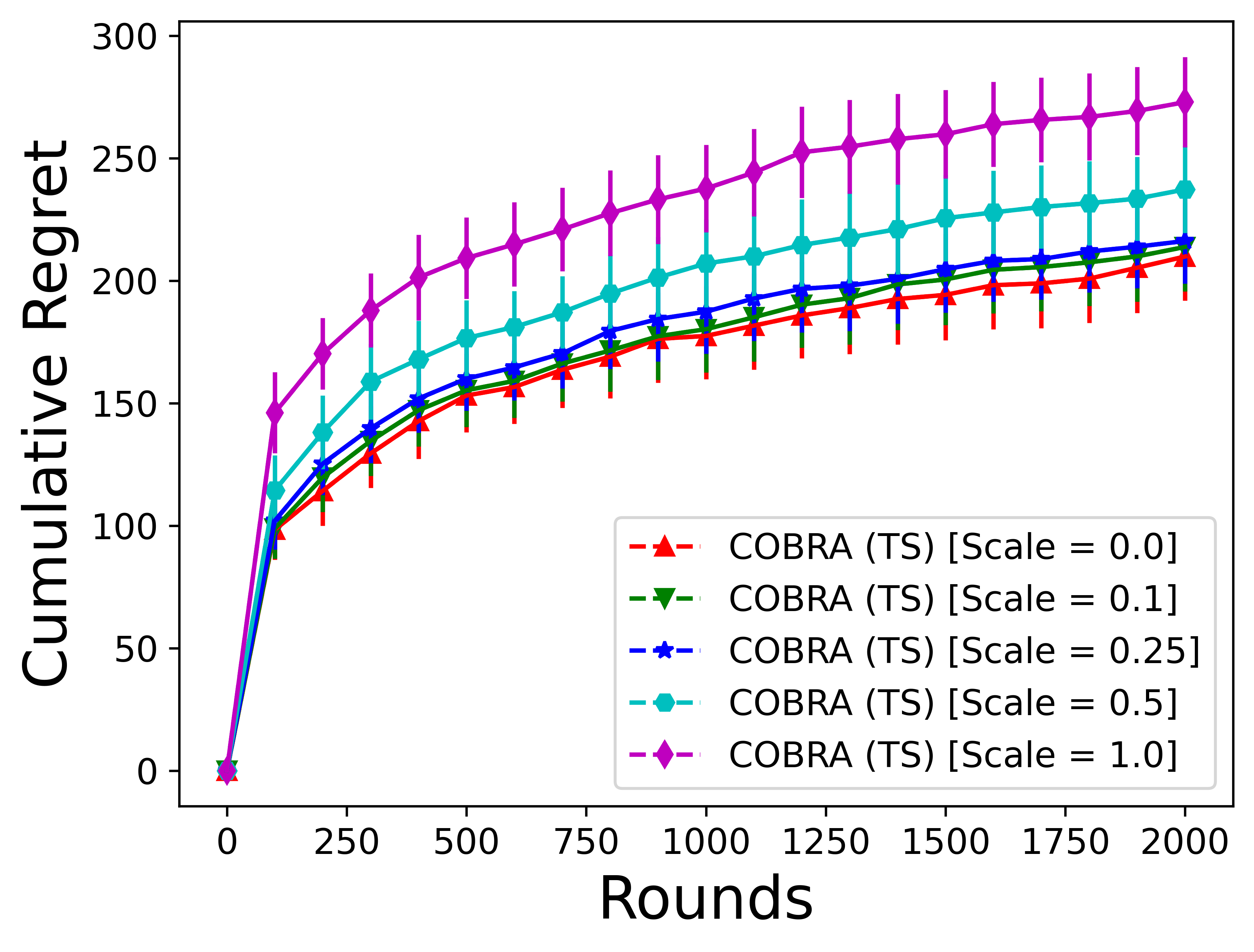}}\\
    \vspace{-1mm}
    \subfloat[Vary agents/arms (UCB)]{\label{fig:vary_agents_C2strategic_ucb}
		\includegraphics[width=0.24\linewidth]{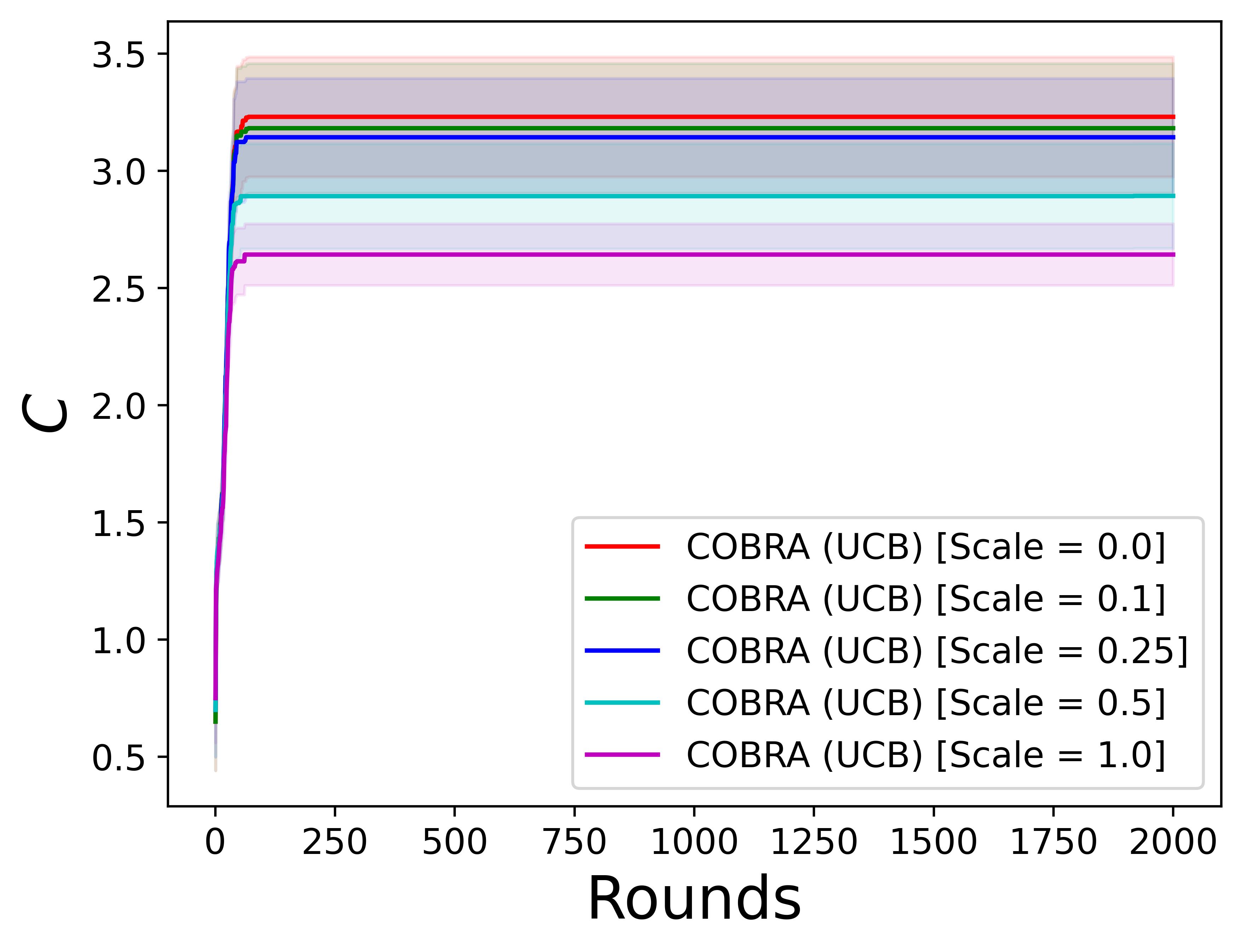}}
	\subfloat[Vary agents/arms (TS)]{\label{fig:vary_agents_C2strategic_ts}
		\includegraphics[width=0.24\linewidth]{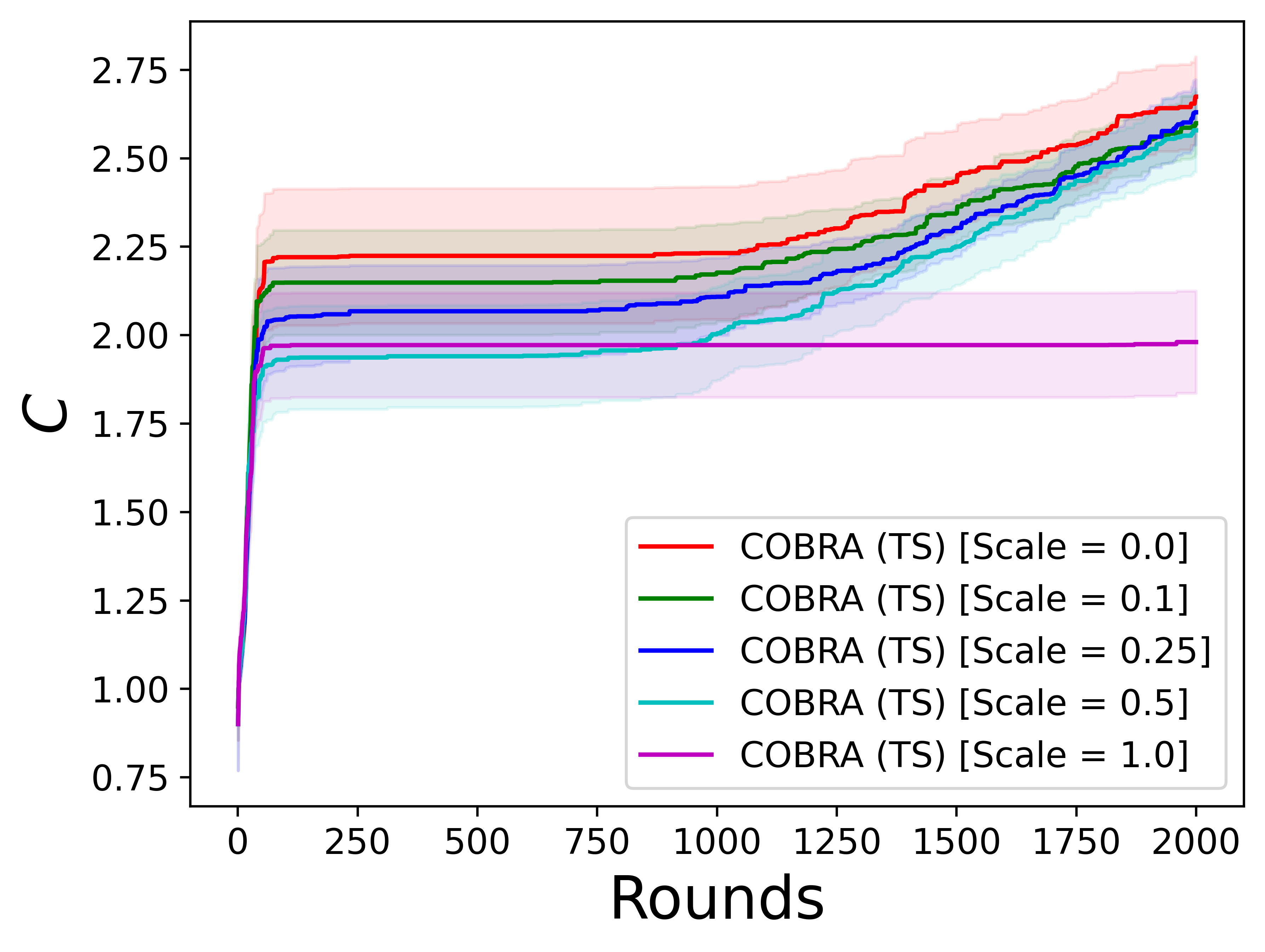}}
	\subfloat[Fix agents/arms (UCB)]{\label{fig:fix_agents_C2strategic_ucb}
		\includegraphics[width=0.24\linewidth]{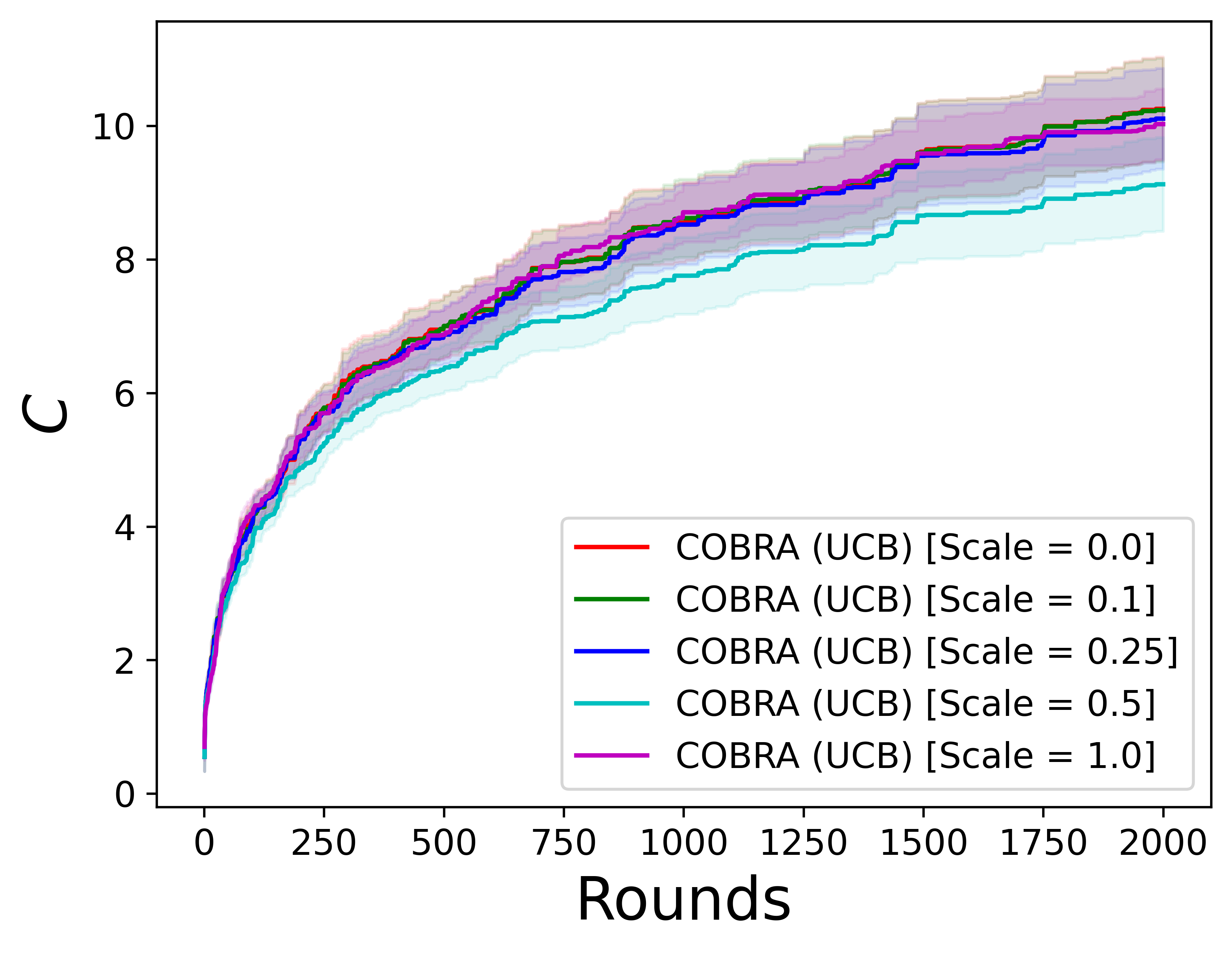}}
	\subfloat[Fix agents/arms (TS)]{\label{fig:fix_agents_C2strategic_ts}
		\includegraphics[width=0.24\linewidth]{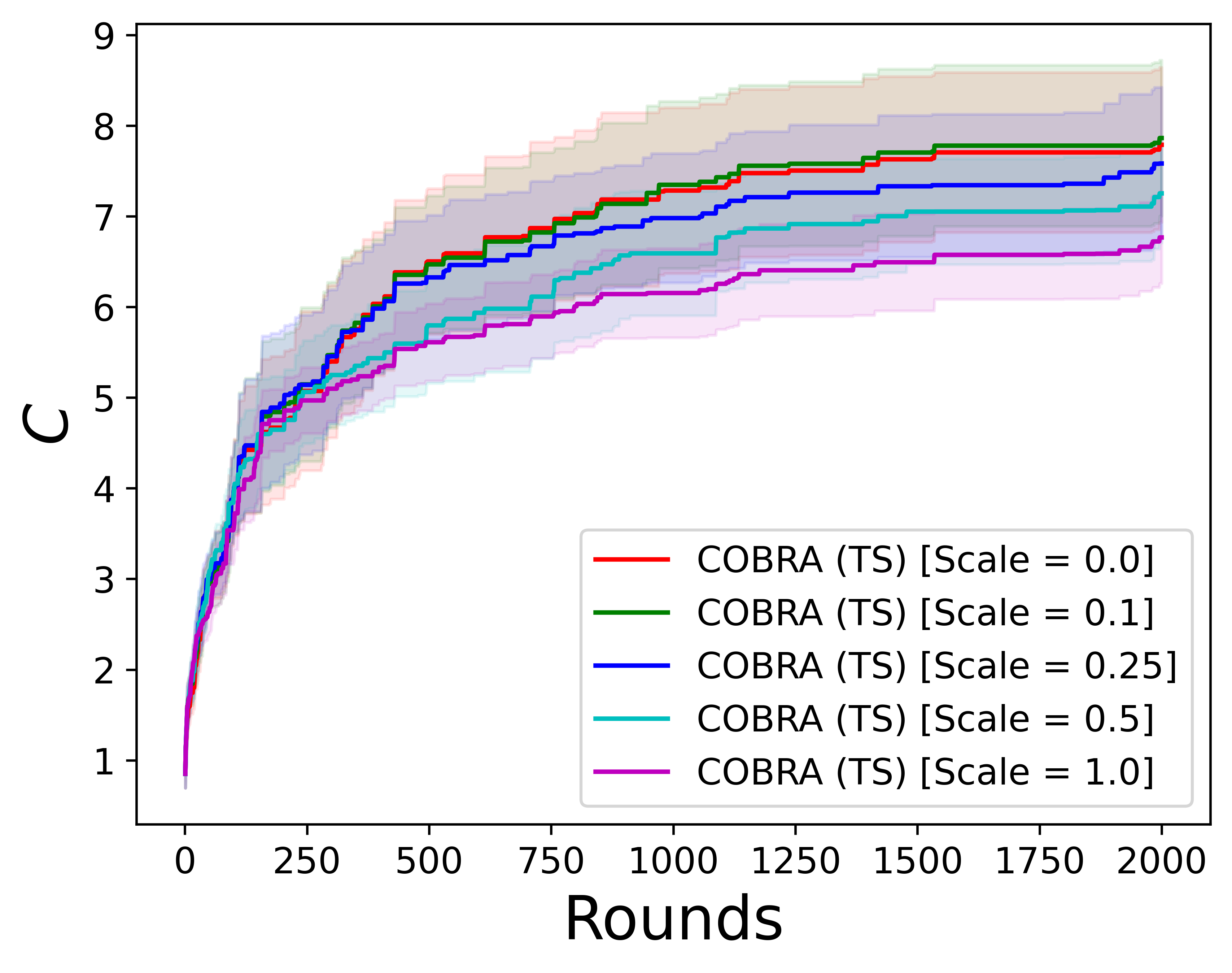}}
    \caption{
       \textbf{Top row:} Cumulative regret of \algo{} vs. different level of two over-reporting agents. \textbf{Bottom row:} Maximum value of $C$ vs. different level of two over-reporting agents.
	}
	\label{fig:compare_overReport}
    \vspace{-2mm}
\end{figure}

\para{Regret of \algo{} vs. large number of agents ($N$) and dimension $(d)$.}
We conducted additional experiments to evaluate the performance of our LOOM-based mechanism in higher-dimensional settings and with larger numbers of strategic agents. Specifically, we considered the following scales: number of agents $N \in \{100, 200, 300, 400, 500\}$ and feature dimension $d \in \{60, 70, 80, 90, 100\}$. Note that the overall dimension of the context-arm feature is $2d$.
When varying the number of agents, we fixed the feature dimension to $50$ and $100$. As expected, the regret of COBRA remains sub-linear as the number of agents increases (\cref{fig:vary_large_agents_50_ucb}-\ref{fig:vary_large_agents_100_ucb}). Similarly, when varying the feature dimension, we set the number of agents to $50$ and $100$. As anticipated, the regret of COBRA increases with the dimension, since higher-dimensional problems require substantially more samples to estimate the reward function accurately (\cref{fig:vary_large_dims_50_ucb}-\ref{fig:vary_large_dims_100_ucb}). These results demonstrate the robustness of our theoretical guarantees, even in complex, large-scale environments. For the same settings, we also plot the maximum value of $C$ under different $N$ and $d$. As expected, $C$ grows very slowly with the number of rounds and remains almost constant, as shown in \cref{fig:vary_large_agentsC_50_ucb}–\ref{fig:vary_large_dimsC_100_ucb}.
\begin{figure}[!ht]
    \vspace{-4mm}
	\centering
	\subfloat[Vary agents/arms ($50$)]{\label{fig:vary_large_agents_50_ucb}
		\includegraphics[width=0.24\linewidth]{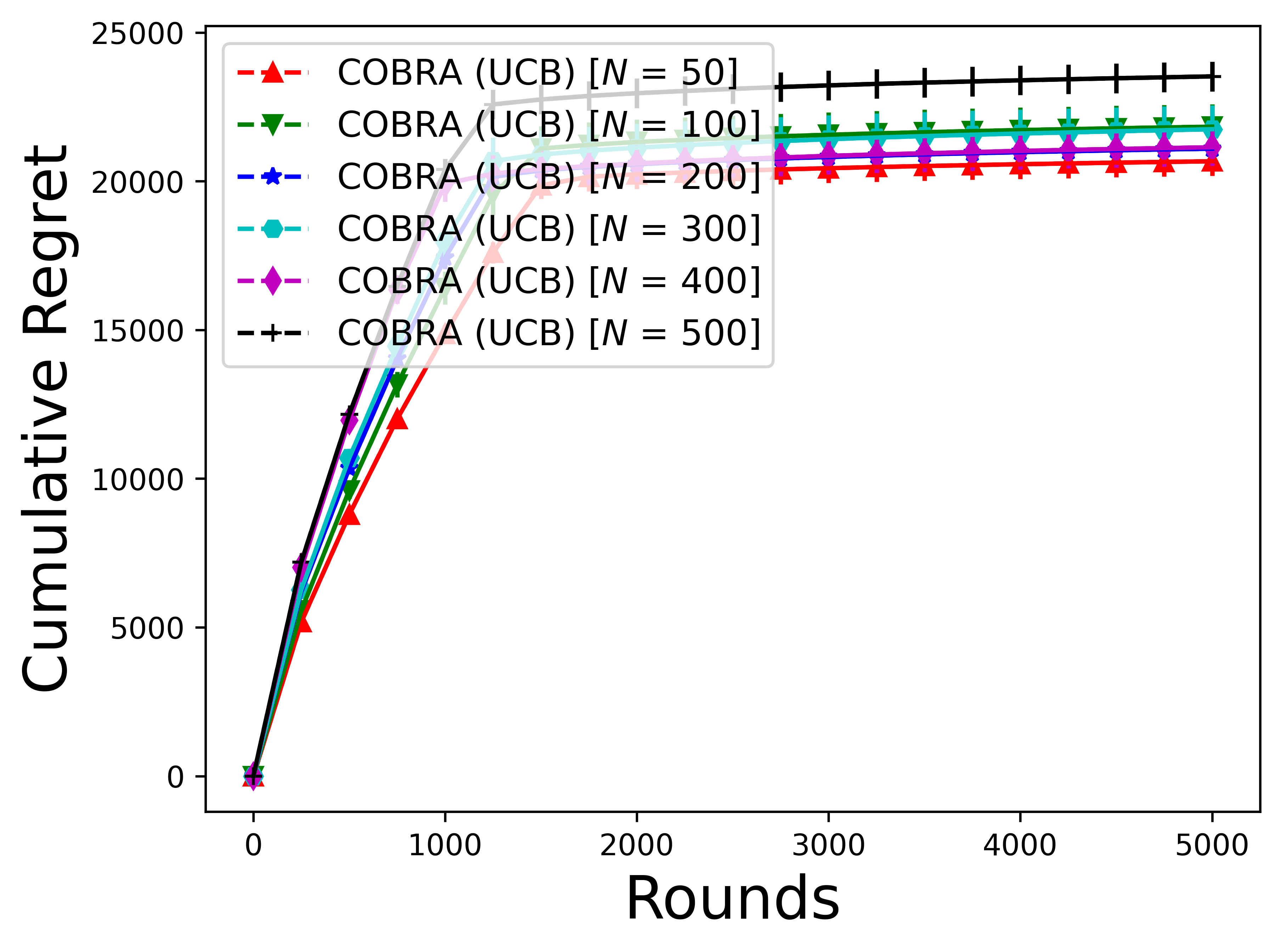}}
	\subfloat[Vary agents/arms ($100$)]{\label{fig:vary_large_agents_100_ucb}
		\includegraphics[width=0.24\linewidth]{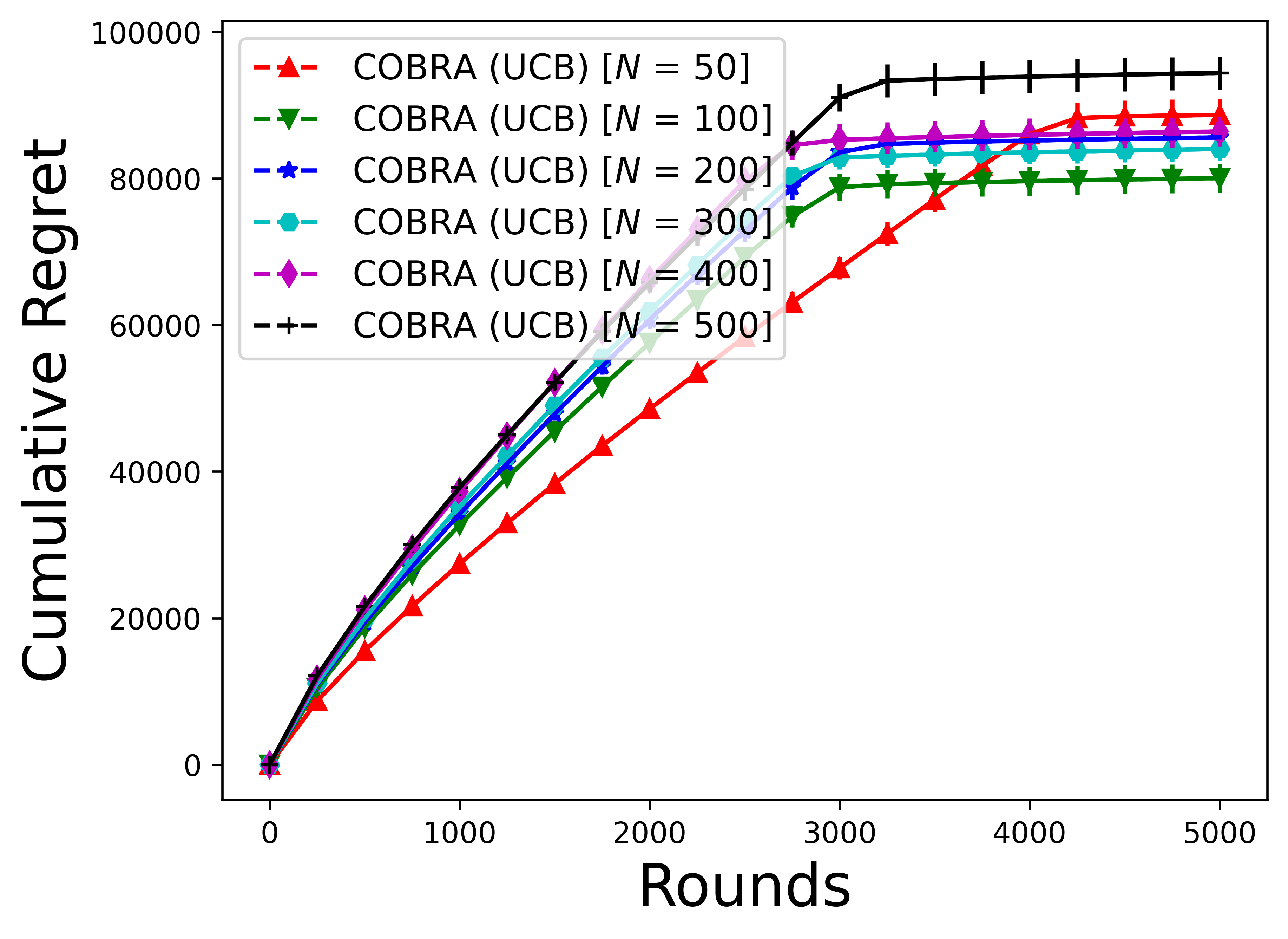}}
	\subfloat[Vary dimension ($50$)]{\label{fig:vary_large_dims_50_ucb}
		\includegraphics[width=0.24\linewidth]{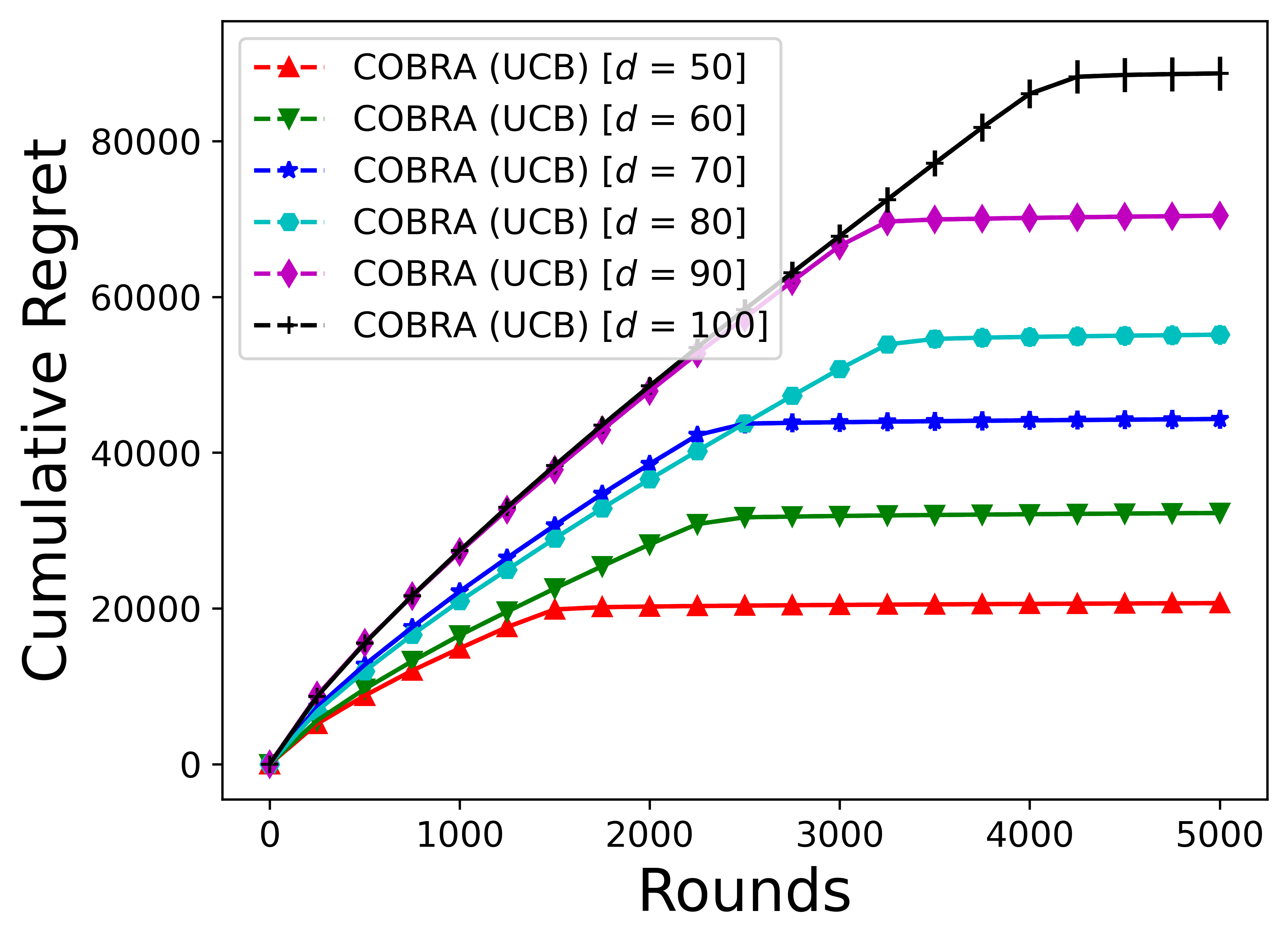}}
	\subfloat[Vary dimension ($100$)]{\label{fig:vary_large_dims_100_ucb}
		\includegraphics[width=0.24\linewidth]{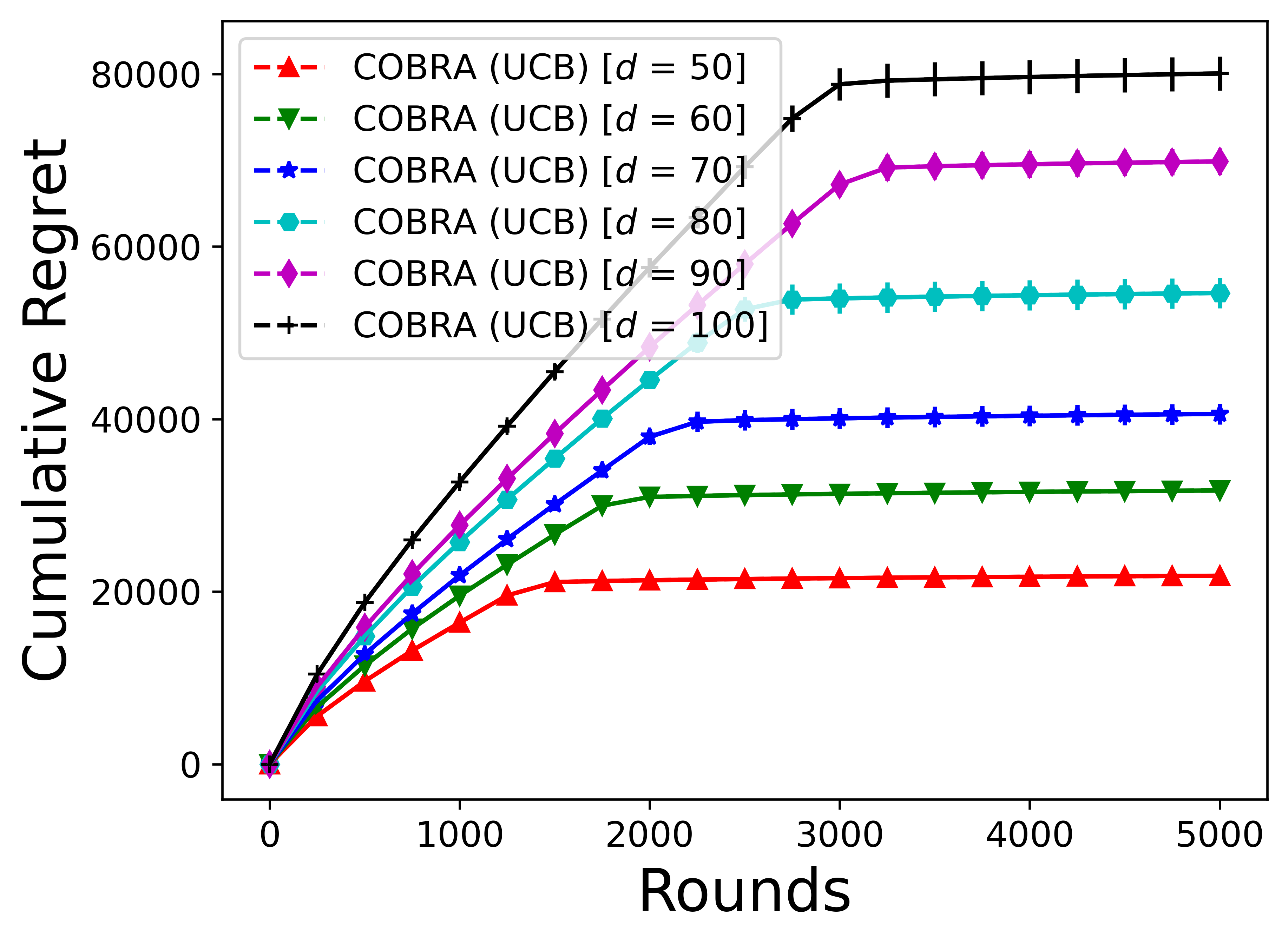}}\\
    \vspace{-1mm}
    \subfloat[Vary agents/arms ($50$)]{\label{fig:vary_large_agentsC_50_ucb}
		\includegraphics[width=0.24\linewidth]{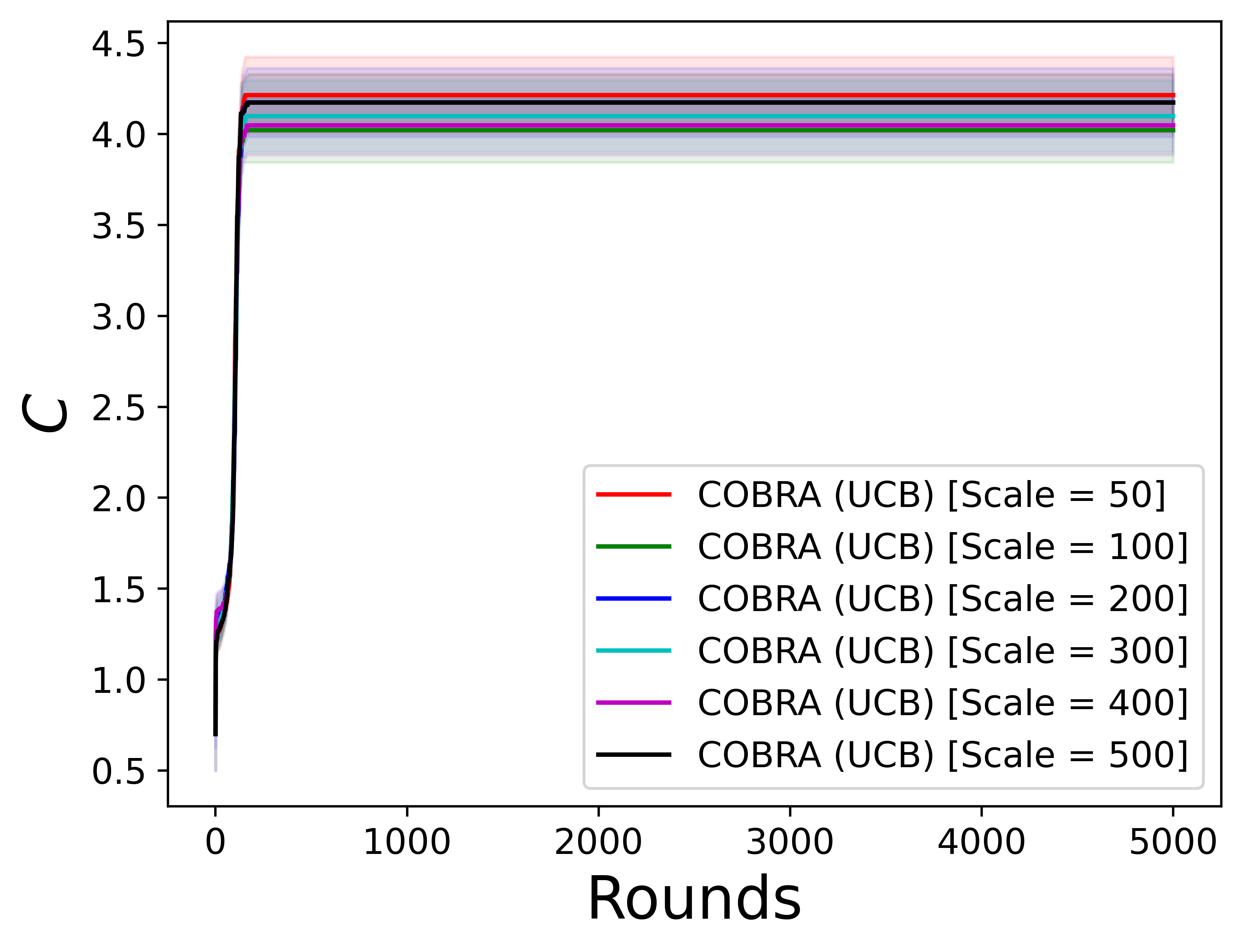}}
	\subfloat[Vary agents/arms ($100$)]{\label{fig:vary_large_agentsC_100_ucb}
		\includegraphics[width=0.24\linewidth]{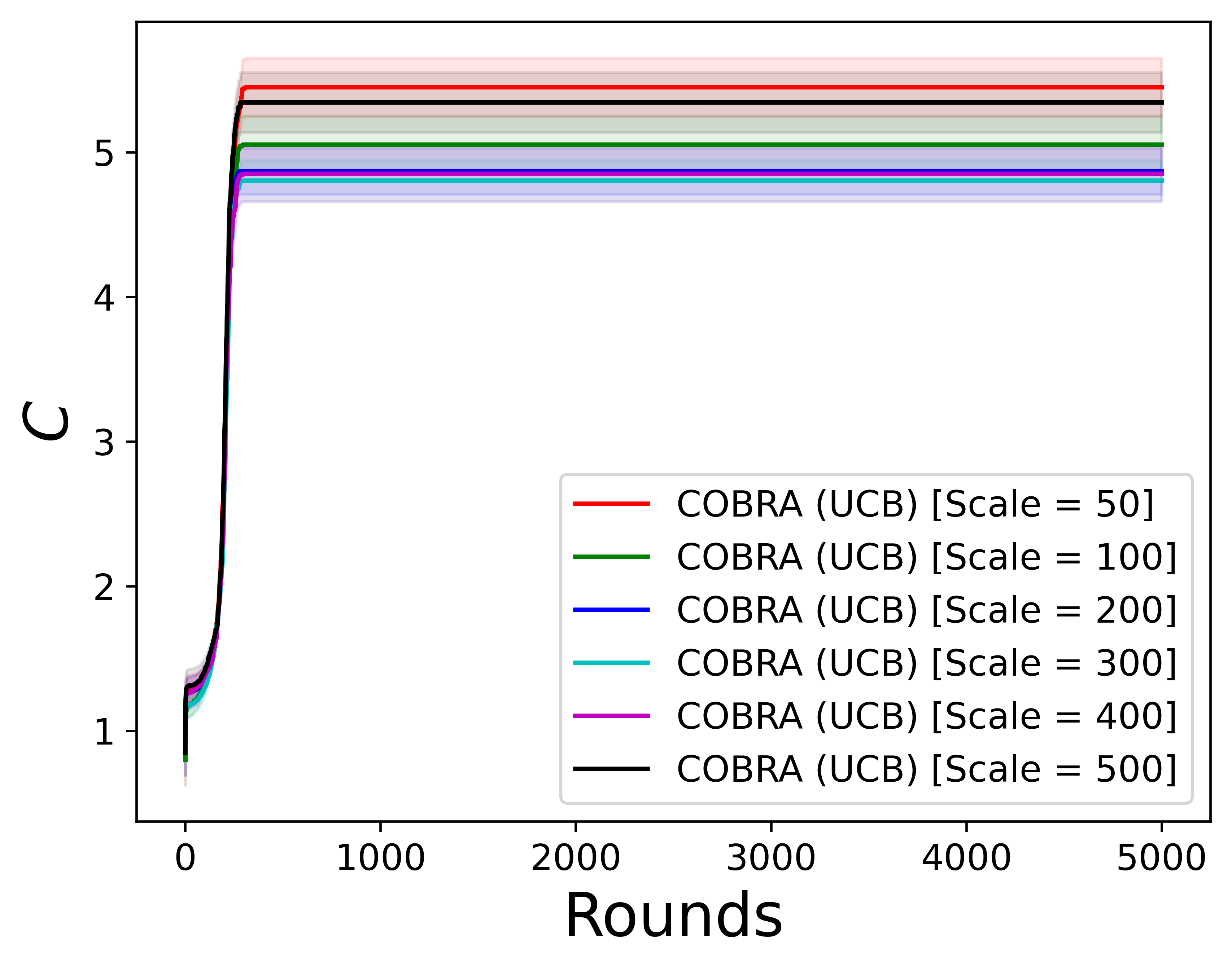}}
	\subfloat[Vary dimension ($50$)]{\label{fig:vary_large_dimsC_50_ucb}
		\includegraphics[width=0.24\linewidth]{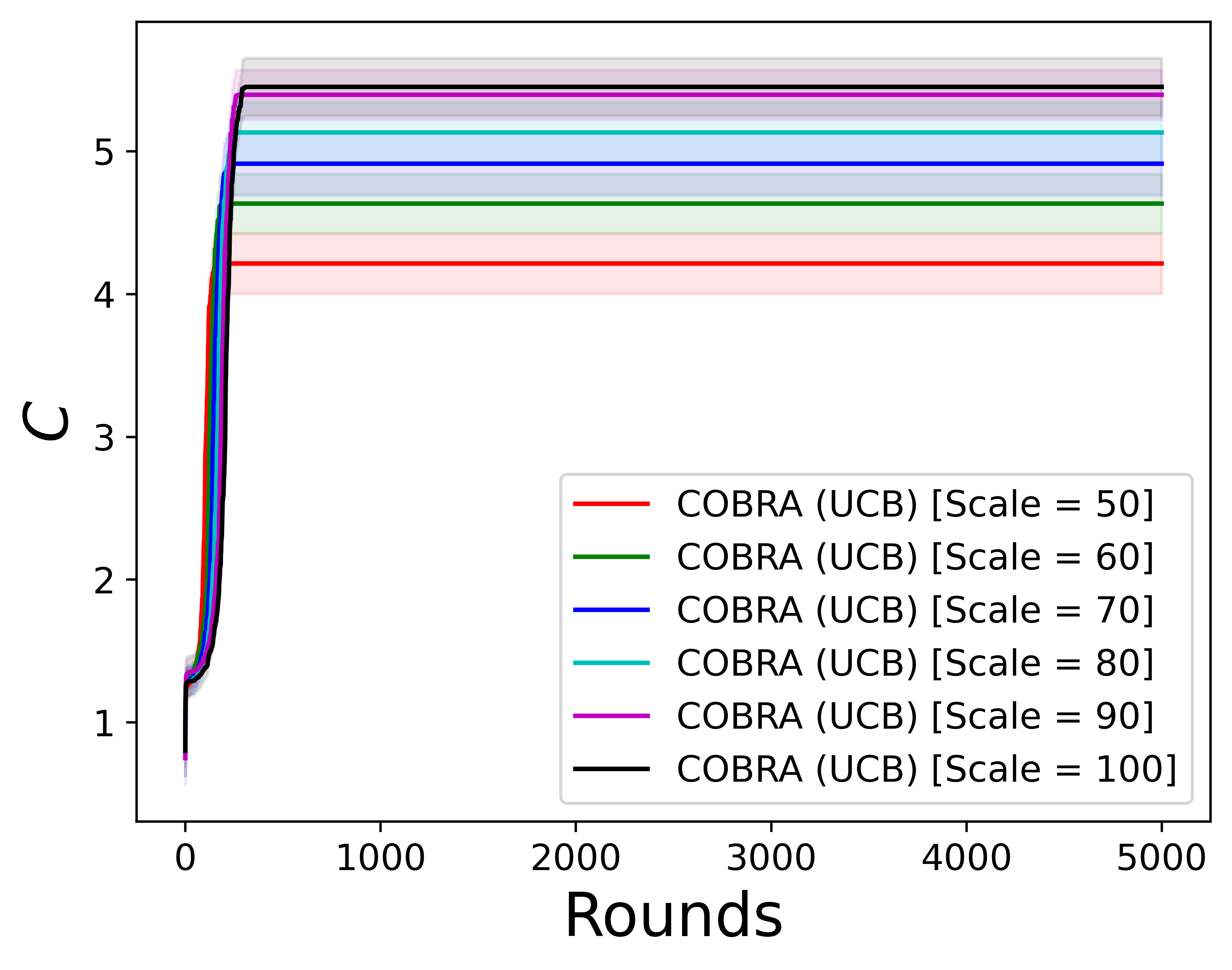}}
	\subfloat[Vary dimension ($100$)]{\label{fig:vary_large_dimsC_100_ucb}
		\includegraphics[width=0.24\linewidth]{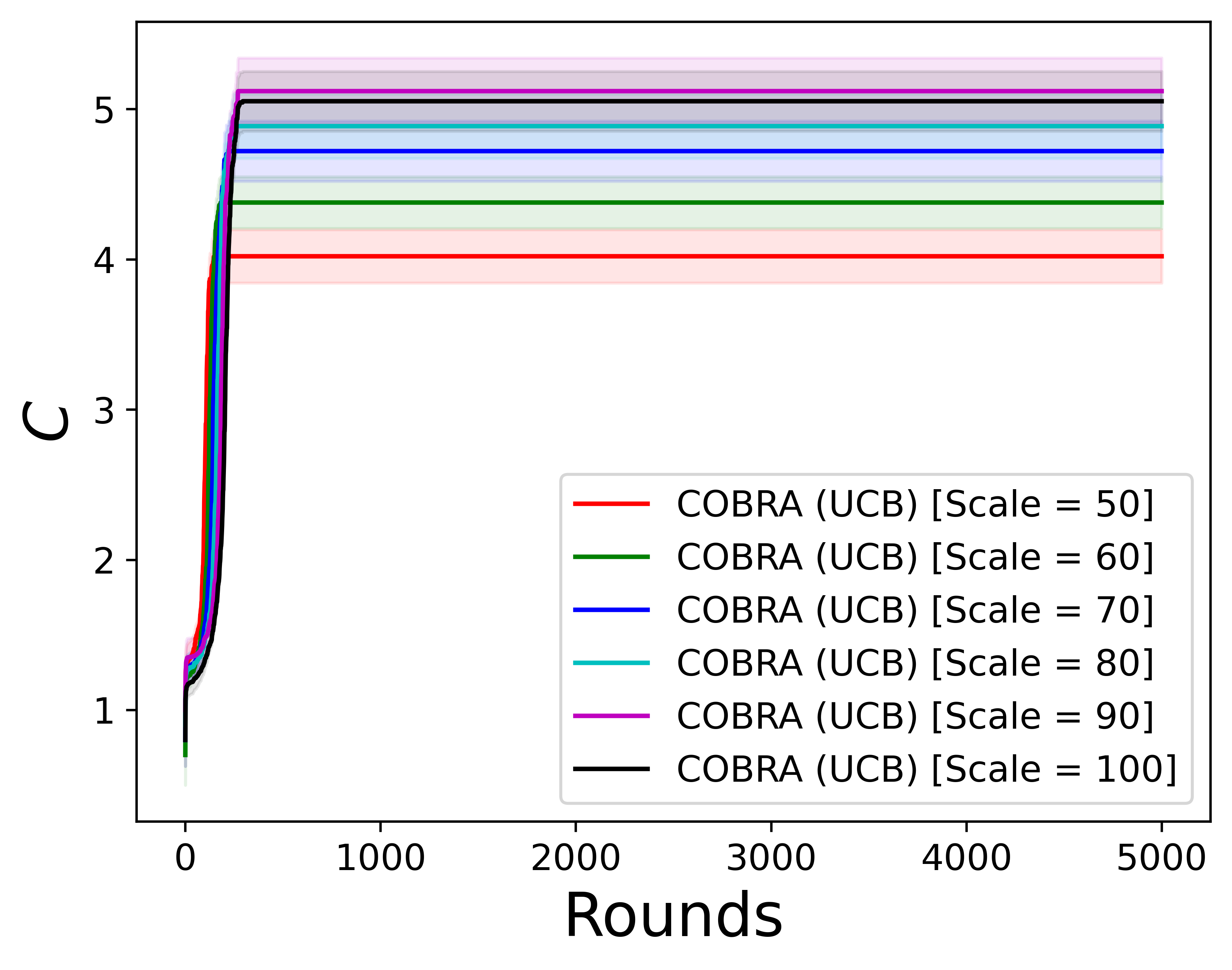}}
    \caption{
       \textbf{Top row:} Cumulative regret of \algo{} vs. different values of $N$ and $d$. \textbf{Bottom row:} Maximum value of $C$ vs. different values of $N$ and $d$. 
	}
	\label{fig:compareLargeExp}
    \vspace{-3mm}
\end{figure}

\para{Computational resources.} All the experiments are run on an Apple M3 Pro with 18GB memory.

\para{Time and space complexity of \algo{}.}
The computational complexity of COBRA is comparable to that of standard contextual bandit algorithms, and it scales efficiently even when agents control multiple arms, as discussed below:\\
\textbf{Time complexity:} The overall time complexity is dominated by the underlying estimation procedure, which is identical to that of standard contextual bandit algorithms. LOOM requires computing a Leave-One-Out (LOO) estimator for each agent. These LOO estimators are inherently parallelizable, i.e., one can independently estimate each agent-specific estimator (trained on data excluding agent $a$). This parallelism substantially alleviates computational burden, although the total computational resources required still scale linearly with the number of agents.\\
\textbf{Space complexity:} The space complexity is likewise similar to standard contextual bandit approaches, with the main additional overhead arising from maintaining $N$ separate LOO estimators.


\section{Frequently Asked Questions}
\label{asec:faq}

\para{Question 1.} How sensitive is LOOM to the number of agents? What happens when only a few agents are active?

\para{Answer.} 
LOOM uses data contributed by other agents to construct reliable leave-one-out estimates, and its detection power improves as the number of truthful agents increases. When only a small number of agents are active, excluding a single agent removes a substantial portion of the available data, resulting in wider confidence intervals and more conservative behavior. In this regime, LOOM may delay agent elimination but does not introduce false positives under truthful reporting, with probability at least $1-\delta$ (\cref{thm:optgtmarmsstayactive}). Consequently, COBRA smoothly degrades to standard contextual bandit behavior, and, under sufficient peer coverage (bounded $C$), strategic agents can obtain at most a limited and transient advantage; without peer coverage, this need not hold (Example~2 in \cref{asec:loom}). LOOM is therefore most effective in the multi-agent regimes that motivate our study.

\para{Question 2.} How does LOOM behave when two agents collude or misreport symmetrically?

\para{Answer.} 
LOOM detects unilateral over-reporting through leave-one-out comparisons. If two agents collude or misreport symmetrically, each agent’s reference set may include the other’s biased data, potentially delaying or preventing detection; this reflects a fundamental limitation of peer-based mechanisms. Our incentive guarantees do not cover coalitions. If both colluding agents over-report and \cref{assu:NEandregret} holds, the regret bound of \cref{thm:regretAllNE} still applies, since it holds for every over-reporting profile; otherwise, no guarantee is provided. Addressing more general forms of collusion would require additional structural assumptions (e.g., a minimum fraction of truthful agents), which we leave for future work. \\

\para{Question 3.} Can the LOOM idea be extended to heterogeneous agents without requiring a shared reward structure?

\para{Answer.} Our LOOM mechanism relies on a shared reward structure, since the leave-one-out estimator uses observations from other agents to estimate the target agent’s reward. In heterogeneous settings, this transferability may no longer hold.
However, the core idea can potentially be extended under additional structure, such as clustered agent types, similarity-based weighting, or shared latent representations. More generally, LOOM mainly requires sufficient coverage of an agent’s feature distribution by other agents rather than exact reward homogeneity. Extending LOOM to fully heterogeneous settings is an important direction for future work.
We also note that the conditions that create an incentive to misreport often coincide with those that provide peer coverage: as illustrated in \cref{asec:loom}, a seller who is the sole supplier of a product category has no competitor to displace and hence little scope to benefit from misreporting, whereas sellers of similar products, whose features are drawn from similar distributions, provide the peer coverage that LOOM can exploit.

\para{Question 4.} Why not exploit user feedback directly to identify strategic agents, e.g., ask the user whether a recommended restaurant actually had the vegetarian menu it advertised?

\para{Answer.} LOOM already does this, but at the \emph{reward} level rather than the attribute level. The optimistic side of the LOOM test, $\ucb^{(y)}_{t,a} = \sum_{s \le t,\, a_s = a} y_s + \sqrt{2R^2S_t(a)\log(1/\delta^y_{t,a})}$ (see \cref{lem:rewOptEst}), 
is constructed entirely from the realized user feedback collected when agent $a$ is selected. Because rewards depend only on the true features (Footnote~2), an agent that over-reports a feature can receive realized feedback that is inconsistent with its report. The LOOM condition \cref{eqn:loom} therefore compares what the agent's reports {promise}, through the peer-based pessimistic estimate, with what its realized feedback {delivers}.
What LOOM does not use is {attribute-level} verification, such as asking the user whether a specific declared feature was accurate. This is a promising direction for heterogeneous settings because attribute-level feedback would not require cross-agent comparability and hence would not be subject to the coverage limitation captured by $C$ in \cref{def:advantage}.

\section{Scope and Future Work} \label{asec:scope}
LOOM detects over-reporting, but not under-reporting, under {unilateral} deviations, and its power depends on how well the remaining agents cover the deviating agent's arm-feature directions. We make this dependence explicit through the maximum advantage $C$ (\cref{def:advantage}), which appears in \cref{thm:regretAllNE} and \cref{thm:regretNonAllNE}. As peer coverage deteriorates, the corresponding guarantees degrade continuously and can become vacuous in the failure cases illustrated in \cref{asec:loom}.

These limitations suggest several directions for future work. \emph{(i) Theoretical control of $C$}: identifying design or exploration conditions under which $C$ can be provably bounded, rather than measured empirically as in our experiments. \emph{(ii) Heterogeneous agents}: relaxing the shared reward-function (homogeneity) assumption (\cref{assu:homogeneous}) through clustered agent types, similarity-based weighting, or shared latent representations (\cref{asec:faq}, Question~3), since LOOM fundamentally requires peer coverage rather than exact reward homogeneity. \emph{(iii) Attribute-level user feedback}: extending LOOM beyond reward-level feedback to settings where users directly verify declared attributes, which would remove the need for cross-agent comparability and could naturally accommodate heterogeneous agents (\cref{asec:faq}, Question~4). \emph{(iv) Coverage-aware and colluding agents}: designing mechanisms that remain effective when agents deliberately misreport in weakly covered directions or coordinate their reports. \emph{(v) Verifiable conditions and agent dynamics}: replacing \cref{assu:NEandregret} by an instance-level condition, and studying whether learning agents (e.g., no-regret learners) converge to truthful reporting, which is harder here since \algo{} itself learns and the game the agents face changes over time. \emph{(vi) Direct incentive experiments}: measuring a deviating agent's utility against its misreport magnitude, the false-elimination rate, and the detection time, including best-responding agents. \\

    \hrule height 0.5mm

\end{document}